\documentclass[11pt]{article}

\usepackage{theapa}
\usepackage{jair}
\usepackage{times}
\usepackage{curves}
\usepackage{graphicx}
\usepackage{amsmath,amssymb,theorem,enumerate}
\usepackage{array}
\usepackage{wasysym}
\usepackage{multirow}
\usepackage{stmaryrd}
\usepackage{float}

\input epsf 

\input xy
\xyoption{all}

\def\defemb#1#2{\expandafter\def\csname #1\endcsname
  {\relax\ifmmode #2\else\hbox{$#2$}\fi}}

\defemb{cA}{{\cal A}}
\defemb{cB}{{\cal B}}
\defemb{cC}{{\cal C}}
\defemb{cD}{{\cal D}}
\defemb{cE}{{\cal E}}
\defemb{cG}{{\cal G}}
\defemb{cI}{{\cal I}}
\defemb{cK}{{\cal K}}
\defemb{cN}{{\cal N}}
\defemb{cM}{{\cal M}}
\defemb{cP}{{\cal P}}
\defemb{cR}{{\cal R}}
\defemb{cS}{{\cal S}}
\defemb{cT}{{\cal T}}
\defemb{cV}{{\cal V}}
\defemb{cU}{{\cal U}}
\defemb{cW}{{\cal W}}
\defemb{cO}{{\cal O}}
\defemb{cL}{{\cal L}}
\defemb{cX}{{\cal X}}
\defemb{cY}{{\cal Y}}

\defemb{bp}{{\bf p}}
\defemb{bU}{{\bf U}}
\defemb{bu}{{\bf u}}
\defemb{bV}{{\bf V}}
\defemb{bW}{{\bf W}}
\defemb{bw}{{\bf w}}
\defemb{bX}{{\bf X}}
\defemb{bx}{{\bf x}}
\defemb{bY}{{\bf Y}}
\defemb{by}{{\bf y}}
\defemb{bZ}{{\bf Z}}
\defemb{bz}{{\bf z}}

\newenvironment{proof}{\noindent {\bf Proof:}}%
{\hfill \rule[0.3ex]{1ex}{1ex} \par \addvspace{\bigskipamount}}
{\hfill \rule[0.3ex]{1ex}{1ex} \par \addvspace{\bigskipamount}}
\newtheorem{theorem}{Theorem}
 
 \newtheorem{proposition}[theorem]{Proposition}
 \newtheorem{lemma}[theorem]{Lemma}
 \newtheorem{corollary}[theorem]{Corollary}
{\theorembodyfont{\rmfamily}\newtheorem{example}{Example}}

\def\reals{{\mathbb R}}

\newcommand{\nnreals}{\reals^{0+}}

\newcommand{\sharpp}{\#P}

\newcommand{\ff}{FF}

\newcommand{\cff}{Conformant-FF}
\newcommand{\pff}{Probabilistic-FF}
\newcommand{\pond}{POND}

\newcommand{\buildprpg}{${\mathsf{build\mbox{-}PRPG}}$}
\newcommand{\buildts}{${\mathsf{build\mbox{-}timestep}}$}
\newcommand{\buildwil}{${\mathsf{build\mbox{-}w\mbox{-}impleafs}}$}
\newcommand{\getp}{${\mathsf{get\mbox{-}P}}$}
\newcommand{\extractprp}{${\mathsf{extract\mbox{-}PRPlan}}$}

\newcommand{\reduceimp}{${\mathsf{reduce\mbox{-}implication\mbox{-}graph}}$}
\newcommand{\iterativesg}{${\mathsf{extract\mbox{-}subplan}}$}
\newcommand{\constructsupport}{${\mathsf{construct\mbox{-}support\mbox{-}graph}}$}
\newcommand{\subgoal}{${\mathsf{sub\mbox{-}goal}}$}

\newcommand{\strips}{STRIPS}

\newcommand{\tstamp}[1]{_{(#1)}}
\newcommand{\tind}[1]{^{#1}}
\newcommand{\cpt}[1]{T_{#1}}
\newcommand{\pre}{pre}
\newcommand{\effs}{E}

\newcommand{\poutcomeset}{\Lambda}
\newcommand{\poutcome}{\varepsilon}
\newcommand{\probe}{Pr}

\newcommand{\con}{con}
\newcommand{\add}{add}
\newcommand{\del}{del}
\newcommand{\actions}{A}

\newcommand{\goal}{G}

\newcommand{\initial}{I}
\newcommand{\initwstate}{w_{\initial}}
\newcommand{\cinitial}{\ourf(\initBN)}

\newcommand{\plan}{\overline{a}}
\newcommand{\relplanbig}{\plan^{R}}
\newcommand{\relplan}{\plan^{r}}

\newcommand{\implication}{\rightarrow}
\newcommand{\onecondrel}{|^{+}_{1}}
\newcommand{\noop}{{\mathrm{noop}}}

\newcommand{\newchancep}[1]{#1}
\newcommand{\leafs}[1]{\mbox{\it leafs}(#1)}
\newcommand{\grapheffects}[1]{\mbox{\it E}(#1)}

\newcommand{\support}{\mbox{\sl support}}

\newcommand{\weight}{\varpi}
\newcommand{\flbweight}{\alpha}
\newcommand{\subweight}{\weight}

\newcommand{\ourf}{\phi}
\newcommand{\cnf}[1]{\phi(#1)}
\newcommand{\seqactions}{\plan}

\newcommand{\initprob}{\initbelief}
\newcommand{\goalprob}{\theta}
\newcommand{\fluents}{\cP}
\newcommand{\worldstates}{W}

\newcommand{\initbelief}{\belief_{I}}
\newcommand{\belief}{b}
\newcommand{\applying}[2]{\left[#1,#2\right]}
\newcommand{\BN}{\cN}
\newcommand{\BNvars}{\cX}
\newcommand{\bnassign}{\vartheta}
\newcommand{\initBN}{\BN_{\initbelief}}
\newcommand{\basicWMC}{{\mathsf{basic\mbox{-}WMC}}}

\newcommand{\argmin}{\operatornamewithlimits{argmin}}

\newcommand{\commentout}[1]{}

\jairheading{30}{2007}{565-620}{3/07}{12/07}
\ShortHeadings{\pff}{Domshlak \& Hoffmann}
\firstpageno{565}

\begin{document}

\title{Probabilistic Planning via Heuristic Forward Search\\ and Weighted Model Counting}

\author{\name Carmel Domshlak
\email dcarmel@ie.technion.ac.il\\
\addr Technion - Israel Institute of Technology,\\
Haifa, Israel\\
\name J\"org Hoffmann
  \email Joerg.Hoffmann@deri.at\\
  \addr University of Innsbruck, DERI,\\
  Innsbruck, Austria}

\maketitle

\begin{abstract}
  We present a new algorithm for probabilistic planning with no
  observability.  Our algorithm, called \pff, extends the heuristic
  forward-search machinery of \cff\ to problems with probabilistic
  uncertainty about both the initial state and action effects.
  Specifically, \pff\ combines \cff's techniques with a powerful
  machinery for weighted model counting in (weighted) CNFs, serving to
  elegantly define both the search space and the heuristic function.
  Our evaluation of \pff\ shows its fine scalability in a range of
  probabilistic domains, constituting a several orders of magnitude
  improvement over previous results in this area. We use a problematic
  case to point out the main open issue to be addressed by further
  research.
\end{abstract}

\section{Introduction}
\label{intro}

In this paper we address the problem of {\em probabilistic planning
  with no observability}~\cite{Buridan},
  also known in the AI planning community as conditional~\cite{Zander}
  or conformant~\cite{CPP2} probabilistic planning.
  In such problems we are given an initial belief state in the form of
  a probability distribution over the world states $W$, a set of
  actions (possibly) having probabilistic effects, and a set of
  alternative goal states $\worldstates_{G} \subseteq \worldstates$. A
  solution to such a problem is a single sequence of actions that
  transforms the system into one of the goal states with probability
  exceeding a given threshold $\goalprob$.  The basic assumption of
  the problem is that the system cannot be observed at the time of
  plan execution. Such a setting is useful in controlling systems with
  uncertain initial state and non-deterministic actions, if sensing is
  expensive or unreliable.  Non-probabilistic conformant planning may
  fail due to non-existence of a plan that achieves the goals with
  100\% certainty.  Even if there is such a plan, that plan does not
  necessarily contain information about what actions are most useful
  to achieve only the requested threshold $\goalprob$.

%

The state-of-the-art performance of probabilistic planners has been
advancing much more slowly than that of deterministic planners,
scaling from 5-10 step plans for problems with $\approx$20 world
states to 15-20 step plans for problems with $\approx$100 world
states~\cite{Buridan,Maxplan,CPP2}.  Since probabilistic planning is
inherently harder than its deterministic
counterpart~\cite{pp-complexity}, such a difference in evolution rates
is by itself not surprising. However, recent developments in the
area~\cite{probapop,pond:icaps-06,huang:icaps-06}, and in particular
our work here, show that dramatic improvements in probabilistic
planning {\em can} be obtained.

In this paper we introduce \pff, a new probabilistic planner based on
{\em heuristic forward search} in the space of {\em implicitly
represented} probabilistic belief states. The planner is a natural
extension of the recent (non-probabilistic) conformant planner
\cff~\cite{cff}. The main trick is to replace \cff's SAT-based
techniques with a recent powerful technique for probabilistic
reasoning by weighted model counting (WMC) in propositional CNFs
\cite{sang:etal:aaai-05}. In more detail, \cff\ does a forward search
in a belief space in which each belief state corresponds to a set of
world states considered to be possible.  The main trick of \cff\ is
the use of CNF formulas for an implicit representation of belief
states.  Implicit, in this context, means that formulas
$\ourf(\seqactions)$ encode the semantics of executing action sequence
$\seqactions$ in the initial belief state, with propositional
variables corresponding to facts with time-stamps.  Any actual
knowledge about the belief states has to be (and can be) inferred from
these formulas. Most particularly, a fact $p$ is known to be true in a
belief state if and only if $\ourf(\seqactions) \implication p(m)$, where $m$ is
the time endpoint of the formula. The only knowledge computed by \cff\
about belief states are these {\em known facts}, as well as
(symmetrically) the facts that are known to be false. This suffices to
do \strips-style planning, that is, to determine applicable actions and
goal belief states. In the heuristic function, FF's
\cite{hoffmann:nebel:jair-01} relaxed planning graph technique is
enriched with approximate SAT reasoning.

The basic ideas underlying \pff\ are: 
\begin{enumerate}[(i)]
\item Define time-stamped Bayesian networks (BNs) describing
probabilistic belief states.
\item Extend \cff's belief state CNFs to model these BNs.
\item In addition to the SAT reasoning used by \cff, use weighted
model-counting to determine whether the probability of the (unknown)
goals in a belief state is high enough.
\item Introduce approximate probabilistic reasoning into \cff's
heuristic function.
\end{enumerate}
Note the synergetic effect: \pff\ re-uses all of \cff's technology to
recognize facts that are true or false with probability $1$. This
fully serves to determine applicable actions, as well as detect
whether part of the goal is already known. In fact, it is as if \cff's
CNF-based techniques were specifically made to suit the probabilistic
setting: while without probabilities one could imagine successfully
replacing the CNFs with BDDs, {\em with} probabilities this seems much more problematic.

The algorithms we present cover probabilistic initial belief states
given as Bayesian networks, deterministic and probabilistic actions,
conditional effects, and standard action preconditions. Our
experiments show that our approach is quite effective in a range of
domains. In contrast to the SAT and CSP based approaches mentioned
above~\cite{Maxplan,CPP2}, \pff\ can find {\em 100-step plans for problem instances with billions of world states}. 
However, such a comparison is not entirely fair due
to the different nature of the results provided; the SAT and CSP based
approaches provide guarantees on the length of the solution. The
approach most closely related to \pff\ is implemented in \pond\
\cite{pond:icaps-06}: this system, like \pff, does conformant
probabilistic planning for a threshold $\goalprob$, using a
non-admissible, planning-graph based heuristic to guide the search. Hence a comparison between \pff\ and \pond\ is fair, and in our experiments we perform a comparative evaluation of \pff\ and \pond. While the two approaches are related, there are significant differences in the search space
representation, as well as in the definition and computation of the heuristic
function.\footnote{\pond\ does not use implicit belief states, and the
probabilistic part of its heuristic function uses sampling techniques,
rather than the probabilistic reasoning techniques we employ.} We run
the two approaches on a range of domains partly taken from the
probabilistic planning literature, partly obtained by enriching
conformant benchmarks with probabilities, and partly obtained by
enriching classical benchmarks with probabilistic uncertainty. In
almost all cases, \cff\ outperforms \pond\ by at least an order of
magnitude. We make some interesting observations regarding the
behavior of the two planners; in particular we identify a domain --
derived from the classical Logistics domain -- where both approaches
fail to scale. The apparent reason is that neither approach is good
enough at detecting how many times, at an early point in the plan, a
probabilistic action must be applied in order to sufficiently support
a high goal threshold at the end of the plan. Devising methods that
are better in this regard is the most pressing open issue in this line
of work.

The paper is structured as follows. The next section provides the
technical background, formally defining the problem we address and
illustrating it with our running example. Section~\ref{ss:primal}
details how probabilistic belief states are represented as
time-stamped Bayesian networks, how these Bayesian networks are encoded as weighted CNF
formulas, and how the necessary reasoning is performed on this
representation. Section~\ref{ss:heuristic} explains and illustrates our extension of
\cff's heuristic function to the probabilistic
settings. Section~\ref{ss:results} provides the empirical results, and
Section~\ref{ss:conclusion} concludes. 
All proofs are moved into Appendix~\ref{ss:proofs}.

\section{Background}
\label{prob-planning}

The probabilistic planning framework we consider adds probabilistic
uncertainty to a subset of the classical ADL language, namely
(sequential) STRIPS with conditional effects.  Such STRIPS planning
tasks are described over a set of propositions $\fluents$ as triples
$(\actions, \initial, \goal)$, corresponding to the {\em action set},
{\em initial world state}, and {\em goals}. $\initial$ and $\goal$ are
sets of propositions, where $I$ describes a concrete initial state
$\initwstate$, while $\goal$ describes the set of goal states $w
\supseteq G$.  Actions $a$ are pairs $(\pre(a), \effs(a))$ of the {\em
precondition} and the {\em (conditional) effects}.  A conditional
effect $e$ is a triple $(\con(e), \add(e), \del(e))$ of (possibly
empty) proposition sets, corresponding to the effect's {\em
condition}, {\em add}, and {\em delete} lists, respectively. The
precondition $\pre(a)$ is also a proposition set, and an action $a$ is
{\em applicable} in a world state $w$ if $w \supseteq \pre(a)$. If $a$
is not applicable in $w$, then the result of applying $a$ to $w$ is
undefined. If $a$ is applicable in $w$, then all conditional effects
$e \in \effs(a)$ with $w \supseteq \con(e)$ occur.  Occurrence of a
conditional effect $e$ in $w$ results in the world state $w \cup
\add(e) \setminus \del(e)$.

If an action $a$ is applied to $w$, and there is a proposition $q$
  such that $q \in \add(e)\cap \del(e')$ for (possibly the same)
  occurring $e,e' \in \effs(a)$, then the result of applying $a$ in
  $w$ is undefined. Thus, we require the actions to be not
  self-contradictory, that is, for each $a \in \actions$, and every
  $e,e' \in \effs(a)$, if there exists a world state $w \supseteq
  \con(e) \cup \con(e')$, then $\add(e)\cap \del(e') = \emptyset$.
  Finally, an action sequence $\seqactions$ is a {\em plan} if the
  world state that results from iterative execution of $\seqactions$'s
  actions, starting in $\initwstate$, leads to a {\em goal state} $w
  \supseteq \goal$.

\subsection{Probabilistic Planning}

Our probabilistic planning setting extends the above with (i)
probabilistic uncertainty about the initial state, and (ii) actions
that can have probabilistic effects.  In general, probabilistic
planning tasks are quadruples $(\actions,\initprob,\goal,\goalprob)$,
corresponding to the {\em action set}, {\em initial belief state},
{\em goals}, and {\em acceptable goal satisfaction probability}.  As
before, $\goal$ is a set of propositions. The initial state is no
longer assumed to be known precisely.  Instead, we are given a
probability distribution over the world states, $\initprob$, where
$\initprob(w)$ describes the likelihood of $w$ being the initial world
state.

Similarly to classical planning, actions $a \in \actions$ are pairs
$(\pre(a), \effs(a))$, but the effect set $\effs(a)$ for such $a$ has
richer structure and semantics. Each $e \in \effs (a)$ is a pair
$(\con(e), \poutcomeset(e))$ of a propositional condition and a set of
{\em probabilistic outcomes}.  Each probabilistic outcome $\poutcome
\in \poutcomeset(e)$ is a triplet $(\probe(\poutcome),
\add(\poutcome),\del(\poutcome))$, where {\em add} and {\em delete}
lists are as before, and $\probe(\poutcome)$ is the probability that
outcome $\poutcome$ occurs as a result of effect $e$. Naturally, we
require that probabilistic effects define probability distributions
over their outcomes, that is,
$\sum_{\poutcome\in\poutcomeset(e)}{\probe(\poutcome)} = 1$.  The
special case of deterministic effects $e$ is modeled this way via
$\poutcomeset(e) = \{\poutcome\}$ and $\probe(\poutcome)=1$.
Unconditional
actions are modeled as having a single effect $e$ with $\con(e) =
\emptyset$. As before, if $a$ is not applicable in $w$, then the
result of applying $a$ to $w$ is undefined.  Otherwise, if $a$ is
applicable in $w$, then there exists exactly one effect $e \in
\effs(a)$ such that $\con(e) \subseteq w$, and for each $\poutcome \in
\poutcomeset(e)$, applying $a$ to $w$ results in $w \cup
\add(\poutcome) \setminus \del(\poutcome)$ with probability
$\probe(\poutcome)$.  The likelihood $\applying{b}{a}(w')$ of a world
state $w'$ in the belief state $\applying{b}{a}$, resulting from
applying a probabilistic action $a$ in $b$, is given by
\begin{equation}
\label{e:bpa}
    \applying{b}{a}(w') =  
    \sum_{w \supseteq \pre(a)}{
    		\belief(w)
		\sum_{\poutcome \in \poutcomeset(e)}{
		  \probe(\poutcome)
		    \cdot
		    \delta\left( w' = w \cup s \setminus s', 
		                 s \subseteq \add(\poutcome),
		                 s' \subseteq \del(\poutcome)
		                 \right)
		}
    },
\end{equation}
where $e$ is the effect of $a$ such that $\con(e) \subseteq w$, and
$\delta(\cdot)$ is the Kronecker step function that takes the value
$1$ if the argument predicate evaluates to TRUE, and $0$ otherwise.

Our formalism covers all the problem-description features supported by
the previously proposed formalisms for conformant probabilistic
planning
\cite{Buridan,Maxplan,CPP2,probapop,pond:icaps-06,huang:icaps-06}, and
it corresponds to what is called Unary Nondeterminism (1ND) normal
form~\cite{rintanen:icaps-06}. We note that there are more succinct
forms for specifying probabilistic planning
problems~\cite{rintanen:icaps-06}, yet 1ND normal form appears to be
most intuitive from the perspective of knowledge engineering.

\begin{example} 
\label{example1} 
Say we have a robot and a block that physically can be at one of two
locations. This information is captured by the propositions
$r_{1},r_{2}$ for the robot, and $b_{1},b_{2}$ for the block,
respectively. The robot can either move from one location to another,
or do it while carrying the block. If the robot moves without the
block, then its move is guaranteed to succeed. This provides us with a
pair of symmetrically defined deterministic actions
$\{move\mbox{-}right, move\mbox{-}left\}$. The action
$move\mbox{-}right$ has an empty precondition, and a single
conditional effect $e = (\{r_{1}\}, \{\poutcome\})$ with
$\probe(\poutcome)=1$, $\add(\poutcome) = \{r_{2}\}$, and
$\del(\poutcome) = \{r_{1}\}$. 
If the robot
tries to move while carrying the block, then this move succeeds with
probability $0.7$, while with probability $0.2$ the robot ends up
moving without the block, and with probability $0.1$ this move of the
robot fails completely. This provides us with a pair of (again,
symmetrically defined) probabilistic actions
$\{move\mbox{-}b\mbox{-}right, move\mbox{-}b\mbox{-}left\}$. The
action $move\mbox{-}b\mbox{-}right$ has an empty precondition, and two
conditional effects specified as in Table~\ref{table:example1}.

\begin{table}[ht]
\begin{center}
  \setlength{\extrarowheight}{2pt}	
  \begin{tabular}{|c|c|c|c|c|c|}
  \hline
   $\!\effs(a)\!$ & $\con(e)$ & $\poutcomeset(e)$ & $\probe(\poutcome) $ & $\add(\poutcome)$ & $\del(\poutcome)$ \\
  \hline
  \hline
  & & $\poutcome_{1}$ & 0.7 & $\{r_{2},b_{2}\}$ & $\{r_{1},b_{1}\}$\\
  \cline{3-6}
  $e$ & $\!r_{1}\wedge b_{1}\!$ & $\poutcome_{2}$ & 0.2 & $\{r_{2}\}$ & $\{r_{1}\}$\\
  \cline{3-6}
  & & $\poutcome_{3}$ & 0.1 & $\emptyset$ & $\emptyset$\\
  \hline
  $e'$ & $\!\neg r_{1}\vee \neg b_{1}\!$ & $\poutcome_{1}'$ &  1.0 & $\emptyset$ & $\emptyset$\\
  \hline
  \end{tabular}
\end{center}
\caption{\label{table:example1} Possible effects and outcomes of the action $move\mbox{-}b\mbox{-}right$ in Example~\ref{example1}.}
\end{table}

\end{example}

Having specified the semantics and structure of all the components of
$(\actions,\initprob,\goal,\goalprob)$ but $\goalprob$, we are now
ready to specify the actual task of probabilistic planning in our
setting.
Recall that our actions transform probabilistic belief states to
belief states. For any action sequence $\plan \in \actions^{*}$, and
any belief state $\belief$, the new belief state
$\applying{\belief}{\plan}$ resulting from applying $\plan$ at
$\belief$ is given by
\begin{equation}
\label{e:bas}
\applying{\belief}{\plan} = \begin{cases}
\belief, & \plan = \langle\rangle\\
\applying{\belief}{a}, & \plan = \langle a \rangle, a\in\actions\\
\applying{\applying{\belief}{a}}{\plan'}, & \plan = \langle a \rangle\cdot\plan', a\in\actions,\plan'\neq\emptyset\\
\end{cases}.
\end{equation}
In such setting, achieving $G$ with certainty is typically
unrealistic. Hence, $\goalprob$ specifies the required {\em lower
bound} on the probability of achieving $G$. A sequence of actions
$\seqactions$ is called a {\em plan} if we have
$\belief_{\seqactions}(G) \geq \goalprob$ for the belief state
$\belief_{\seqactions} = \applying{\initbelief}{\plan}$.

\subsection{Specifying the Initial Belief State}

Considering the initial belief state, practical considerations force
us to limit our attention only to compactly representable probability
distributions $\initprob$.  While there are numerous alternatives for
compact representation of structured probability distributions, Bayes
networks (BNs)~\cite{Pearl} to date is by far the most popular such
representation model.\footnote{While BNs are our choice here, our
framework can support other models as well, e.g.\ stochastic decision
trees.} Therefore, in \pff\ we assume that the initial belief state
$\initprob$ is described by a BN $\initBN$ over our set of
propositions $\fluents$.

As excellent introductions to BNs abound~\cite<e.g., see>{Jensen:bn},
it suffices here to briefly define our notation. A BN $\BN = (\cG,
\cT)$ represents a probability distribution as a directed acyclic
graph $\cG$, where its set of nodes $\BNvars$ stands for random
variables (assumed discrete in this paper), and $\cT$, a set of tables
of conditional probabilities (CPTs)---one table $\cpt{X}$ for each
node $X\in \BNvars$. For each possible value $x\in Dom(X)$ (where
$Dom(X)$ denotes the domain of $X$), the table $\cpt{X}$ lists the
probability of the event $X=x$ given each possible value assignment to
all of its immediate ancestors (parents) $Pa(X)$ in $\cG$. Thus, the
table size is exponential in the in-degree of $X$.  Usually, it is
assumed either that this in-degree is small~\cite{Pearl}, or that the
probabilistic dependence of $X$ on $Pa(X)$ induces a significant local
structure allowing a compact representation of
$\cpt{X}$~\cite{shimony:irrel:93,shimony:irrel:95,boutiliercontextspecific}.  (Otherwise, representation
of the distribution as a BN would not be a good idea in the first
place.)  The joint probability of a complete assignment $\bnassign$ to
the variables $\BNvars$ is given by the product of $|\BNvars|$ terms
taken from the respective CPTs~\cite{Pearl}:
\[
Pr(\bnassign) = \prod_{X\in \BNvars} Pr\left(\bnassign[X]\mid\bnassign[Pa (X)]\right) = \prod_{X\in \BNvars} \cpt{X} \left(\bnassign[X] \mid \bnassign[Pa (X)]\right),
\]
where $\bnassign[\cdot]$ stands for the partial assignment provided by
$\bnassign$ to the corresponding subset of $\BNvars$.



In \pff\ we allow $\initBN$ to be described over the multi-valued
variables underlying the planning problem. 
This significantly simplifies the
process of specifying $\initBN$ since the STRIPS propositions
$\fluents$ do not correspond to the true random variables underlying
problem specification.\footnote{Specifying $\initBN$ directly over
$\fluents$ would require identifying the multi-valued variables
anyway, followed by connecting all the propositions corresponding to a
multi-valued variable by a complete DAG, and then normalizing the CPTs
of these propositions in a certain manner.}  Specifically, let
$\bigcup_{i=1}^{k}\fluents_{i}$ be a partition of $\fluents$ such that
each proposition set $\fluents_{i}$ uniquely corresponds to the domain
of a multi-valued variable underlying our problem.  That is, for every
world state $w$ and every $\fluents_{i}$, if $|\fluents_{i}| > 1$,
then there is {\em exactly one} proposition $q \in \fluents_{i}$ that
holds in $w$.  The variables of the BN $\initBN$ describing our
initial belief state $\initprob$ are $\BNvars =
\{X_{1},\ldots,X_{k}\}$, where $Dom(X_{i}) = \fluents_{i}$ if
$|\fluents_{i}| > 1$, and $Dom(X_{i}) = \{q,\neg q\}$ if $\fluents_{i} =
\{q\}$.

\begin{example} 
\label{example2}
For an illustration of such $\initBN$, consider our running example,
and say the robot is known to be initially at one of the two possible
locations with probability $Pr(r_{1}) = 0.9$ and $Pr(r_{2}) = 0.1$.
Suppose there is a correlation in our belief about the initial
locations of the robot and the block. We believe that, if the robot is
at $r_{1}$, then $Pr(b_{1}) = 0.7$ (and $Pr(b_{2}) = 0.3$), while if
the robot is at $r_{2}$, then $Pr(b_{1}) = 0.2$ (and $Pr(b_{2}) =
0.8$). The initial belief state BN $\initBN$ 
is then defined over two variables $R$ (``robot'') and $B$
(``block'') with $Dom(R) = \{r_{1},r_{2}\}$ and $Dom(B) =
\{b_{1},b_{2}\}$, respectively, and it is depicted in
Figure~\ref{fig:bn-init}.

\begin{figure}[ht]
\begin{center}
\begin{tabular}{ccc}
\begin{tabular}{|cc|}
\hline
 $r_{1}$ & $r_{2}$\\
 \hline
 $0.9$ & $0.1$\\
\hline
\end{tabular}
&
\begin{minipage}{0.5in}
  \def\objectsizestyle{\scriptstyle}
      \xymatrix @R=35pt@C=25pt{
        *++[o][F-,]{R} \ar[r] &
        *++[o][F-,]{B}
        }
\end{minipage}
&
\begin{tabular}{|c|cc|}
\hline
 & $b_{1}$ & $b_{2}$\\
 \hline
 $r_{1}$ & $0.7$ & $0.3$\\
 \hline
 $r_{2}$ & 0.2 & 0.8\\
\hline
\end{tabular}
\end{tabular}
\end{center}
\caption{Bayes network $\initBN$ for Example~\ref{example1}.}
\label{fig:bn-init}
\end{figure}

\end{example}

It is not hard to see that our STRIPS-style actions $a \in \actions$
can be equivalently specified in terms of the multi-valued variables
$\BNvars$. Specifically, if $|\fluents_{i}| > 1$, then no action $a$
can add a proposition $q \in \fluents_{i}$ without deleting some other
proposition $q' \in \fluents_{i}$, and thus, we can
consider $a$ as setting $X_{i} = q$. If $|\fluents_{i}| = 1$, then
adding and deleting $q \in \fluents_{i}$ has the standard semantics of
setting $X_{i} = q$ and $X_{i} = \neg q$, respectively.  For
simplicity of presentation, we assume that our actions are not
self-contradictory at the level of $\BNvars$ as well---if two
conditional effects $e,e' \in \effs(a)$ can possibly occur in some
world state $w$, then the subsets of $\BNvars$ affected by these two
effects have to be disjoint. Finally, our goal $\goal$ directly corresponds to a partial assignment to $\BNvars$ (unless our $\goal$ is self-contradictory,
requiring $q\wedge q'$ for some $q,q' \in \fluents_{i}$.)

\section{Belief States}
\label{ss:primal}

In this section, we explain our representation of, and reasoning
about, belief states. We first explain how probabilistic belief states
are represented as time-stamped BNs, then we explain how those BNs are
encoded and reasoned about in the form of weighted CNF formulas. This representation of belief states by weighted CNFs is then illustrated on the belief state from our running example in Figure~\ref{fig:example2}. 
We finally provide the details about how this works in \pff.

\subsection{Bayesian Networks}
\label{ss:primal1}

\pff\ performs a forward search in a space of belief states. The search states are belief states (that is, probability distributions over the world states $w$), and the search is restricted to belief states reachable
from the initial belief state $\initprob$ through some sequences of
actions $\seqactions$.
A key decision one should make is the actual representation of the
belief states. Let $\initprob$ be our initial belief state captured by
the BN $\initBN$, and let $\belief_{\seqactions}$ be a belief state
resulting from applying to $\initprob$ a sequence of actions
$\seqactions$. One of the well-known problems in the area of
decision-theoretic planning is that the description of
$\belief_{\seqactions}$ directly over the state variables $\BNvars$ becomes less and less structured as the number of (especially stochastic) actions in $\seqactions$ increases.  To overcome this limitation, we represent
belief states $\belief_{\seqactions}$ as a BN
$\BN_{\belief_{\seqactions}}$ that {\em explicitly} captures the
sequential application of $\seqactions$ starting from $\initbelief$,
trading the representation size for the cost of inference, compared to representing belief states directly as distributions over world states. Below we formally specify the structure of such a BN
$\BN_{\belief_{\seqactions}}$, assuming that all the actions
$\seqactions$ are applicable in the corresponding belief states of
their application, and later showing that \pff\ makes sure this is
indeed the case. We note that these belief-state BNs are similar in
spirit and structure to those proposed in the AI literature for
verifying that a probabilistic plan achieves its goals with a certain
probability~\cite{dean:kanazawa:ci89,hanks:drew:aij94,Buridan}.

\begin{figure}[t]
\begin{center}
\begin{tabular}{c}
\begin{tabular}{cccc}
\begin{minipage}{1.5in}
{
\begin{tabular}{|@{\hspace{1pt}}c@{\hspace{3pt}}c@{\hspace{1pt}}|}
\hline
 { $r_{1}$} & { $r_{2}$}\\
 \hline
 { 0.9} & { 0.1}\\
\hline
\end{tabular}
}
\end{minipage}
& 
& 
\begin{minipage}{2in}
{ 
 \setlength{\extrarowheight}{1pt}
 \begin{tabular}{|@{\hspace{1pt}}c@{\hspace{1pt}}|@{\hspace{1pt}}c@{\hspace{1pt}}|c@{\hspace{3pt}}c@{\hspace{1pt}}|}
 \hline
  \multicolumn{2}{|c|}{}  & $r_{1}$ & $r_{2}$\\
 \hline
   \multicolumn{2}{|c|}{$\poutcome_{1} \vee \poutcome_{2}$} & 0 & 1\\
 \hline
 \multirow{2}{1cm}{$\poutcome_{3} \vee \poutcome_{1}'$} & $r_{1}$ & 1 & 0\\
 \cline{2-4}
 & $r_{2}$ & 0 & 1\\
 \hline
 \end{tabular}
%
}
\vspace{0.1cm}
\end{minipage}
&
\begin{minipage}{0.7in}
{
\begin{tabular}{|@{\hspace{1pt}}c@{\hspace{1pt}}|@{\hspace{1pt}}c@{\hspace{3pt}}c@{\hspace{1pt}}|}
\hline
 & $r_{1}$ & $r_{2}$\\
 \hline
 $r_{1}$ & 1 &  0 \\
 \hline
 $r_{2}$ & 1 &  0 \\
\hline
\end{tabular}
}
\end{minipage}
\\
\multicolumn{4}{l}{
\vspace{0.1cm}
  \def\objectsizestyle{\scriptstyle}
      \xymatrix @R=10pt@C=40pt{
        *+[F-,]{R\tstamp{0}} \ar[dd] \ar[rr] \ar[dr] & & *+[F-,]{R\tstamp{1}} \ar[rr] & & *+[F-,]{R\tstamp{2}}\\
        & *+[F-,]{Y\tstamp{1}} \ar[ur] \ar[dr] 
        & \mbox{
        \begin{minipage}{1in}
{
\begin{tabular}{|@{\hspace{1pt}}c|@{\hspace{3pt}}c@{\hspace{3pt}}c@{\hspace{3pt}}c@{\hspace{3pt}}c@{\hspace{1pt}}|}
\hline
 & $\poutcome_{1}$ &  $\poutcome_{2}$ & $\poutcome_{3}$ &  $\poutcome_{1}'$\\
 \hline
 $r_{1}\wedge b_{1}$ & 0.7 & 0.2 & 0.1 & 0\\
 \hline
 ${\mathrm{othrw}}$ & 0 & 0 & 0 & 1\\
\hline
\end{tabular}
}
\end{minipage}
        }\\
        *+[F-,]{B\tstamp{0}} \ar[rr] \ar[ur]& & *+[F-,]{B\tstamp{1}} \ar[rr] & & *+[F-,]{B\tstamp{2}}
        \\
        }
}
\\
\begin{minipage}{1.5in}
{
\begin{tabular}{|@{\hspace{1pt}}c@{\hspace{1pt}}|@{\hspace{1pt}}c@{\hspace{3pt}}c@{\hspace{1pt}}|}
\hline
 & $b_{1}$ & $b_{2}$\\
 \hline
 $r_{1}$ & $0.7$ & 0.3\\
 \hline
 $r_{2}$ & $0.2$ & 0.8\\
\hline
\end{tabular}
}
\end{minipage}
& \multicolumn{2}{c}{
\begin{minipage}{1.6in}
{ \setlength{\extrarowheight}{1pt}
 \begin{tabular}{|@{\hspace{1pt}}c@{\hspace{1pt}}|@{\hspace{1pt}}c@{\hspace{1pt}}|c@{\hspace{3pt}}c@{\hspace{1pt}}|}
 \hline
  \multicolumn{2}{|c|}{}  & $b_{1}$ & $b_{2}$\\
 \hline
   \multicolumn{2}{|c|}{$\poutcome_{1}$} & 0 & 1\\
 \hline
 \multirow{2}{0.5cm}{$\neg \poutcome_{1}$} & $b_{1}$ & 1 & 0\\
 \cline{2-4}
 & $b_{2}$ & 0 & 1\\
 \hline
 \end{tabular}
}
\end{minipage}
} 
& 
\begin{minipage}{0.7in}
{
\begin{tabular}{|@{\hspace{1pt}}c@{\hspace{1pt}}|@{\hspace{1pt}}c@{\hspace{3pt}}c@{\hspace{1pt}}|}
\hline
 & $b_{1}$ & $b_{2}$\\
 \hline
 $b_{1}$ & $1$ & $0$\\
 \hline
 $b_{2}$ & 0 & 1\\
\hline
\end{tabular}
}
\end{minipage}
\end{tabular}
\end{tabular}
\end{center}
\caption{\label{fig:example2} Bayes network $\BN_{\belief_{\seqactions}}$ for our running Example~\ref{example1}-\ref{example2} and action sequence\newline $\seqactions = \langle move\mbox{-}b\mbox{-}right,move\mbox{-}left\rangle$.}
\end{figure}

Figure~\ref{fig:example2} illustrates the construction of
$\BN_{\belief_{\seqactions}}$ for our running example with
$\seqactions = \langle move\mbox{-}b\mbox{-}right,$
$move\mbox{-}left\rangle$. In general, let $\seqactions = \langle
a\tind{1},\ldots,a\tind{m} \rangle$ be a sequence of actions, numbered
according to their appearance on $\seqactions$. For $0 \leq t \leq m$,
let $\BNvars\tstamp{t}$ be a replica of our state variables $\BNvars$,
with $X\tstamp{t} \in \BNvars\tstamp{t}$ corresponding to $X \in
\BNvars$. The variable set of $\BN_{\belief_{\seqactions}}$ is the
union of $\BNvars\tstamp{0},\ldots,\BNvars\tstamp{m}$, plus some
additional variables that we introduce for the actions in
$\seqactions$.

First, for each $X\tstamp{0} \in \BNvars\tstamp{0}$, we set the
parents $Pa(X\tstamp{0})$ and conditional probability tables
$\cpt{X\tstamp{0}}$ to simply copy these of the state variable $X$ in
$\initBN$.  Now, consider an action $a\tind{t}$ from $\seqactions$,
and let $a\tind{t} = a$. For each such action we introduce a discrete
variable $Y\tstamp{t}$ that ``mediates'' between the variable layers
$\BNvars\tstamp{t-1}$ and $\BNvars\tstamp{t}$. The domain of
$Y\tstamp{t}$ is set to $Dom(Y\tstamp{t}) =
\bigcup_{e\in\effs(a)}{\Lambda(e)}$, that is, to the union of
probabilistic outcomes of all possible effects of $a$.  The parents of
$Y\tstamp{t}$ in $\BN_{\belief_{\seqactions}}$ are set to
\begin{equation}
\label{e:paY}
Pa(Y\tstamp{t}) = \bigcup_{e\in \effs(a)}\left\{ X\tstamp{i-1} 
\mid \con(e) \cap Dom(X) \neq \emptyset \right\},
\end{equation}
and, for each $\pi \in Dom(Pa(Y\tstamp{t}))$, we set
\begin{equation}
\label{e:cpt1p}
  T_{Y(t)}(Y\tstamp{i} = \poutcome \mid \pi) = 
  \begin{cases}
  \probe(\poutcome), & \con\left( e(\poutcome) \right) \subseteq \pi\\
  0, & {\mathrm{otherwise}}
  \end{cases},
\end{equation}
where $e(\poutcome)$ denotes the effect $e$ of $a$ such that $\poutcome \in \poutcomeset(e)$.

We refer to the set of all such variables $Y\tstamp{t}$ created for the actions of $\seqactions$ as $\cY$. Now, let $\effs_{X}(a) \subseteq \effs(a)$ be the probabilistic effects of $a$ that affect a variable $X \in \BNvars$. 
If $\effs_{X}(a) = \emptyset$, then we set $Pa(X\tstamp{t}) = \{X\tstamp{t-1}\}$, and 
 \begin{equation}
  \label{e:cpt1d}
  T_{X(t)}(X\tstamp{t} = x \mid X\tstamp{t-1} = x') = 
  \begin{cases}
     1, & x = x',\\
     0, & {\mathrm{otherwise}}
  \end{cases}.
\end{equation}
Otherwise, if $\effs_{X}(a) \neq\emptyset$, let $x_{\poutcome} \in Dom(X)$ be the value provided to $X$ by $\poutcome \in \poutcomeset(e)$, $e \in \effs_{X}(a)$. 
Recall that the outcomes of effects $\effs(a)$ are {\em all} mutually exclusive. Hence, we set $Pa(X\tstamp{t}) = \{X\tstamp{t-1},Y\tstamp{t-1}\}$, and
 \begin{equation}
  \label{e:cpt2p}
  T_{X\tstamp{i}}(X\tstamp{i} = x \mid X\tstamp{i-1} = x',Y\tstamp{i-1} = \poutcome ) =
  \begin{cases}
     1, & e(\poutcome) \in E_{X}(a) \;\wedge\; x = x_{\poutcome},\\
     1, & e(\poutcome) \not\in E_{X}(a) \;\wedge\; x = x',\\
     0, & {\mathrm{otherwise}}
  \end{cases},
\end{equation}
where $e(\poutcome)$ denotes the effect responsible for the outcome $\poutcome$.

It is not hard to verify that Equations~\ref{e:cpt1p}-\ref{e:cpt2p} capture the frame axioms and probabilistic semantics of our actions. 
In principle, this accomplishes our construction of $\BN_{\belief_{\seqactions}}$ over the variables $\BNvars_{\belief_{\seqactions}} = \cY\bigcup_{t=0}^{m}{\BNvars\tstamp{t}}$. We note, however, that  the mediating variable $Y\tstamp{t}$ are really needed only for truly probabilistic actions. Specifically, if
$a\tind{t}$ is a deterministic action $a$, let $\effs_{X}(a) \subseteq \effs(a)$ be the conditional effects
of $a$ that add and/or delete propositions associated with the domain
of a variable $X \in \BNvars$. If $\effs_{X}(a) = \emptyset$, then we  set $Pa(X\tstamp{t}) = \{X\tstamp{t-1}\}$, and $T_{X(t)}$ according to Equation~\ref{e:cpt1d}.
Otherwise, we set
\begin{equation}
\label{e:parE}
Pa(X\tstamp{t}) \;=\; \{X\tstamp{t-1}\}\hspace{-0.2cm}\bigcup_{e\in \effs_{X}(a)}\hspace{-0.3cm}\left\{ X'\tstamp{t-1} 
\mid \con(e) \cap Dom(X) \neq \emptyset \right\},
\end{equation}
and specify $T_{X\tstamp{t}}$ as follows. Let $x_{e} \in Dom(X)$ be
the value that (the only deterministic outcome of) the effect $e \in \effs_{X}(a)$ provides to $X$.  For each
$\pi \in Dom(Pa(X\tstamp{t}))$, if there exists $e \in \effs_{X}(a)$
such that $\con(e) \subseteq \pi$, then we set 
 \begin{equation}
  \label{e:cpt2d1}
  T_{X(t)}(X\tstamp{t} = x \mid \pi ) = \begin{cases}
     1, & x = x_{e},\\
     0, & {\mathrm{otherwise}} 
  \end{cases}
\end{equation}
Otherwise, we set
 \begin{equation}
  \label{e:cpt2d2}
  T_{X(t)}(X\tstamp{t} = x \mid \pi ) = \begin{cases}
     1, &  x = \pi[X\tstamp{t-1}],\\
     0, & {\mathrm{otherwise}} 
     \end{cases} 
   \end{equation}
Due to the self-consistency of the action, it is not hard to verify
that Equations~\ref{e:cpt2d1}-\ref{e:cpt2d2} are consistent, and, together
with Equation~\ref{e:cpt1d}, capture the semantics of the conditional
deterministic actions. This special treatment of deterministic actions
is illustrated in Figure~\ref{fig:example2} by the direct dependencies
of $\BNvars\tstamp{2}$ on $\BNvars\tstamp{1}$.

\begin{proposition}
\label{prop:bn-correct}
   Let $(\actions,\initBN,\goal,\goalprob)$ be a probabilistic
   planning problem, and $\seqactions$ be an $m$-step sequence of
   actions applicable in $\initprob$. Let $Pr$ be the joint
   probability distribution induced by $\BN_{\belief_{\seqactions}}$
   on its variables $\BNvars_{\belief_{\seqactions}}$. The belief
   state $\belief_{\seqactions}$ corresponds to the marginal
   distribution of $Pr$ on $\BNvars\tstamp{m}$, that is:
   $\belief_{\seqactions}(\BNvars) = Pr(\BNvars\tstamp{m})$, and if
   $G\tstamp{m}$ is a partial assignment provided by $G$ to
   $\BNvars\tstamp{m}$, then the probability
   $\belief_{\seqactions}(G)$ that $\seqactions$ achieves $G$ starting
   from $\initbelief$ is equal to $Pr(G\tstamp{m})$.
\end{proposition}

As we already mentioned, our belief-state BNs are constructed along
the principles outlined and used
by~\citeA{dean:kanazawa:ci89},~\citeA{hanks:drew:aij94}, and~\citeA{Buridan}, and thus the
correctness of Proposition~\ref{prop:bn-correct} is immediate from these previous results.  At this point, it is worth bringing attention to
the fact that all the variables in
$\BNvars\tstamp{1},\ldots,\BNvars\tstamp{m}$ 
are completely deterministic. Moreover, the CPTs of all the variables
 of $\BN_{\belief_{\seqactions}}$ are all {\em compactly representable} due to either a low number of parents, or some local structure induced by a large amount of
context-specific independence, or both. This compactness of the CPTs in $\BN_{\belief_{\seqactions}}$ is implied by the compactness of the STRIPS-style specification of the  planning actions. By exploiting this compactness of the action specification, the  size of the $\BN_{\belief_{\seqactions}}$'s description can be kept linear in the size of the input and the number of actions in $\seqactions$. 
%

\begin{proposition}
\label{prop:sizecomp}
Let $(\actions,\initBN,\goal,\goalprob)$ be a probabilistic planning
problem described over $k$ state variables, and $\seqactions$ be an $m$-step sequence of actions from $\actions$. Then, we have $|\BN_{\belief_{\seqactions}}| = O(|\initBN| + m\alpha(k+1))$
  where $\alpha$ is the largest
description size of an action in $\actions$.
\end{proposition}

The proof of Proposition~\ref{prop:sizecomp}, as well as the proofs of other formal claims in the paper, are relegated to Appendix~\ref{ss:proofs}, pp.~\pageref{prop:sizecompA}.

\subsection{Weighted CNFs}
\label{ss:dual}

Given the representation of belief states as BNs, next we should
select a mechanism for reasoning about these BNs. In general,
computing the probability of a query in BNs is known to be
\sharpp-complete~\cite{roth:aij-96}.  In addition, it is not hard to
verify, using an analysis similar to the ones
of~\citeA{darwiche:dbns} and~\citeA{brafman:domshlak:aaai06}, that the networks
arising in our work will typically exhibit large tree-width. While
numerous exact algorithms for inference with BNs have been proposed in
the literature~\cite{Darwiche:aij00,dechter:aij-99,zhang:poole:cai-94}, the
classical algorithms do not scale well on large networks
exhibiting high tree-width.
On the positive side, however, an observation that guides some recent
advances in the area of probabilistic reasoning is that real-world
domains typically exhibit a significant degree of deterministic
dependencies and context-specific independencies between the problem
variables. Targeting this property of practical BNs already resulted
in powerful inference
techniques~\cite{Chavira:Darwiche:ijcai-05,sang:etal:aaai-05}. The
general principle underlying these techniques is to
%
\begin{enumerate}[(i)]
\item Compile a BN $\BN$ into a {\em weighted propositional logic
    formula} $\cnf{\BN}$ in CNF, and
 \item Perform an efficient {\em weighted model counting} for
   $\cnf{\BN}$ by reusing (and adapting) certain techniques that
   appear powerful in enhancing backtracking DPLL-style search for
   SAT.
\end{enumerate}

One observation we had at the early stages of developing \pff\ is that
the type of networks and type of queries we have in our problems make
this machinery for solving BNs by weighted CNF model counting very
attractive for our needs.
First, in Section~\ref{ss:primal1} we have
already shown that the BNs representing our belief states exhibit a
large amount of both deterministic nodes and context-specific
independence. Second, the queries of our interest correspond to
computing probability of the ``evidence'' $G\tstamp{m}$ in
$\BN_{\belief_{\seqactions}}$, and this type of query has a clear
interpretation in terms of model
counting~\cite{sang:etal:aaai-05}. Hence, taking this route in \pff,
we compile our belief state BNs to weighted CNFs following the
encoding scheme proposed by~\citeA{sang:etal:aaai-05}, and answer
probabilistic queries using Cachet~\cite{sang:etal:sat-04}, one of the
most powerful systems to date for exact weighted model counting in CNFs.

In general, the weighted CNFs and the weights of such formulas are specified as
follows. Let $\cV = \{V_{1},\dots,$ $V_{n}\}$ be a set of
propositional variables with $Dom(V_{i}) = \{v_{i},\neg v_{i}\}$, and
let $\weight : \bigcup_{i}Dom(V_{i}) \rightarrow \nnreals$ be a
non-negative, real-valued {\em weight} function from the literals of
$\cV$. For any partial assignment $\pi$ to $\cV$, the weight
$\weight(\pi)$ of this assignment is defined as the product of its literals' weights, that is, $\weight(\pi) =
\prod_{l\in\pi}{\weight(l)}$. Finally, a propositional logic formula
$\phi$ is called {\em weighted} if it is defined over such a weighted set of propositional variables.  For any weighted formula $\phi$ over
$\cV$, the weight $\weight(\phi)$ is defined as the sum of the weights of all the complete assignments to $\cV$ satisfying 
$\phi$, that is,
\[
\weight(\phi) = \sum_{\pi \in Dom(\cV)}{\weight(\pi)\delta\left(
\pi \models \phi
\right)},
\]
where $Dom(\cV) = \times_{i}Dom(V_{i})$. For instance, if for all
variables $V_{i}$ we have $\weight(v_{i}) = \weight(\neg v_{i}) = 1$,
then $\weight(\phi)$ simply stands for the number of complete
assignments to $\cV$ that satisfy $\phi$.

Given an initial belief state BN $\initBN$, and a sequence of actions
$\seqactions = \langle a\tind{1},\ldots,a\tind{m}\rangle$ applicable
in $\initbelief$, here we describe how the weighted CNF encoding
$\phi(\BN_{\belief_{\seqactions}})$ (or $\phi(\belief_{\seqactions})$,
for short) of the belief state $\belief_{\seqactions}$ is built and
used in \pff. 
First, we formally specify the generic scheme introduced
by~\citeA{sang:etal:aaai-05} for encoding a BN $\BN$ over variables
$\BNvars$ into a weighted CNF $\phi(\BN)$. 
The encoding formula $\phi(\BN)$ contains two sets of
     variables. First, for each variable $Z \in \BNvars$ and each
     value $z\in Dom(Z)$, the formula $\phi(\BN)$ contains a {\em
     state proposition} with literals $\{z,\neg z\}$, weighted as
     $\weight(z) = \weight(\neg z) = 1$.  These state propositions act
     in $\phi(\belief_{\seqactions})$ as regular SAT propositions.
     Now, for each variable $Z \in \BNvars_{\belief_{\seqactions}}$, let  $Dom(Z) = \{z_{1},\dots,z_{k}\}$ be an arbitrary fixed ordering of $Dom(Z)$.
      Recall that each row $\cpt{Z}[i]$ in the CPT of $Z$ corresponds to an assignment 
      $\zeta_{i}$ (or a set of such assignments) to $Pa(Z)$. Thus, the number of rows in 
      $\cpt{Z}$ is upper bounded by the number of different assignments to $Pa(Z)$, but 
      (as it happens in our case) it can be significantly lower if the dependence of $Z$ on $Pa(Z)$
      induces a substantial local structure.
      Following the ordering of $Dom(Z)$ as above, the entry $\cpt{Z}[i,j]$ 
      contains the conditional probability of $Pr(z_{j}\mid\zeta_{i})$.
      For every CPT entry $\cpt{Z}[i,j]$ but the last one (i.e., $\cpt{Z}[i,k]$), the formula
      $\phi(\BN)$ contains a {\em chance proposition} with literals 
      $\{\langle z_{j}^{i} \rangle, \neg \langle z_{j}^{i} \rangle\}$. 
      These chance variables aim at capturing the probabilistic information from
      the CPTs of $\BN_{\belief_{\seqactions}}$. Specifically, 
    the weight of the literal $\langle z_{j}^{i} \rangle$ 
      is set to $Pr(z_{j} \mid \zeta_{i},\neg z_{1},\dots,\neg z_{j-1})$, that is to
      conditional probability that the entry is 
      true, given that the row is true, and no prior entry in the row is true:
      \begin{equation}
      \label{e:chancew}
      \begin{split}
        \weight \left(\langle z_{j}^{i} \rangle\right) &=
        \frac{\cpt{Z}[i,j]}{1-\sum_{k=1}^{j-1}{\cpt{Z}[i,k]}}\\
        \weight \left( \neg\langle z_{j}^{i} \rangle\right) &=
        1 - \weight \left(\langle z_{j}^{i} \rangle\right)
      \end{split}
      \end{equation}

Considering the clauses of $\phi(\BN)$,  for each variable $Z \in \BNvars$, and each CPT entry 
    $\cpt{Z}[i,j]$, the formula $\phi(\BN)$ contains a clause 
     \begin{equation}
     \label{e:cptclause}
      \left( \zeta_{i} \wedge \neg \langle z_{1}^{i} \rangle \wedge \dots \wedge \neg \langle z_{j-1}^{i} 
      \rangle
      \wedge \langle  z_{j}^{i} \rangle\right) \implication z_{j},
    \end{equation}
    where $\zeta_{i}$ is a conjunction of the literals forming the assignment 
    $\zeta_{i} \in Dom(Pa(Z))$. These clauses ensure that the weights of the complete assignments 
    to the variables of $\phi(\BN)$ are equal to the probability of the 
    corresponding atomic events as postulated by the BN $\BN$. 
    To illustrate the construction in Equations~\ref{e:chancew}-\ref{e:cptclause}, let boolean variables $A$ and $B$ be the parents of a  ternary variable $C$ (with $Dom(C)=\{C_{1},C_{2},C_{3}\}$) in some BN, and let $Pr(C_{1}|A,\neg B) = 0.2$, $Pr(C_{2}|A,\neg B) = 0.4$, and $Pr(C_{3}|A,\neg B) = 0.4$. Let the raw  corresponding to the assignment $A,\neg B$ to $Pa(C)$ be the $i$-th row  of the CPT $\cpt{C}$. In the encoding of this BN, the first two  entries of this raw of $\cpt{C}$ are captured by a pair of respective  chance propositions $\langle C_{1}^{i} \rangle$, and $\langle C_{2}^{i} \rangle$. According to Equation~\ref{e:chancew}, the weights of these propositions are set to  $\weight \left(\langle C_{1}^{i} \rangle\right) = 0.2$, and $\weight \left(\langle C_{1}^{i} \rangle\right) = \frac{0.4}{1-0.2} = 0.5$. Then, according to Equation~\ref{e:cptclause}, the encoding contains three clauses
    \[
      \begin{split}
      &\left( \neg A \vee B \vee \neg\langle C_{1}^{i} \rangle \vee C_{1}\right)\\
      &\left( \neg A \vee B \vee \langle C_{1}^{i} \rangle \vee \neg\langle C_{2}^{i} \rangle \vee C_{2}\right)\\
      &\left( \neg A \vee B \vee \langle C_{1}^{i} \rangle \vee \langle C_{2}^{i} \rangle \vee C_{3}\right)\\
      \end{split}
    \]
    
Finally, for each variable $Z \in \BNvars$, the formula 
    $\phi(\BN)$ contains a standard set of clauses encoding
    the ``exactly one'' relationship between the state propositions capturing the value 
    of $Z$. This accomplishes the encoding of $\BN$ into $\phi(\BN)$.
In the next Section~\ref{ss:ex-encoding} we illustrate this encoding on the belief state BN from our running example.



\begin{figure}[t]
\vspace{-0.0cm}
\begin{tabbing}
xx\=xx\=xx\=xx\=xx\=xx\= \kill
{\bf procedure} $\basicWMC(\phi)$\\
\>  {\bf if} $\phi = \emptyset$ {\bf return} 1\\
\> {\bf if} $\phi$ has an empty clause {\bf return} 0\\
\> {\bf select} a variable $V \in \phi$\\
\>\> {\bf return} $\basicWMC(\phi|_{v})\cdot \weight(v)\;+ \;\basicWMC(\phi|_{\neg v})\cdot \weight(\neg v)$
\end{tabbing}
\caption{\label{fig:bwmc}Basic DPPL-style weighted model counting.}
\end{figure}

The weighted CNF encoding $\phi({\belief_{\seqactions}})$ of the belief state BN $\BN_{\belief_{\seqactions}}$ provides the input to a weighted model counting
procedure. A simple recursive DPPL-style procedure $\basicWMC$
underlying Cachet~\cite{sang:etal:sat-04} is depicted in
Figure~\ref{fig:bwmc}, where the formula $\phi|_{v}$ is obtained from
$\phi$ by setting the literal $v$ to true.
Theorem~3 by~\citeA{sang:etal:aaai-05} shows that if $\phi$ is a
weighted CNF encoding of a BN $\BN$, and $Pr(Q|E)$ is a general query
with respect to $\BN$, query $Q$, and evidence $E$, then we have:
\begin{equation}
\label{e:pqe}
Pr(Q|E) = \frac{\basicWMC(\phi \wedge Q \wedge E)}{\basicWMC(\phi \wedge E)},
\end{equation}
where query $Q$ and evidence $E$ can in fact be arbitrary formulas in propositional logic. Note that, in a special (and very relevant to us) case of empty evidence, Equation~\ref{e:pqe} reduces to $Pr(Q) = \basicWMC(\phi \wedge Q)$, that is, a single call to the $\basicWMC$ procedure. Corollary~\ref {prop:wcnf-correct} is then immediate from our Proposition~\ref{prop:bn-correct} and Theorem~3 by~\citeA{sang:etal:aaai-05}.

\begin{corollary}
\label{prop:wcnf-correct}
   Let $(\actions,\initprob,\goal,\goalprob)$ be a probabilistic planning task with a BN $\initBN$ describing $\initprob$, and $\seqactions$ be an $m$-step sequence of actions applicable in $\initprob$. The probability $\belief_{\seqactions}(G)$ that $\seqactions$ achieves $G$ starting from $\initprob$ is given by:
      \begin{equation}
     \belief_{\seqactions}(G) = {\mathsf{WMC}}\left(\ourf(\belief_{\seqactions})\wedge G(m)\right),
   \end{equation}
where $G(m)$ is a conjunction of the goal literals time-stamped with the time endpoint $m$ of $\seqactions$.
\end{corollary}


\subsection{Example: Weighted CNF Encoding of Belief States}
\label{ss:ex-encoding}

We now illustrate the generic BN-to-WCNF encoding scheme of~\citeA{sang:etal:aaai-05} on the belief state BN $\BN_{\belief_{\seqactions}}$ from our running example in Figure~\ref{fig:example2}.

For $0 \leq i \leq 2$, we introduce time-stamped state propositions $r_{1}(i),r_{2}(i),b_{1}(i),b_{2}(i)$. Likewise, we introduce four state propositions $\poutcome_{1}(1),\poutcome_{2}(1),\poutcome_{3}(1),\poutcome_{1}'(1)$ corresponding to the values of the variable $Y\tstamp{1}$.
The first set of clauses in $\ourf(\belief_{\seqactions})$ ensure the ``exactly one'' relationship between the state propositions capturing the value of a variable in $\BN_{\belief_{\seqactions}}$:
\begin{equation}
\label{e:encoding1}
\begin{split}
  & \left( \poutcome_{1}(1) \vee \poutcome_{2}(1) \vee \poutcome_{3}(1) \vee \poutcome_{1}'(1)\right),\\
  1 \leq i < j & \leq 4\;: \\
  &\left( \neg y_{i}(1) \vee \neg y_{j}(1)\right),\\
  0 \leq i & \leq 2\;: \\
  & \left( r_{1}(i) \vee r_{2}(i) \right),\;
    \left( \neg r_{1}(i) \vee \neg r_{2}(i) \right)\\
  & \left( b_{1}(i) \vee b_{2}(i) \right),\;
    \left( \neg b_{1}(i) \vee \neg b_{2}(i) \right)
\end{split}
\end{equation}

Now we proceed with encoding the CPTs of $\BN_{\belief_{\seqactions}}$. The root nodes have only one row in their CPTs so their chance propositions can be identified with the corresponding state variables~\cite{sang:etal:aaai-05}.
Hence, for the root variable $R\tstamp{0}$ we need neither additional clauses nor special chance propositions, but the state proposition $r_{1}(0)$ of $\phi(\belief_{\seqactions})$ is treated as a chance proposition with $\weight\left(r_{1}(0)\right) = 0.9$. 

Encoding of the variable $B\tstamp{0}$ is a bit more involved. The CPT $\cpt{B\tstamp{0}}$ contains two (content-wise different) rows corresponding to the ``given $r_{1}$'' and ``given $r_{2}$'' cases, 
and both these cases induce a non-deterministic dependence of $B\tstamp{0}$ on $R\tstamp{0}$. 
To encode the content of $\cpt{B\tstamp{0}}$ 
we introduce two chance variables 
$\langle b_{1}(0)^{1}\rangle$ and $\langle b_{1}(0)^{2}\rangle$ with $\weight(\langle b_{1}(0)^{1}\rangle) = 0.7$ and $\weight(\langle b_{1}(0)^{2}\rangle) = 0.2$. 
The positive literals of $\langle b_{1}(0)^{1}\rangle$ and $\langle b_{1}(0)^{2}\rangle$ 
capture the events ``$b_{1}$ given $r_{1}$'' and
``$b_{1}$ given $r_{2}$'', while the negations $\neg \langle b_{1}(0)^{1}\rangle$ and $\neg \langle b_{1}(0)^{2}\rangle$ capture the complementary events ``$b_{2}$ given $r_{2}$'' and ``$b_{2}$ given $r_{2}$'', respectively.
Now consider the ``given $r_{1}$'' row in $T_{B\tstamp{0}}$. To encode this row, we need $\ourf(\belief_{\seqactions})$ to contain $\left( r_{1}(0) \wedge \langle b_{1}(0)^{1}\rangle \right) \implication b_{1}(0)$ and  $\left( r_{1}(0) \wedge \neg \langle b_{1}(0)^{1}\rangle\right) \implication b_{2}(0)$. Similar encoding is required for the row  ``given $r_{2}$'', and thus the encoding of $T_{B\tind{0}}$ introduces four additional clauses:
\begin{equation}
\label{e:encoding2}
\begin{split}
 & \left( \neg r_{1}(0) \vee \neg \langle b_{1}(0)^{1}\rangle \vee b_{1}(0) \right),\;
 \left( \neg r_{1}(0) \vee \langle b_{1}(0)^{1}\rangle \vee b_{2}(0)\right)\\
 & \left(  \neg r_{2}(0) \vee \neg \langle b_{1}(0)^{2}\rangle \vee b_{1}(0) \right),\;
 \left( \neg r_{2}(0) \vee \langle b_{1}(0)^{2}\rangle \vee b_{2}(0)\right)
\end{split}
\end{equation}

Having finished with the $\initBN$ part of $\BN_{\belief_{\seqactions}}$, we proceed with encoding the variable $Y\tstamp{1}$ corresponding to the probabilistic action $move\mbox{-}b\mbox{-}right$. 
To encode the first row of $\cpt{Y\tstamp{1}}$ we introduce three chance propositions $\langle\poutcome_{1}(1)^{1}\rangle$, $\langle\poutcome_{2}(1)^{1} \rangle$, and $\langle\poutcome_{3}(1)^{1} \rangle$; in general, no chance variables are needed for the last entries of the CPT rows. 
The weights of these chance propositions are set according to Equation~\ref{e:chancew} to
$\weight\left(\langle\poutcome_{1}(1)^{1} \rangle\right)=0.7$,  $\weight\left(\langle\poutcome_{2}(1)^{1} \rangle\right)=\frac{0.2}{1-0.7} = 0.6(6)$, and $\weight\left(\langle\poutcome_{3}(1)^{1} \rangle\right)=\frac{0.1}{1-0.9} = 0.1$.
Using these chance propositions, we add to $\phi({\belief_{\seqactions}})$ four clauses as in Equation~\ref{e:cptclause}, notably the first four clauses of Equation~\ref{e:encoding22} below. 

Proceeding the second row of $\cpt{Y\tstamp{1}}$, observe that the value of $R\tstamp{0}$ and $B\tstamp{0}$ in this case fully determines the value of $Y\tstamp{1}$. This deterministic dependence can be encoded 
without using any chance propositions using the last two clauses in Equation~\ref{e:encoding22}. 
\begin{equation}
\label{e:encoding22}
\begin{split}
 & \left( \neg r_{1}(0) \vee  \neg b_{1}(0) \vee \neg \langle\poutcome_{1}(1)^{1}\rangle \vee \poutcome_{1}(1)\right),\\
 & \left( \neg r_{1}(0) \vee  \neg b_{1}(0) \vee \langle\poutcome_{1}(1)^{1}\rangle \vee \neg \langle\poutcome_{2}(1)^{1}\rangle \vee \poutcome_{2}(1)\right),\\
 & \left( \neg r_{1}(0) \vee  \neg b_{1}(0) \vee \langle\poutcome_{1}(1)^{1}\rangle \vee \langle\poutcome_{2}(1)^{1}\rangle \vee \neg \langle\poutcome_{3}(1)^{1}\rangle \vee \poutcome_{3}(1)\right),\\
 & \left( \neg r_{1}(0) \vee  \neg b_{1}(0) \vee \langle\poutcome_{1}(1)^{1}\rangle \vee \langle\poutcome_{2}(1)^{1}\rangle \vee \langle\poutcome_{3}(1)^{1}\rangle \vee \poutcome_{1}'(1)\right),\\
 & \ \\
 & \left( r_{1}(0) \vee \neg \poutcome_{1}'(1) \right),\; \left( b_{1}(0) \vee \neg \poutcome_{1}'(1) \right)
\end{split}
\end{equation}

Using the state/chance variables introduced for $R\tind{0}$, $B\tind{0}$, and $Y\tstamp{1}$, we encode the CPTs of $R\tstamp{1}$ and $B\tstamp{1}$ as:
\begin{equation}
\label{e:encoding3}
\begin{split}
 R\tstamp{1}: & \left( \neg \poutcome_{1}(1) \vee r_{2}(1)\right),\;\left( \neg \poutcome_{2}(1) \vee r_{2}(1) \right),\\
            & \left( \neg \poutcome_{3}(1) \vee \neg r_{1}(0) \vee r_{1}(1)\right),\;\left( \neg \poutcome_{1}'(1) \vee \neg r_{1}(0) \vee r_{1}(1)\right),\\
            & \left( \neg \poutcome_{3}(1) \vee \neg r_{1}(0) \vee r_{1}(1)\right),\;\left( \neg \poutcome_{1}'(1) \vee \neg r_{2}(0) \vee r_{2}(1)\right)\\
 B\tstamp{1}: & \left( \neg \poutcome_{1}(1) \vee b_{2}(1)\right),\\
            & \left(  \poutcome_{1}(1) \vee \neg b_{1}(0) \vee b_{1}(1)\right),\\
            & \left(  \poutcome_{1}(1) \vee \neg b_{2}(0) \vee b_{2}(1)\right)
\end{split}
\end{equation}
Since the CPTs of both $R\tstamp{1}$ and $B\tstamp{1}$ are completely deterministic, their encoding as well is using no chance propositions. Finally, we encode the (deterministic) CPTs of $R\tstamp{2}$ and $B\tstamp{2}$ as:
\begin{equation}
\label{e:encoding4}
\begin{split}
 R\tstamp{2}: & \left( r_{1}(2)\right)\\
 B\tstamp{2}: & \left(  \neg b_{1}(1) \vee b_{1}(2)\right)\\
            & \left(  \neg b_{2}(1) \vee b_{2}(2)\right)
\end{split}
\end{equation}
where the unary clause $\left(r_{1}(2)\right)$ is a reduction of 
$\left(\neg r_{1}(1) \vee r_{1}(2)\right)$ and $\left(\neg r_{2}(1) \vee r_{1}(2)\right)$.
This accomplishes our encoding of $\phi({\belief_{\seqactions}})$.

\subsection{From \cff\ to \pff}
\label{ss:cff2pff}

Besides the fact that weighted model counting is attractive for the
kinds of BNs arising in our context, the weighted CNF representation
of belief states works extremely well with the ideas underlying
\cff~\cite{cff}. This was outlined in the introduction already; here
we give a few more details.

As stated, \cff\ does a forward search in a non-probabilistic belief
space in which each belief state corresponds to a set of world states
considered to be possible.  The main trick of \cff\ is the use of CNF
formulas for an implicit representation of belief states, where
formulas $\ourf(\seqactions)$ encode the semantics of executing action
sequence $\seqactions$ in the initial belief state. Facts known to be
true or false are inferred from these formulas. This computation of
only a partial knowledge constitutes a {\em lazy} kind of belief state
representation, in comparison to other approaches that use explicit
enumeration \cite{bonet:geffner:aips-00} or BDDs \cite{MBP} to fully
represent belief states. The basic ideas underlying \pff\ are: 
\begin{enumerate}[(i)]
\item Define time-stamped Bayesian Networks (BN) describing
probabilistic belief states (Section~\ref{ss:primal1} above).
\item Extend \cff's belief state CNFs to model these BN
(Section~\ref{ss:dual} above).
\item In addition to the SAT reasoning used by \cff, use weighted
model-counting to determine whether the probability of the (unknown)
goals in a belief state is high enough (directly below).
\item Introduce approximate probabilistic reasoning into \cff's
heuristic function (Section~\ref{ss:heuristic} below).
\end{enumerate}
In more detail, given a probabilistic planning task
$(\actions,\initprob,\goal,\goalprob)$, a belief state
$\belief_{\seqactions}$ corresponding to some applicable in
$\initprob$ $m$-step action sequence $\seqactions$, and a proposition
$q \in \fluents$, we say that $q$ is {\em known} in
$\belief_{\seqactions}$ if $\belief_{\seqactions}(q) = 1$, {\em
negatively known} in $\belief_{\seqactions}$ if
$\belief_{\seqactions}(q) = 0$, and {\em unknown} in
$\belief_{\seqactions}$, otherwise.  We begin with determining whether
each $q$ is known, negatively known, or unknown at time $m$.  Re-using
the \cff\ machinery, this classification requires up to two SAT tests
of $\ourf(\belief_{\seqactions}) \wedge \neg q(m)$ and
$\ourf(\belief_{\seqactions}) \wedge q(m)$, respectively. The
information provided by this classification is used threefold. First,
if a subgoal $g \in \goal$ is negatively known at time $m$, then we
have $\belief_{\seqactions}(\goal) = 0$.  On the other extreme, if all
the subgoals of $\goal$ are known at time $m$, then we have
$\belief_{\seqactions}(\goal) = 1$. Finally, if some subgoals of $G$
are known and the rest are unknown at time $m$, then we accomplish
evaluating the belief state $\belief_{\seqactions}$ by testing whether
\begin{equation}
\label{e:testseqacts}
\belief_{\seqactions}(\goal) = {\mathsf{WMC}}\left(\ourf(\belief_{\seqactions}) \wedge \goal(m) \right) \geq \goalprob.
\end{equation}
Note also that having the sets of all (positively/negatively) known
propositions at all time steps up to $m$ allows us to {\em significantly}
simplify the CNF formula $\ourf(\belief_{\seqactions}) \wedge G(m)$ by
inserting into it the corresponding values of known propositions.

After evaluating the considered action sequence $\seqactions$, if we
get $\belief_{\seqactions}(\goal) \geq \goalprob$, then we are
done. Otherwise, the forward search continues, and the actions that
are applicable in $\belief_{\seqactions}$ (and thus used to generate
the successor belief states) are actions whose preconditions are all
known in $\belief_{\seqactions}$.


\section{Heuristic Function}
\label{ss:heuristic}

The key component of any heuristic search procedure is the heuristic
function. The quality (informedness) and computational cost of that
function determine the performance of the search. The heuristic
function is usually obtained from solutions to a relaxation of the
actual problem of interest~\cite{Pearl.heuristics,russell-norvig}. In
classical planning, a successful idea has been to use a relaxation
that ignores the delete effects of the
actions~\cite{drew,bonet:geffner:ai-01,hoffmann:nebel:jair-01}. In
particular, the heuristic of the \ff\ planning system is based on the
notion of {\em relaxed plan}, which is a plan that achieves the goals
while assuming that all delete lists of actions are empty. The relaxed
plan is computed using a Graphplan-style~\cite{blum:furst:ai-97}
technique combining a forward chaining graph construction phase with a
backward chaining plan extraction phase. The heuristic value $h(w)$
that \ff\ provides to a world state $w$ encountered during the search
is the length of the relaxed plan from $w$. In \cff, this methodology
was extended to the setting of conformant planning under initial state
uncertainty (without uncertainty about action effects). Herein, we
extend \cff's machinery to handle probabilistic initial states and
effects. Section~\ref{ss:heuristic0} provides background on the
techniques used in \ff\ and \cff, then Sections~\ref{ss:heuristic1}
and~\ref{ss:heuristic2} detail our algorithms for the forward and
backward chaining phases in \pff, respectively. These algorithms for the two phases of the \pff\ heuristic computation are illustrated on our running example in Sections~\ref{ss:ex-prpg} and~\ref{ss:ex-extract}, respectively.

\subsection{\ff\ and \cff}
\label{ss:heuristic0}

We specify how relaxed plans are computed in \ff; we provide a coarse
sketch of how they are computed in \cff. The purpose of the latter is
only to slowly prepare the reader for what is to come: \cff's
techniques are re-used for \pff\ anyway, and hence will be described in
full detail as part of Sections~\ref{ss:heuristic1}
and~\ref{ss:heuristic2}.

Formally, relaxed plans in classical planning are computed as follows.
Starting from $w$, \ff\ builds a {\em relaxed planning graph\/} as a
sequence of alternating proposition layers $P(t)$ and action layers
$A(t)$, where $P(0)$ is the same as $w$, $A(t)$ is the set of all
actions whose preconditions are contained in $P(t)$, and $P(t+1)$ is
obtained from $P(t)$ by including the add effects (with fulfilled
conditions) of the actions in $A(t)$.  That is, $P(t)$ always contains
those facts that will be true if one would execute (the relaxed
versions of) all actions at the earlier layers up to $A(t-1)$.  The
relaxed planning graph is constructed either until it reaches a
propositional layer $P(m)$ that contains all the goals, or until the
construction reaches a fixpoint $P(t) = P(t+1)$ without reaching the
goals. The latter case corresponds to (all) situations in which a
relaxed plan does not exist, and because existence of a relaxed plan
is a necessary condition for the existence of a real plan, the state $w$
is excluded from the search space by setting $h(w)=\infty$. In the
former case of $G \subseteq P(m)$, a relaxed plan is a subset of
actions in $A(1),\ldots,A(m)$ that suffices to achieve the goals
(under ignoring the delete lists), and it can be extracted by a simple
backchaining loop: For each goal in $P(m)$, select an action in
$A(1),\ldots,A(m)$ that achieves this goal, and iterate the process by
considering those actions' preconditions and the respective effect
conditions as new subgoals.  The heuristic estimate $h(w)$ is then set
to the length of the extracted relaxed plan, that is, to the number of
actions selected in this backchaining process.

Aiming at extending the machinery of \ff\ to conformant planning, in
\cff,~\citeA{cff} suggested to extend the
relaxed planning graph with additional fact layers $uP(t)$ containing
the facts {\em unknown} at time $t$, and then to reason about when
such unknown facts become known in the relaxed planning graph.
As the complexity of this type of reasoning is prohibitive, \cff\
further relaxes the planning task by ignoring not only the delete
lists, but also all but one of the unknown conditions of each action
effect.  That is, if action $a$ appears in layer $A(t)$, and for
effect $e$ of $a$ we have $\con(e) \subseteq P(t) \cup uP(t)$ and
$\con(e) \cap uP(t) \neq \emptyset$, then $\con(e) \cap uP(t)$ is
arbitrarily reduced to contain exactly one literal, and reasoning is
done as if $\con(e)$ had this reduced form from the beginning.

This relaxation converts implications $(\bigwedge_{c \in \con(e) \cap
uP(t)} c(t)) \implication q(t+1)$
that the action effects induce between unknown propositions into their
2-projections that take the form of {\em binary implications} $c(t)
\implication q(t+1)$, for arbitrary $c \in \con(e) \cap uP(t)$.  Due
to the layered structure of the planning graph, the set of all these
binary implications $c(t) \implication q(t+1)$ can be seen as forming
a directed acyclic graph $\mbox{\sl Imp}$. Under the given
relaxations, this graph captures exactly all dependencies between the
truth of propositions over time. Hence, checking whether a proposition
$q$ becomes known at time $t$ can be done as follows. First, backchain
over the implication edges of $\mbox{\sl Imp}$ that end in $q(t)$, and
collect the set {\sl support}$(q(t))$ of leafs\footnote{Following the
\cff\ terminology, by ``leafs'' we refer to the nodes having zero
in-degree.} that are reached. Then, if $\Phi$ is the CNF formula
describing the possible initial states, test by a SAT check whether
\[
\Phi \rightarrow \bigvee_{l \in \mbox{\scriptsize \sl support}(q(t))} l
\]
This test will succeed if and only if at least one of the leafs in {\sl
support}$(q(t))$ is true in every possible initial state. Under the
given relaxations, this is the case if and only if, when applying all actions in
the relaxed planning graph, $q$ will always be true at time
$t$.\footnote{Note here that it would be possible to do a full SAT
check, without any 2-projection (without relying on $\mbox{\sl Imp}$),
to see whether $q$ becomes known at $t$.  However, as indicated above,
doing such a full check for every unknown proposition at every level
of the relaxed planning graph for every search state would very likely
be too expensive, computationally.}

The process of extracting a relaxed plan from the constructed
conformant relaxed planning graph is an extension of \ff's respective
process with machinery that selects actions responsible for relevant
paths in $\mbox{\sl Imp}$. The overall \cff\ heuristic machinery is
sound and complete for relaxed tasks, and yields a heuristic function
that is highly informative across a range of challenging
domains~\cite{cff}.

In this work, we adopt \cff's relaxations, ignoring the delete lists
of the action effects, as well as all but one of the propositions in
the effect's condition. Accordingly, we adopt the following notations
from \cff. Given a set of actions $A$, we denote by $\onecondrel$ any
function from $A$ into the set of all possible actions, such that
$\onecondrel$ maps each $a \in A$ to the action similar to $a$ but
with empty delete lists and with all but one conditioning propositions
of each effect removed; for $\onecondrel(a)$, we write
$a\onecondrel$. By $A\onecondrel$ we denote the action set obtained by
applying $\onecondrel$ to all the actions of $A$, that is,
$A\onecondrel = \left\{ a\onecondrel \mid a \in A \right\}$.  For an
action sequence $\seqactions$ we denote by $\seqactions\onecondrel$
the sequence of actions obtained by applying $\onecondrel$ to every
action along $\seqactions$, that is,
\[
   \seqactions\onecondrel = 
   \begin{cases}
   \langle\rangle, & \seqactions = \langle\rangle\\
   \langle a\onecondrel \rangle \cdot \seqactions'\onecondrel, & \seqactions = \langle a \rangle \cdot \seqactions'\\
   \end{cases}.
\]
For a probabilistic planning task
$(\actions,\initprob,\goal,\goalprob)$, the task
$(\actions\onecondrel,\initprob,\goal,\goalprob)$ is called a
relaxation of $(\actions,\initprob,\goal,\goalprob)$. Finally, if
$\seqactions\onecondrel$ is a plan for
$(\actions\onecondrel,\initprob,\goal,\goalprob)$, then $\seqactions$
is called a relaxed plan for $(\actions,\initprob,\goal,\goalprob)$.

In the next two sections we describe the machinery underlying the
\pff\ heuristic estimation. Due to the similarity between the
conceptual relaxations used in \pff\ and \cff, \pff\ inherits almost
all of \cff's machinery. Of course, the new contributions are those
algorithms dealing with probabilistic belief states and probabilistic
actions.


\subsection{Probabilistic Relaxed Planning Graphs}
\label{ss:heuristic1}

Like \ff\ and \cff, \pff\ computes its heuristic function in two
steps, the first one chaining forward to build a relaxed planning
graph, and the second step chaining backward to extract a relaxed plan. In
this section, we describe in detail \pff's forward chaining step,
building a {\em probabilistic relaxed planning graph} (or PRPG, for
short). In Section~\ref{ss:heuristic2}, we then show how one can
extract a (probabilistic) relaxed plan from the PRPG. We provide a
detailed illustration of the PRPG construction process on the basis of our running example; since the illustration is lengthy, it is moved to  a separate Section~\ref{ss:ex-prpg}.

The algorithms building a PRPG are quite involved; it is instructive
to first consider (some of) the key points before delving into the
details. The main issue is, of course, that we need to extend \cff's
machinery with the ability to determine when the goal set is
sufficiently {\em likely}, rather than when it is known to be true for
sure.  To achieve that, we must introduce into relaxed planning some
effective reasoning about both the probabilistic initial state, and
the effects of probabilistic actions. It turns out that such a
reasoning can be obtained by a certain {\em weighted} extension of the
implication graph. In a nutshell, if we want to determine how likely
it is that a fact $q$ is true at a time $t$, then we propagate certain weights 
backwards through the implication graph, starting in $q(t)$; the
weight of $q(t)$ is set to $1$, and the weight for any $p(t')$ gives
an estimate of {\em the probability of achieving $q$ at $t$ given that
$p$ holds at $t'$}. Computing this probability exactly would, of
course, be too expensive. Our estimation is based on {\em assuming
independence of the various probabilistic events involved.} This is a
choice that we made very carefully; we experimented widely with
various other options before deciding in favor of this technique.

Any simplifying assumption in the weight propagation constitutes, of
course, another relaxation, on top of the relaxations we already
inherited from \cff. The particularly problematic aspect of assuming
independence is that it is not an under-estimating technique.  The
actual weight of a node $p(t')$ -- the probability of achieving $q$ at
$t$ given that $p$ holds at $t'$ -- may be lower than our estimate. In
effect, the PRPG may decide wrongly that a relaxed plan exists: even
if we execute all relaxed actions contained in the successful PRPG,
the probability of achieving the goal by this execution may be less
than the required threshold. In other words, we lose the soundness
(relative to relaxed tasks) of the relaxed planning process.

We experimented with an alternative weight propagation method, based
on an opposite assumption, that the relevant probabilistic events
always co-occur, and that hence the weights must be propagated
according to simple maximization operations. This propagation method
yielded very uninformative heuristic values, and hence inacceptable
empirical behaviour of \pff, even in very simple benchmarks. In our
view, it seems unlikely that an under-estimating yet informative and
efficient weight computation exists. We further experimented with some
alternative non under-estimating propagation schemes, in particular
one based on assuming that the probabilistic events are completely
disjoint (and hence weights should be {\em added}); these schemes gave
better performance than maximization, but lagged far behind the
independence assumption in the more challenging benchmarks.

Let us now get into the actual algorithm building a PRPG. A coarse outline
of the algorithm is as follows. The PRPG is built in a layer-wise
fashion, in each iteration extending the PRPG, reaching up to time
$t$, by another layer, reaching up to time $t+1$. The actions in the
new step are those whose preconditions are known to hold at
$t$. Effects conditioned on unknown facts (note here the reduction of
effect conditions to a single fact) constitute new edges in the
implication graph. In difference to \cff, we don't obtain a single
edge from condition to add effect; instead, we obtain edges from the
condition to ``chance nodes'', where each chance node represents a
probabilistic outcome of the effect; the chance nodes, in turn, are
linked by edges to their respective add effects. The weights of the chance
nodes are set to the probabilities of the respective outcomes, 
the weights of all other nodes are set to $1$. These weights are ``static weights''
which are not ``dynamically'' modified by weight propagation; rather,
the static weights form an input to the propagation.

Once all implication graph edges are inserted at a layer, the
algorithm checks whether any new facts become known. This check is
done very much like the corresponding check in \cff, by testing
whether the disjunction of the support leafs for a proposition $p$ at
$t+1$ is implied by the initial state formula. The two differences to
\cff\ are: (1) Only leafs are relevant whose dynamic weight is $1$
(otherwise, achieving a leaf is not guaranteed to accomplish $p$ at
$t+1$). (2) Another reason for $p$ to become known may be that all
outcomes of an unconditional effect (or an effect with known
condition) result in achievement of $p$ at time $t+1$. We elegantly
formulate the overall test by a single implication test over support
leafs whose dynamic weight equals their own weight.

Like \ff's and \cff's algorithms, the PRPG process has two termination
criteria. The PRPG terminates positively if the goal probability is
high enough at time $t$; the PRPG terminates negatively if, from $t$
to $t+1$, nothing has changed that may result in a higher goal
propability at some future $t'$. The goal probability in a layer $t$
is computed based on weighted model counting over a formula derived
from the support leafs of all goals not known to be true. The criteria
for negative termination check: whether any new facts have become
known or unknown (not negatively known); whether any possibly relevant
new support leafs have appeared; and whether the goal probability has
increased. If neither is the case, then we can stop safely---if the
PRPG terminates unsuccessfully then we have a guarantee that there is
no relaxed plan, and that the corresponding belief is hence a dead
end.

\begin{figure}[thb]
{\small
\begin{tabbing}
{\bf procedure} \= \underline{\sf build-PRPG}$(\plan, \actions, \cinitial,
\goal, \goalprob, \onecondrel)$,\\
\> returns \= a Bool saying if there is a relaxed plan for the belief
state\\
\> \> given by $\plan = \langle a^{-m}, \dots, a^{-1} \rangle$, and\\
\> builds data structures from which a relaxed plan can be extracted\\
$\Phi := \cinitial$, $\mbox{\sl Imp} := \emptyset$\\
$P(-m) := \{ p \mid p \mbox{ is known in
  $\Phi$} \}$, $uP(-m) := \{ p \mid p \mbox{ is unknown in
  $\Phi$} \}$\\
{\bf for} \= $t := -m \dots -1$ {\bf do}\\
\> $A(t)$ := $\{ a^{t}|_{1}^+\} \cup NOOPS$\\
\> ${\mathsf{build\mbox{-}timestep}}(t, A(t))$\\
{\bf endfor}\\
$t := 0$\\
{\bf while} \getp$(t, \goal) < \goalprob$ {\bf do}\\
\> $A(t)$ := $\{ a|_{1}^+ \mid a \in \actions, \pre(a) \subseteq P(t)
\} \cup NOOPS$\\
\> ${\mathsf{build\mbox{-}timestep}}(t, A(t))$\\
\> {\bf if} \= $P(t+1) = P(t)$ and\\
\> \> $uP(t+1) = uP(t)$ and\\
\> \> $\forall p \in uP(t+1): uP(-m)\cap\support(p(t+1)) = uP(-m)\cap\support(p(t))$ and\\
\> \>  \getp$(t+1, \goal) =$ \getp$(t, \goal)$ {\bf then}\\
\> \> {\bf return} FALSE\\
\> {\bf endif}\\
\> $t$ := $t + 1$\\
{\bf endwhile}\\
$T$ := $t$, {\bf return} TRUE
\end{tabbing}}
\caption{\label{code:build-PRPG} Main routine for building a probabilistic relaxed planning graph (PRPG).}
\end{figure}

Let us get into the details. Figure~\ref{code:build-PRPG} depicts the
main routine for building the PRPG for a belief state
$\belief_{\seqactions}$.  As we already specified, the sets $P(t)$,
$uP(t)$, and $A(t)$ contain the propositions that are known to hold at
time $t$ (hold at $t$ with probability $1$), the propositions that are
unknown to hold at time $t$ (hold at $t$ with probability less than
$1$ but greater than $0$), and actions that are known to be applicable
at time $t$, respectively. The layers $t \geq 0$ of PRPG capture
applying the relaxed actions starting from
$\belief_{\seqactions}$. The layers $-m$ to $-1$ of PRPG correspond to
the $m$-step action sequence $\seqactions$ leading from the initial
belief state to the belief state in question
$\belief_{\seqactions}$. We inherit the latter technique from \cff; in
a sense, the PRPG ``reasons about the past''. This may look confusing
at first sight, but it has a simple reason. Imagine the PRPG starts at
level $0$ instead. Then, to check whether a proposition becomes known,
we have to do SAT tests regarding support leafs {\em against the
belief state formula, $\ourf(\belief_{\seqactions})$, instead of the
initial state formula} (similarly for weighted model counting to test
whether the goal is likely enough). Testing against
$\ourf(\belief_{\seqactions})$ is possible, but very expensive
computationally.\footnote{In \cff, this configuration is implemented
as an option; it significantly slows down the search in most domains,
and brings advantages only in a few cases.} The negative-index layers
chain the implication graph all the way back to the initial state, and
hence enable us to perform SAT tests against the -- typically much
smaller -- initial state formula.


Returning to Figure~\ref{code:build-PRPG}, the PRPG is initialized
with an empty implication set {\sl Imp}, $P(-m)$ and $uP(-m)$ are
assigned the propositions that are known and unknown in the initial
belief state, and a weighted CNF formula $\Phi$ is initialized with
$\cinitial$. $\Phi$ is the formula against which implication/weighted
model checking tests are run when asking whether a proposition becomes
known/whether the goal is likely enough. While the PRPG is built,
$\Phi$ is incrementally extended with further clauses to capture the behavior of different effect outcomes.

The {\bf for} loop builds the sets $P$ and $uP$ for the $\plan$'s time
steps $-m \dots -1$ by iterative invocation of the \buildts\ procedure
that each time expands PRPG by a single time level.  At each iteration
$-m \leq t \leq -1$, the sets $P(t+1)$ and $uP(t+1)$ are made to
contain the propositions that are known/unknown after applying the
relaxed version of the action $a\tind{t}\in\plan$ (remember that
$\plan = \langle a\tind{1},\ldots,a\tind{m} \rangle$). To simplify the
presentation, each action set $A(t)$ contains a set of dummy actions
$NOOPS$ that simply transport all the propositions from time layer $t$
to time layer $t+1$. More formally, $NOOPS = \left\{ \noop_{p} \mid p
\in \fluents \right\}$, where $\pre(\noop_{p}) = \emptyset$,
$\effs(\noop_{p}) = \{(\{p\},\{\poutcome\})\}$, and $\poutcome =
(1.0,\{p\},\emptyset)\})$.

The subsequent {\bf while} loop constructs the relaxed planning graph
from layer $0$ onwards by, again, iterative invocation of the
\buildts\ procedure. The actions in each layer $t \geq 0$ are
relaxations of those actions whose preconditions are known to hold at
time $t$ with certainty. This iterative construction is controlled by
two termination tests. First, if the goal is estimated to hold at
layer $t$ with probability higher than $\goalprob$, then we know that
a relaxed plan estimate can be extracted.
Otherwise, if the graph reaches a fix point, then we know that no
relaxed (and thus, no real) plan from $\initbelief$ exists. We
postpone the discussion of these two termination criteria,
and now focus on the time layer construction procedure \buildts.

\begin{figure}[htb]
{\small
\begin{tabbing}
{\bf procedure} \= \underline{\sf build-timestep}$(t, A)$,\\
\> builds $P(t+1)$, $uP(t+1)$, and the implication edges from $t$ to
$t+1$,\\
\> as induced by the action set $A$\\
\hspace{0.4cm} \= \hspace{0.25cm} \= \hspace{0.3cm} \= \hspace{0.3cm} \= \hspace{0.3cm} \= \hspace{0.3cm} \=\kill
$P(t+1) := P(t), uP(t+1) := \emptyset$\\
{\bf for} \= all effects $e$ of an action $a \in A$, 
$\con(e)\in P(t)\cup uP(t)$ {\bf do}\\
\> {\bf for} all $\poutcome \in \poutcomeset(e)$ {\bf do}\\
\> \> $uP(t+1) := uP(t+1) \cup add(\poutcome)$\\
\> \> introduce new fact $\newchancep{\poutcome(t)}$ with $\weight(\newchancep{\poutcome(t)}) = \probe(\poutcome)$\\
\> \> $\mbox{\sl Imp} := \mbox{\sl Imp} \cup \{ (\newchancep{\poutcome(t)},p(t+1)) \mid p \in \add(\poutcome)\}$\\
\> {\bf endfor}\\
\> {\bf if} \= $\con(e) \in uP(t)$ {\bf then}\\
\> \> $\mbox{\sl Imp} := \mbox{\sl Imp} \cup \bigcup_{\poutcome \in \poutcomeset(e)}\{ (\con(e)(t),\newchancep{\poutcome(t)}) \}$\\
\>  {\bf else}\\
\> \> $\Phi := \Phi \wedge \left(\vee_{\poutcome \in \poutcomeset(e)}{\poutcome(t)}\right) \wedge\bigwedge_{\poutcome,\poutcome' \in \poutcomeset(e)}\left(\neg \poutcome(t) \vee \neg \poutcome'(t)\right)$\\
\> {\bf endif}\\
{\bf endfor}\\
{\bf for} \= all $p \in uP(t+1)$ {\bf do}\\
\> ${\mathsf{build\mbox{-}w\mbox{-}impleafs}}(p(t+1),\mbox{\sl Imp})$\\
\> $\mbox{\sl support}(p(t+1)) := \{l \mid l \in \leafs{\mbox{\sl Imp}_{\rightarrow p(t+1)}} \wedge\subweight_{p(t+1)}(l) = \weight(l)\}$\\
\> {\bf if} $\Phi \rightarrow \bigvee_{l \in \mbox{\scriptsize \sl support}(p(t+1))} l$ {\bf then} $P(t+1) :=  P(t+1) \cup \{ p \}$ {\bf endif}\\
{\bf endfor}\\
$uP(t+1) := uP(t+1) \setminus P(t+1)$
\end{tabbing}}
\caption{\label{code:build-timestep}Building a time step of the PRPG.}
\end{figure}
%

The \buildts\ procedure is shown in Figure~\ref{code:build-timestep}.
The first {\bf for} loop of \buildts\ proceeds over all outcomes of
(relaxed) actions in the given set $A$ that may occur at time $t$. For
each such probabilistic outcome we introduce a new chance proposition
weighted by the conditional likelihood of that outcome.\footnote{Of
course, in our implementation we have a special case treatment for
deterministic actions, using no chance nodes (rather than a single
``chance node'' with static weight $1$).}
Having that, we extend {\sl Imp} with binary implications from this
new chance proposition to the add list of the outcome. If we are
uncertain about the condition $\con(e)$ of the corresponding effect at
time $t$, that is, we have $\con(e) \in uP(t)$, then we also add implications from $\con(e)$ to the chance propositions created for the outcomes of $e$. Otherwise, if $\con(e)$
is known at time $t$, then there is no uncertainty about our ability
to make the effect $e$ to hold at time $t$. In this case, we do not ``ground'' the chance propositions created for the outcomes of $e$ into the implication graph, but simply extend 
the running formula $\Phi$ with clauses capturing the ``exactly one''
relationship between these chance propositions corresponding to the
alternative outcomes of $e$ at time $t$. This way, the probabilistic
uncertainty about the outcome of $e$ can be treated as if being a
property of the initial belief state $\initbelief$; This is the only
type of knowledge we add into the knowledge base formula $\Phi$ after
initializing it in \buildprpg\ to $\cinitial$.

%

\begin{table}[thb]
\begin{center}
\setlength\extrarowheight{1pt}
\begin{tabular}{|c|l|}
\hline
\emph{Notation} & \emph{Description}\\
\hline
$\mbox{\sl Imp}_{v\rightarrow u}$ &
The graph containing
exactly all the paths from node $v$ to node $u$ in $\mbox{\sl Imp}$.\\
\hline
$\mbox{\sl
Imp}_{\rightarrow u}$ & The subgraph of $\mbox{\sl Imp}$ formed
by node $u$ and all the ancestors of $u$ in $\mbox{\sl Imp}$.\\
\hline
$\leafs{\mbox{\sl Imp}'}$ &
The set of all zero in-degree nodes in the subgraph $\mbox{\sl
Imp}'$ of $\mbox{\sl Imp}$.\\
\hline
$\grapheffects{\mbox{\sl Imp}'}$ &
The set of time-stamped action effects 
responsible for the implication edges\\
&  of the subgraph
$\mbox{\sl Imp}'$ of $\mbox{\sl Imp}$.\\
\hline
\end{tabular}
\end{center}
\caption{\label{table:implicationgraphnotation}Overview of notations
around the implication graph.}
\end{table}

The second {\bf for} loop checks whether a 
proposition $p$, unknown at time $t$, becomes known at time $t+1$. This part of the
\buildts\ procedure is somewhat more involved;
Table~\ref{table:implicationgraphnotation} provides an overview of the
main notations used in the follows when discussing the various uses
of the implication graph $\mbox{\sl Imp}$.


First thing in the second {\bf for} loop of \buildts, a call to
\buildwil\ procedure associates each node $v(t')$ in $\mbox{\sl
Imp}_{\rightarrow p(t+1)}$ with an estimate
$\subweight_{p(t+1)}(v(t'))$ on the probability of achieving $p$ at
time $t+1$ by the effects $\grapheffects{ \mbox{\sl
Imp}_{v(t')\rightarrow p(t+1)} }$, given that $v$ holds at time
$t'$. In other words, the dynamic weight (according to $p(t+1)$) of
the implication graph nodes is computed. Note that $v(t')$ can be
either a time-stamped proposition $q(t')$ for some $q \in \fluents$,
or a chance proposition $\poutcome(t')$ for some probabilistic outcome
$\poutcome$.

We will discuss the \buildwil\ procedure in detail below. For
proceeding to understand the second {\bf for} loop of \buildts, the
main thing we need to know is the following lemma:

\begin{lemma}
\label{lemma:subweight-complete}
Given a node $v(t')\in\mbox{\sl Imp}_{\rightarrow p(t+1)}$, we have
$\subweight_{p(t+1)}\left(v(t')\right) = \weight\left(v(t')\right)$ if
and only if, given $v$ at time $t'$, the sequence of effects
$\grapheffects{\mbox{\sl Imp}_{v(t')\rightarrow p(t+1)}}$ achieves $p$
at $t+1$ with probability $1$.
\end{lemma}

In words, $v(t')$ leads to $p(t+1)$ with certainty iff the dynamic
weight of $v(t')$ equals its static weight. This is a simple
consequence of how the weight propagation is arranged; it should hold
true for any reasonable weight propagation scheme (``do not mark a node as certain if it is not''). A full proof of the lemma appears in
Appendix~\ref{ss:proofs} on pp.~\pageref{lemma:subweight-completeA}.

Re-consider the second {\bf for} loop of \buildts. What happens is the
following. Having finished the \buildwil\ weight propagation for $p$
at time $t+1$, we
\begin{enumerate}
\item collect all the leafs $\mbox{\sl support}(p(t+1))$ of $\mbox{\sl
Imp}_{\rightarrow p(t)}$ that meet the criteria of
Lemma~\ref{lemma:subweight-complete}, and
\item check (by a call to a SAT solver) whether the knowledge-base
formula $\Phi$
implies the disjunction of these leafs.
\end{enumerate}
If the implication holds, then the examined fact $p$ at time $t$ is
added to the set of facts known at time $t$. Finally, the procedure
removes from the set of facts that are known to possibly hold at time
$t+1$ all those facts that were just proven to hold at time $t+1$ with
certainty.

To understand the above, consider the following. With
Lemma~\ref{lemma:subweight-complete}, $\mbox{\sl support}(p(t+1))$
contains exactly the set of leafs achieving which will lead to
$p(t+1)$ with certainty. Hence we can basically use the same
implication test as in \cff. Note, however, that the word
``basically'' in the previous sentence hides a subtle but important
detail. In difference to the situation in \cff, $\mbox{\sl
support}(p(t+1))$ may contain two kinds of nodes: (1) proposition
nodes at the start layer of the PRPG, i.e., at layer $-m$
corresponding to the initial belief; (2) chance nodes at later layers
of the PRPG, corresponding to outcomes of effects that have no unknown
conditions. This is the point where the discussed above updates on the formula $\Phi$
are needed---those keep track of alternative effect outcomes. Hence
testing $\Phi \rightarrow \bigvee_{l \in \mbox{\scriptsize \sl
support}(p(t+1))} l$ is the same as testing whether either: (1) $p$ is
known at $t+1$ because it is always triggered with certainty by at
least one proposition true in the initial world; or (2) $p$ is known
at $t+1$ because it is triggered by all outcomes of an effect that
will appear with certainty. We get the following result:

\begin{lemma}
\label{lemma:disjunction}
Let $(\actions,\initBN,\goal,\goalprob)$ be a probabilistic planning
task, $\seqactions$ be a sequence of actions applicable in
$\initprob$, and $\onecondrel$ be a relaxation function for $A$. For
each time step $t \geq -m$, and each proposition $p \in \fluents$, if
$P(t)$ is constructed by \buildprpg$(\seqactions, \actions,
\cinitial,\goal,\goalprob,\onecondrel)$, then $p$ at time $t$ can be
achieved by a relaxed plan starting with $\seqactions\onecondrel$
\begin{enumerate}[(1)]
\item with probability $> 0$ (that is, $p$ is not negatively known at
time $t$) if and only if $p \in uP(t) \cup P(t)$, and
\item with probability $1$ (that is, $p$ is known at time $t$) if and
only if $p \in P(t)$.
\end{enumerate}
\end{lemma}


This is a consequence of the arguments outlined above. The full proof
of Lemma~\ref{lemma:disjunction} is given in Appendix~\ref{ss:proofs}
on pp.~\pageref{lemma:disjunctionA}.

\begin{figure}[t]
{\small
\begin{tabbing}
{\bf procedure} \= \underline{\sf build-w-impleafs}\ $(p(t),\mbox{\sl Imp})$\\
\> top-down propagation of weights $\subweight_{p(t)}$ from $p(t)$ to all nodes in $\mbox{\sl Imp}_{\rightarrow p(t)}$\\
$\subweight_{p(t)}\left(p(t)\right)$ := 1\\
{\bf for} \= decreasing time steps $t' := (t-1) \dots (-m)$ {\bf do}\\
\> {\bf for} \= all \= chance nodes $\newchancep{\poutcome(t')} \in 
			\mbox{\sl Imp}_{\rightarrow p(t)}$ {\bf do}\\
\> \> $\flbweight := \prod_{r\in\add(\poutcome),r(t'+1)\in \mbox{\scriptsize \sl Imp}_{\rightarrow p(t)}}{\left[ 1- \subweight_{p(t)}\left(r(t'+1)\right) \right]}$\\
\> \> $\subweight_{p(t)}\left(\newchancep{\poutcome(t')}\right) := 
			\weight\left(\newchancep{\poutcome(t')}\right) \cdot (1 - \flbweight)$\\
\> {\bf endfor}\\
\> {\bf for} \= all fact nodes $q(t') \in \mbox{\sl Imp}_{\rightarrow p(t)}$ {\bf do}\\
\> \>  $\flbweight := 1$\\
\> \>  {\bf for} \= all $a \in A(t'), e \in \effs(a), \con(e)=q$ {\bf do}\\
\> \> \> $\flbweight := \flbweight \cdot \left[1 - \sum_{\poutcome\in \poutcomeset(e),\newchancep{\poutcome(t')} \in \mbox{\scriptsize \sl Imp}_{\rightarrow p(t)}}{\subweight_{p(t)}\left(\newchancep{\poutcome(t')}\right)} \right]$\\
\> \> {\bf endfor}\\
\> \> $\subweight_{p(t)}\left(q(t')\right) := 1 - \flbweight$ \\
\> {\bf endfor}\\
{\bf endfor}
\end{tabbing}}
\caption{\label{code:build-wimpleafs}The {\sl build-w-impleafs} procedure for weight back-propagation over the implication graph.}
\end{figure}

Let us now consider the weight-propagating\footnote{The weight propagation scheme of the \buildwil\ procedure is similar in nature to this used in the heuristics module of the recent probabilistic temporal planner Prottle of~\citeA{little:etal:aaai05}.} procedure \buildwil\  depicted in Figure~\ref{code:build-wimpleafs}.
This procedure performs a layered, top-down weight propagation from a given node\footnote{Note that the ``$t$'' here will be
instantiated with $t+1$ when called from \buildts.} $p(t)\in \mbox{\sl Imp}$ down to
the leafs of $\mbox{\sl Imp}_{\rightarrow p(t)}$.  This order of
traversal ensures that each node of $\mbox{\sl Imp}_{\rightarrow
p(t)}$ is processed only after all its successors in $\mbox{\sl
Imp}_{\rightarrow p(t)}$.  
For the chance nodes $\newchancep{\poutcome(t')}$, 
the dynamic weight 
$\subweight_{p(t)}\left(\newchancep{\poutcome(t')}\right)$
 is set to 
\begin{enumerate}
\item the probability that the outcome $\poutcome$ takes place at time
$t'$ given that the corresponding action effect $e(\poutcome)$ does
take place at $t'$, times
\item an estimate of the probability of achieving $p$ at time $t$ by
the effects $\grapheffects{\mbox{\sl
Imp}_{\newchancep{\poutcome(t')}\rightarrow p(t)}}$.
\end{enumerate}
The first quantity is given by the ``global'', static weight
$\weight\left(\newchancep{\poutcome(t')}\right)$ assigned to
$\newchancep{\poutcome(t')}$ in the first {\bf for} loop of \buildts.
The second quantity is derived from the dynamic weights 
$\subweight_{p(t)}\left(r(t'+1)\right)$ for $r \in \add(\poutcome)$,
computed in the previous iteration of the outermost {\bf for} loop of
\buildwil.  Making a heuristic assumption that the effect sets
$\grapheffects{\mbox{\sl Imp}_{r(t'+1)\rightarrow p(t)}}$ for
different $r\in\add(\poutcome)$ are all pairwise independent,
$\flbweight$ is then set to 
the probability of failure to achieve $p$ at $t$ by the effects
$\grapheffects{\mbox{\sl Imp}_{\newchancep{\poutcome(t')}\rightarrow
p(t)}}$. This computation of $\flbweight$ for
$\newchancep{\poutcome(t')}$ is decomposed over the artifacts of
$\poutcome$, and this is where the weight propagation starts taking
place. 
For the fact nodes $q(t')$, 
the dynamic weight $\subweight_{p(t)}\left(q(t')\right)$
 is set to the probability that some action effect conditioned on $q$ at time $t'$  allows (possibly indirectly) achieving the desired fact $p$ at time $t$. Making again the heuristic assumption of independence between various such effects conditioned on $q$ at $t'$, computing $\subweight_{p(t)}\left(q(t')\right)$ is decomposed over the outcomes of these effects.

\begin{figure}[thb]
\begin{tabbing}
{\bf procedure} \= \underline{\getp}\ $(t, G)$\\
\> estimates the probability of achieving $G$ at time $p$.\\
{\bf if} $\goal \not \subseteq P(t) \cup uP(t)$ {\bf then} {\bf return} $0$ {\bf endif}\\
{\bf if} $\goal \subseteq P(t)$ {\bf then} {\bf return} $1$ {\bf endif}\\
{\bf for} \= $g \in G \setminus P(t)$ {\bf do}\\
\> for each $l \in \leafs{\mbox{\sl Imp}_{\rightarrow g(t)}}$, introduce a chance proposition $\langle l_{g} \rangle$ with weight $\subweight_{g(t)}\left(l\right)$\\
\> $\varphi_{g} := (\bigvee_{l \in \mbox{\scriptsize $\leafs{\mbox{\sl
    Imp}_{\rightarrow g(t)}}$} } l) 
\wedge\bigwedge_{l \in \mbox{\scriptsize $\leafs{\mbox{\sl Imp}_{\rightarrow g(t)}}$} \cap uP(-m)}{\left(\neg l \vee \langle l_{g} \rangle \right)}$\\
%
%
%
{\bf endfor}\\
{\bf return} ${\sf WMC}(\Phi \wedge \bigwedge_{g \in G \setminus P(t)} \varphi_{g})$
\end{tabbing}
\caption{\label{code:get-P}Estimating the goal likelihood at a given time step.}
\end{figure}

What remains to be explained of the \buildprpg\ procedure are the two
termination criteria of the {\bf while} loop constructing the planning
graph from the layer $0$ onwards.  The first test is made by a call to
the \getp\ procedure, and it checks whether the PRPG built to the time
layer $T$ contains a relaxed plan for
$(\actions,\initBN,\goal,\goalprob)$.
The \getp\ procedure is shown in Figure~\ref{code:get-P}.  First, if
one of the subgoals is negatively known at time $t$, then, from
Lemma~\ref{lemma:disjunction}, the overall probability of achieving
the goal is $0$. On the other extreme, if all the subgoals are known
at time $t$, then the probability of achieving the goal is $1$. The
correctness of the latter test is implied by
Lemma~\ref{lemma:disjunction} and non-interference of relaxed actions.
This leaves us with the main case in which we are uncertain about some
of the subgoals. This uncertainty is either due to dependence of these
subgoals on the actual initial world state, or due to achieving these
subgoals using probabilistic actions, or due to both. The uncertainty
about the initial state is fully captured by our weighted CNF formula 
$\ourf(\initBN) \subseteq \Phi$. Likewise, the outcomes' chance
propositions $\newchancep{\poutcome}(t')$ introduced into the
implication graph by the \buildts\ procedure are ``chained up'' in
$\mbox{\sl Imp}$ to the propositions on the add lists of these
outcomes, and ``chained down'' in $\mbox{\sl Imp}$ to the unknown
(relaxed) conditions of these outcomes, if any.  Therefore, if some
action outcome $\poutcome$ at time $t' < t$ is relevant to achieving a
subgoal $g \in G$ at time $t$, then the corresponding node
$\newchancep{\poutcome}(t')$ must appear in $\mbox{\sl
Imp}_{\rightarrow g(t)}$, and its weight will be back-propagated by
\buildwil$(g(t),\mbox{\sl Imp})$ down to the leafs of $\mbox{\sl
Imp}_{\rightarrow g(t)}$.  The \getp\ procedure then exploits these
back-propagated estimates by, again, taking a heuristic assumption of
independence between achieving different subgoals. Namely, the
probability of achieving the unknown sub-goals $G \setminus P(t)$ is
estimated by weighted model counting over the formula $\Phi$,
conjoined with probabilistic theories $\varphi_{g}$ of achieving each
unknown goal $g$ in isolation. To understand the formulas
$\varphi_{g}$, consider that, in order to make $g$ true at $t$, we
must achieve at least one of the leafs $l$ of $\mbox{\sl
Imp}_{\rightarrow g(t)}$; hence the left part of the
conjunction. On the other hand, if we make $l$ true, then this achieves
$g(t)$ only with (estimated) probability
$\subweight_{g(t)}\left(l\right)$; hence the right part of the
conjunction requires us to ``pay the price'' if we set $l$ to
true.\footnote{If we do not introduce the extra chance propositions
$\langle l_{g} \rangle$, and instead assign the weight
$\subweight_{g(t)}\left(l\right)$ to $l$ itself, then the outcome is
not correct: we have to ``pay'' also for setting $l$ to false.}

%
%
%

As was explained at the start of this section, the positive PRPG
termination test may fire even if the real goal probability is {\em
not} high enough. That is, \getp\ may return a value higher than the
real goal probability, due to the approximation (independence
assumption) done in the weight propagation. Of course, due to the same
approximation, it may also happen that \getp\ returns a value lower
than the real goal probability.




The second PRPPG termination test comes to check whether we have
reached a point in the construction of PRPG that allows us to conclude
that there is no relaxed plan for $(\actions,\initBN,\goal,\goalprob)$
that starts with the given action sequence $\seqactions$. This
termination criterion asks whether, from time step $t$ to time step
$t+1$, any potentially relevant changes have occurred. A potentially
relevant change would be if the goal-satisfaction probability estimate
\getp\ grows, or if the known and unknown propositions grow, of if the
support leafs of the latter propositions in $\mbox{\sl Imp}$ that
correspond to the initial belief state grow.\footnote{To understand
the latter, note that PRPG can always be added with more and more
replicas of probabilistic actions irrelevant to achieving the goals,
and having effects with {\em known} conditions.  While these action
effects (since they are irrelevant) will not influence our estimate of
goal-satisfaction probability, the chance propositions corresponding
to the outcomes of these effects may become the support leafs of some
unknown proposition $p$. In the latter case, the set of support leafs
$\support(p(t'))$ will infinitely grow with $t' \rightarrow \infty$,
while the projection of $\support(p(t'))$ on the initial belief state
(that is, $\support(p(t))\cap uP(t)$) is guaranteed to reach a fix
point.} If none occurs, then the same would hold in all future
iterations $t' > t$, implying that the required goal satisfaction
probability $\goalprob$ would never be reached. In other words, the
PRPG construction is complete.

\begin{theorem}
\label{t:complete}
Let $(\actions,\initBN,\goal,\goalprob)$ be a probabilistic planning
task, $\seqactions$ be a sequence of actions applicable in
$\initprob$, and $\onecondrel$ be a relaxation function for $A$. If
\buildprpg$(\seqactions, \actions, \cinitial,\goal,
\goalprob,\onecondrel)$ returns FALSE, then there is no relaxed plan
for $(\actions,\initbelief,\goal,\goalprob)$ that starts with
$\seqactions\onecondrel$.
\end{theorem}

Note that Theorem~\ref{t:complete} holds despite the approximation
done during weight propagation, making the assumption of probabilistic
independence. For Theorem~\ref{t:complete} to hold, the only
requirement on the weight propagation is this: {\em if the real weight
still grows, then the estimated weight still grows.} This requirement
is met under the independence assumption. It would not be met under
the assumption of co-occurence, propagating weights by maximization
operations, and thereby conservatively under-estimating the
weights. With that propagation, if the PRPG fails then we cannot
conclude that there is no plan for the respective belief. This is
another good argument (besides the bad quality heuristics we observed
empirically) against using the conservative estimation.

The full proof to Theorem~\ref{t:complete} is given in
Appendix~\ref{ss:proofs} on pp.~\pageref{t:completeA}. The theorem
finalizes our presentation and analysis of the process of constructing
probabilistic relaxed planning graphs. 

\subsection{Example: PRPG Construction}
\label{ss:ex-prpg}

To illustrate the construction of a PRPG by the algorithm in Figures~\ref{code:build-PRPG}-\ref{code:get-P}, let us consider a simplification of our running Examples~\ref{example1}-\ref{example2} in which
\begin{enumerate}[(i)]
\item only the actions $\{move\mbox{-}b\mbox{-}right, move\mbox{-}left\}$ constitute the action set $\actions$,
\item the goal is $G = \{r_{1},b_{2}\}$, and the required lower bound on the probability of success $\goalprob = 0.9$,
\item the initial belief state $\initbelief$ is given by the BN $\initBN$ as in Example~\ref{example2}, and
\item the belief state $\belief_{\plan}$ evaluated by the heuristic function corresponds to the actions sequence $\plan = \langle move\mbox{-}b\mbox{-}right \rangle$. 
\end{enumerate}
The effects/outcomes of the actions $\actions$ considered in the construction of PRPG are described in Table~\ref{table:example3}, where $e^{\sf mbr}$ is a re-notation of the effect $e$ in Table~\ref{table:example1}, the effect $e'$ in Table~\ref{table:example1} is effectively ignored due to the emptiness of its add effects.

\begin{table}[ht]
\begin{center}
  \setlength{\extrarowheight}{2pt}	
  \begin{tabular}{|c|c|c|c|c|c|c|}
  \hline
   $a$ & $\!\effs(a)\!$ & $\con(e)$ & $\con(e)\onecondrel$ & $\poutcomeset(e)$ & $\probe(\poutcome)$ & $\add(\poutcome)$ \\
  \hline
  \hline
  & & & & $\poutcome^{\sf mbr}_{1}$ & 0.7 & $\{r_{2},b_{2}\}$ \\
  \cline{5-7}
  $a^{\sf mbr}$ ($move\mbox{-}b\mbox{-}right$) & $e^{\sf mbr}$ & $\!\{r_{1},b_{1}\}\!$ & $\!\{r_{1}\}\!$ & $\poutcome^{\sf mbr}_{2}$ & 0.2 & $\{r_{2}\}$ \\
  \cline{5-7}
  & & & & $\poutcome^{\sf mbr}_{3}$ & 0.1 & $\emptyset$ \\
  \hline
  $a^{\sf ml}$ $(move\mbox{-}left)$ & $e^{\sf ml}$ & $\!\{r_{2}\}\!$ & $\!\{r_{2}\}\!$ & $\poutcome^{\sf ml}$ &  1.0 & $\!\{r_{1}\}\!$ \\
  \hline
  \hline
  $\noop_{r_{1}}$ & $e^{r_{1}}$ & $\!\{r_{1}\}\!$ & $\!\{r_{1}\}\!$ & $\poutcome^{r_{1}}$ &  1.0 & $\!\{r_{1}\}\!$ \\
  \hline
  $\noop_{r_{2}}$ & $e^{r_{2}}$ & $\!\{r_{2}\}\!$ & $\!\{r_{2}\}\!$ & $\poutcome^{r_{2}}$ &  1.0 & $\!\{r_{2}\}\!$ \\
  \hline
  $\noop_{b_{1}}$ & $e^{b_{1}}$ & $\!\{b_{1}\}\!$ & $\!\{b_{1}\}\!$ & $\poutcome^{b_{1}}$ &  1.0 & $\!\{b_{1}\}\!$ \\
  \hline
  $\noop_{b_{2}}$ & $e^{b_{2}}$ & $\!\{b_{2}\}\!$ & $\!\{b_{2}\}\!$ & $\poutcome^{b_{2}}$ &  1.0 & $\!\{b_{2}\}\!$ \\
  \hline
  \end{tabular}
\end{center}
\caption{\label{table:example3} Actions and their $\onecondrel$ relaxation for the PRPG construction example.} 
\end{table}

The initialization phase of the \buildprpg\ procedure results in $\Phi = \cinitial$, $\mbox{\sl Imp} := \emptyset$, $P(-1) = \emptyset$, and $uP(-1) = \{r_{1},r_{2},b_{1},b_{2}\}$. The content of $uP(-1)$ is  depicted in the first column of nodes in Figure~\ref{fig:impexample}.
The first {\bf for} loop of \buildprpg\ (constructing PRPG for the ``past'' layers corresponding to $\plan$) makes a single iteration, and calls the \buildts\ procedure with $t=-1$ and $A(\mbox{-}1) = \{a^{\sf mbr}\}\cup NOOPS$. (In what follows, using the names of the actions we refer to their $\onecondrel$ relaxations as given in Table~\ref{table:example3}.) The add list of the outcome $\poutcome^{\sf mbr}_{3}$ is empty, and thus it adds no nodes to the implication graph. Other than that, the chance nodes introduced to $\mbox{\sl Imp}$ by this call to \buildts\ 
appear in the second column of Figure~\ref{fig:impexample}. The first outer {\bf for} loop of \buildts\ results in $\mbox{\sl Imp}$ given by columns 1-3 of Figure~\ref{fig:impexample}, $uP(0) = uP(-1)$, and no extension of $\Phi$.

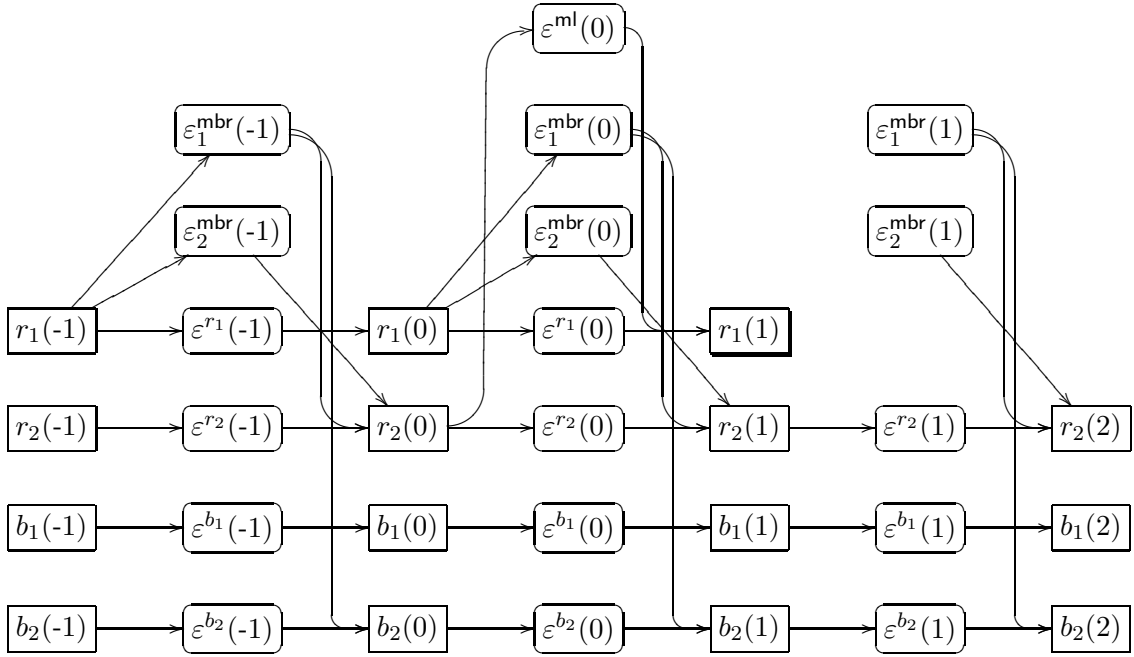
\begin{figure}[t]
\begin{center}
\begin{minipage}{\textwidth}
\def\objectsizestyle{\scriptstyle}
\xymatrix@R=20pt@C=30pt{
							 & & & *+[F-:<3pt>]{\poutcome^{\sf ml}(0)}  \ar `r/7pt[d] `[dddr] [dddr] \\
							 &  *+[F-:<3pt>]{\poutcome^{\sf mbr}_{1}(\mbox{-}1)} \ar `r/12pt[d] `[dddr] [dddr] \ar `r/18pt[d] `/5pt[dddddr] [dddddr] & & *+[F-:<3pt>]{\poutcome^{\sf mbr}_{1}(0)} \ar `r/12pt[d] `[dddr] [dddr] \ar `r/18pt[d] `/5pt[dddddr] [dddddr] & & *+[F-:<3pt>]{\poutcome^{\sf mbr}_{1}(1)} \ar `r/12pt[d] `[dddr] [dddr] \ar `r/18pt[d] `/5pt[dddddr] [dddddr]\\
							 &  *+[F-:<3pt>]{\poutcome^{\sf mbr}_{2}(\mbox{-}1)} \ar[ddr] & & *+[F-:<3pt>]{\poutcome^{\sf mbr}_{2}(0)} \ar[ddr] & & *+[F-:<3pt>]{\poutcome^{\sf mbr}_{2}(1)} \ar[ddr]\\
*+[F]{r_{1}(\mbox{-}1)} \ar[uur] \ar[ur] \ar[r] & *+[F-:<3pt>]{\poutcome^{r_{1}}(\mbox{-}1)} \ar[r] & *+[F]{r_{1}(0)}  \ar[uur] \ar[ur] \ar[r] & *+[F-:<3pt>]{\poutcome^{r_{1}}(0)}  \ar[r] & *+[F-,]{r_{1}(1)}\\
*+[F]{r_{2}(\mbox{-}1)} \ar[r] & *+[F-:<3pt>]{\poutcome^{r_{2}}(\mbox{-}1)} \ar[r] & *+[F]{r_{2}(0)} \ar `r/15pt[u] `[uuuur] [uuuur]  \ar[r] & *+[F-:<3pt>]{\poutcome^{r_{2}}(0)}  \ar[r] & *+[F]{r_{2}(1)} \ar[r] & *+[F-:<3pt>]{\poutcome^{r_{2}}(1)}  \ar[r] & *+[F]{r_{2}(2)}\\
*+[F]{b_{1}(\mbox{-}1)} \ar[r] & *+[F-:<3pt>]{\poutcome^{b_{1}}(\mbox{-}1)} \ar[r] & *+[F]{b_{1}(0)} \ar[r] &  *+[F-:<3pt>]{\poutcome^{b_{1}}(0)} \ar[r] & *+[F]{b_{1}(1)} \ar[r] & 
*+[F-:<3pt>]{\poutcome^{b_{1}}(1)} \ar[r] & *+[F]{b_{1}(2)}\\
*+[F]{b_{2}(\mbox{-}1)} \ar[r] & *+[F-:<3pt>]{\poutcome^{b_{2}}(\mbox{-}1)} \ar[r] & *+[F]{b_{2}(0)} \ar[r] & *+[F-:<3pt>]{\poutcome^{b_{2}}(0)} \ar[r] & *+[F]{b_{2}(1)} \ar[r] & *+[F-:<3pt>]{\poutcome^{b_{2}}(1)} \ar[r] & *+[F]{b_{2}(2)}
}
\end{minipage}
\end{center}
\caption{\label{fig:impexample} 
The implication graph $\mbox{\sl Imp}$. The odd columns of nodes depict the sets of unknown propositions $uP(t)$. The even columns of nodes depict the change propositions introduced for the probabilistic outcomes of the actions $A(t)$.}
\end{figure}

In the second outer {\bf for} loop of \buildts, the weight propagating  procedure \buildwil\ is called for each unknown fact $p(0) \in uP(0) = \{r_{1}(0),r_{2(0)}, b_{1}(0), b_{2}(0)\}$, generating the ``$p(0)$-oriented'' weights as in Table~\ref{table:example3:imp0}. For each $p(0) \in uP(0)$, the set of supporting leafs $\mbox{\sl support}(p(0)) = \{p(-1)\}$, none of them is implied by $\Phi = \initBN$, and thus the set of known facts $P(0)$ remains equal to $P(-1) = \emptyset$, and $uP(-1)$ equal to $= uP(-1)$.

\begin{table}[ht]
\begin{center}
\begin{tabular}{c|c|c|c|c|c|c|c|c|c|c|c|c|c|c|}
 \cline{2-15}
 & \multicolumn{4}{c|}{$t'=0$} & \multicolumn{10}{c|}{$t'=-1$}\\
 \cline{2-15}
 & $\!r_{1}\!$ & $\!r_{2}\!$ & $\!b_{1}\!$ & $\!b_{2}\!$ & $\!\poutcome^{\sf mbr}_{1}\!$ & $\!\poutcome^{\sf mbr}_{2}\!$ & $\!\poutcome^{r_{1}}\!$ & $\!\poutcome^{r_{2}}\!$ & $\!\poutcome^{b_{1}}\!$ & $\!\poutcome^{b_{2}}\!$ & $\!r_{1}\!$ & $\!r_{2}\!$ & $\!b_{1}\!$ & $\!b_{2}\!$\\
 \cline{2-15}
$\subweight_{r_{1}(0)}$ & $1$ & & & & & & 1 & & & & 1 & & & \\
\cline{2-15}
$\subweight_{r_{2}(0)}$ &  & 1 & & & 0.7 & 0.2 & & 1 & & & 0.9 & 1 & & \\
 \cline{2-15}
$\subweight_{b_{1}(0)}$ &  & & 1 & & & &  & & 1 & &  & & 1 & \\
\cline{2-15}
$\subweight_{b_{2}(0)}$ &  & &  & 1 & 0.7 &  &  & &  & 1 & 0.7 & &  & 1\\
\cline{2-15}
\end{tabular}
\end{center}
\caption{\label{table:example3:imp0} The columns in the table correspond to the nodes in the implication graph $\mbox{\sl Imp}$, and each row provides the weights $\subweight_{p(0)}$ for some $p(0) \in uP(0)$. An entry in the row of $p(0)$ is empty if and only if the node associated with the corresponding column does not belong to the implication subgraph $\mbox{\sl Imp}_{\rightarrow p(0)}$.}
\end{table}

Having finished with the {\bf for} loop, the \buildprpg\ procedure proceeds with the {\bf while} loop that builds the ``future'' layers of PRPG. The test of goal (un)satisficing \getp$(0, \goal) < \goalprob$ evaluates to TRUE as we get \getp$(0, \goal) = 0.63 < 0.9$, and thus the loop proceeds with its first iteration. To see 
the former, consider the implication graph $\mbox{\sl Imp}$ constructed so far (columns 1-3 in Figure~\ref{fig:impexample}).
For our goal $\goal = \{r_{1},b_{2}\}$ we have
$\leafs{\mbox{\sl Imp}_{\rightarrow r_{1}(0)}} = \{r_{1}(-1)\}$, and $\leafs{\mbox{\sl Imp}_{\rightarrow b_{2}(0)}} = \{r_{1}(-1),b_{2}(-1)\}$. As $\{r_{1}(0),b_{2}(0)\} \subset uP(0)$ and $\Phi=\cinitial$, we have 
\[
    {\mathsf{get\mbox{-}P}}(0,\goal) = {\sf WMC}\left(\cinitial \wedge \varphi_{r_{1}} \wedge \varphi_{b_{2}}\right),
\]
where
\begin{equation}
\label{e:stam}
\begin{split}
  \varphi_{r_{1}} &= \left( \langle r_{1,r_{1}}\rangle \right) \wedge \left(r_{1} \leftrightarrow \langle r_{1,r_{1}}\rangle \right), \\
  \varphi_{b_{2}} &= \left( \langle r_{1,b_{2}} \rangle \vee \langle b_{2,b_{2}} \rangle \right) \wedge 
  \left(r_{1}(-1) \leftrightarrow \langle r_{1,b_{2}}\rangle \right) \wedge \left(b_{2}(-1) \leftrightarrow \langle b_{2,b_{2}}\rangle \right),
 \end{split}
\end{equation}
and
\begin{equation}
\label{e:superstam}
\begin{split}
  \weight\left( \langle r_{1,r_{1}}\rangle \right) & =  \subweight_{r_{1}(0)}\left( r_{1}(-1) \right) = 1\\
  \weight\left( \langle b_{2,b_{2}}\rangle \right) & =  \subweight_{b_{2}(0)}\left( b_{2}(-1) \right) = 1 \hspace{1cm}.\\
  \weight\left( \langle r_{1,b_{2}}\rangle \right) & =  \subweight_{b_{2}(0)}\left( r_{1}(-1) \right) = 0.7
 \end{split}
\end{equation}
Observe that the two models of $\cinitial$ consistent with $r_{2}$ immediately falsify the sub-formula  $\cinitial \wedge\varphi_{r_{1}}$. Hence, we have 
\begin{eqnarray*}
   {\mathsf{get\mbox{-}P}}(0,\goal) & = &
   {\sf WMC}\left(\cinitial \wedge \varphi_{r_{1}} \wedge \varphi_{b_{2}}|_{r_{1}(-1)=1,b_{1}(-1)=1}\right) +\\
   && {\sf WMC}\left(\cinitial \wedge \varphi_{r_{1}} \wedge \varphi_{b_{2}}|_{r_{1}(-1)=1,b_{2}(-1)=1}\right)\\
    & = & \initbelief(r_{1},b_{1})\cdot \weight\left( \langle r_{1,r_{1}}\rangle \right) \cdot \weight\left( \langle r_{1,b_{2}}\rangle \right) + \initbelief(r_{1},b_{2})\cdot \weight\left( \langle r_{1,r_{1}}\rangle \right) \cdot \weight\left( \langle r_{1,b_{2}}\rangle \right)\cdot \weight\left( \langle b_{2,b_{2}}\rangle \right)\\
    &=& 0.63 \cdot 1 \cdot 0.7 + 0.27 \cdot 1 \cdot 0.7 \cdot 1\\
    &=& 0.63
\end{eqnarray*}
  
In the first iteration of the {\bf while} loop, \buildprpg\ calls the \buildts\ procedure with $t=0$ and $A(0) = \{a^{\sf mbr},a^{\sf ml}\} \cup NOOPS$.
The chance nodes introduced to $\mbox{\sl Imp}$ by this call to \buildts\ 
appear in the forth column of Figure~\ref{fig:impexample}. The first outer {\bf for} loop of \buildts\ results in $\mbox{\sl Imp}$ given by columns 1-5 of Figure~\ref{fig:impexample}, $uP(1) = uP(0)$, and no extension of $\Phi$. As before, in the second {\bf for} loop of \buildts, the \buildwil\ procedure is called for each unknown fact $p(1) \in uP(1) = \{r_{1}(1),r_{2(1)}, b_{1}(1), b_{2}(1)\}$, generating the ``$p(1)$-oriented'' weights. The interesting case here is the case of weight propagation \buildwil$(r_{1}(1),\mbox{\sl Imp})$, resulting in weights
\begin{center}
\begin{tabular}{lclcl}
\begin{minipage}{0.25\textwidth}
\begin{eqnarray*}
\subweight_{r_{1}(1)}(r_{1}(1)) &=& 1\\
\subweight_{r_{1}(1)}(\poutcome^{\sf ml}(0)) &=& 1\\
\subweight_{r_{1}(1)}(\poutcome^{r_{1}}(0)) &=& 1\\
\subweight_{r_{1}(1)}(r_{1}(0)) &=& 1\\
\subweight_{r_{1}(1)}(r_{2}(0)) &=& 1
\end{eqnarray*}
\end{minipage}
& $\Rightarrow$ &
\begin{minipage}{0.25\textwidth}
\begin{eqnarray*}
\subweight_{r_{1}(1)}(\poutcome^{r_{1}}(\mbox{-}1)) &=& 1\\
\subweight_{r_{1}(1)}(\poutcome^{r_{2}}(\mbox{-}1)) &=& 1\\
\subweight_{r_{1}(1)}(\poutcome^{\sf mbr}_{1}(\mbox{-}1)) &=& 0.7\\
\subweight_{r_{1}(1)}(\poutcome^{\sf mbr}_{2}(\mbox{-}1)) &=& 0.2
\end{eqnarray*}
\end{minipage}
& $\Rightarrow$ &
\begin{minipage}{0.25\textwidth}
\begin{eqnarray*}
\subweight_{r_{1}(1)}(r_{1}(\mbox{-}1)) &=& 1\\
\subweight_{r_{1}(1)}(r_{2}(\mbox{-}1)) &=& 1 
\end{eqnarray*}
\end{minipage}
\end{tabular}
\end{center}
for the nodes in $\mbox{\sl Imp}_{\rightarrow r_{1}(1)}$. From that, the set of supporting leafs of $r_{1}(1)$ is assigned to $\mbox{\sl support}(r_{1}(1)) = \{r_{1}(-1),r_{2}(-1)\}$, and since $\Phi = \cinitial$ does implies $r_{1}(-1)\vee r_{2}(-1)$, the fact $r_{1}$ is concluded to be known at time $1$, and is added to $P(1)$. For all other nodes $p(1) \in uP(1)$ we still have $\mbox{\sl support}(p(1)) = \{p(-1)\}$, and thus they all remain unknown at time $t=1$ as well. 
Putting things together, this call to the \buildwil\ procedure 
results with $P(1) = \{r_{1}(1)\}$, and $uP(1) = \{r_{2(1)}, b_{1}(1), b_{2}(1)\}$. The  {\bf while} loop of the \buildprpg\ procedure proceeds with checking the fixpoint termination test, and this immediately fails due to $P(1) \neq P(0)$. Hence, the {\bf while} loop proceeds with the next iteration corresponding to $t=1$. 

The test of goal (un)satisficing \getp$(1, \goal) < \goalprob$ still evaluates to TRUE because we have \getp$(1, \goal) = 0.899 < 0.9$. Let us follow this evaluation of \getp$(1, \goal)$ in detail as well.
Considering the implication graph $\mbox{\sl Imp}$ constructed so far up to time $t=1$ (columns 1-5 in Figure~\ref{fig:impexample}), and having $\goal \cap uP(1) = \{b_{2}(1)\}$, $\leafs{\mbox{\sl Imp}_{\rightarrow b_{2}(1)}} = \{r_{1}(-1),b_{2}(-1)\}$, and (still) $\Phi = \cinitial$, we obtain
\[
    {\mathsf{get\mbox{-}P}}(1,\goal) = {\sf WMC}\left(\cinitial \wedge  \varphi_{b_{2}}\right),
\]
with
\begin{equation}
\label{e:stam2}
\begin{split}
  \varphi_{b_{2}} &= \left( \langle r_{1,b_{2}} \rangle \vee \langle b_{2,b_{2}} \rangle \right) \wedge 
  \left(r_{1}(-1) \leftrightarrow \langle r_{1,b_{2}}\rangle \right) \wedge \left(b_{2}(-1) \leftrightarrow \langle b_{2,b_{2}}\rangle \right),
 \end{split}
\end{equation}
While the structure of $\varphi_{b_{2}}$ in Equation~\ref{e:stam2} is identical to this in Equation~\ref{e:stam}, the weights associated with the auxiliary chance propositions are different, notably
\begin{equation}
\label{e:superstam2}
\begin{split}
  \weight\left( \langle b_{2,b_{2}}\rangle \right) & =  \subweight_{b_{2}(1)}\left( b_{2}(-1) \right) = 1 \hspace{1cm}.\\
  \weight\left( \langle r_{1,b_{2}}\rangle \right) & =  \subweight_{b_{2}(1)}\left( r_{1}(-1) \right) = 0.91
 \end{split}
\end{equation}
The difference in $\weight\left( \langle r_{1,b_{2}}\rangle \right)$ between Equation~\ref{e:superstam} and Equation~\ref{e:superstam2} stems from
the fact that $r_{1}(-1)$ supports $b_{2}(1)$ not only via the effect 
$e^{\sf mbr}$ at time $-1$ but also via the a different instance of the same effect at time $0$. 
Now, the only model of $\cinitial$ that falsify $\varphi_{b_{2}}$ is the one that sets both $r_{1}$ and $b_{2}$ to false. Hence, we have
\begin{eqnarray*}
   {\mathsf{get\mbox{-}P}}(1,\goal) & = & \initbelief(r_{1},b_{1})\cdot \weight\left( \langle r_{1,b_{2}}\rangle \right) +\\
   & & \initbelief(r_{1},b_{2})\cdot \weight\left( \langle r_{1,b_{2}}\rangle \right)\cdot \weight\left( \langle b_{2,b_{2}}\rangle \right) +\\
   & & \initbelief(r_{2},b_{2})\cdot \weight\left( \langle b_{2,b_{2}}\rangle \right)\\
   & = & 0.63 \cdot 0.91 + 0.27 \cdot 0.91 \cdot 1 + 0.08 \cdot 1\\
   & = & 0.899 
\end{eqnarray*}

Having verified \getp$(1, \goal) < \goalprob$, the {\bf while} loop proceeds with the construction for time $t=2$, and calls the \buildts\ procedure with $t=1$ and $A(1) = \{a^{\sf mbr},a^{\sf ml}\} \cup NOOPS$. The chance nodes introduced to $\mbox{\sl Imp}$ by this call to \buildts\ appear in the sixth column of Figure~\ref{fig:impexample}.
The first outer {\bf for} loop of \buildts\ results in $\mbox{\sl Imp}$ given by columns 1-7 of Figure~\ref{fig:impexample}, and

\begin{equation}
\label{e:updatedPhi}
\begin{split}
   \Phi = \cinitial & \wedge \left( \poutcome^{\sf mbr}_{1}(1) \vee \poutcome^{\sf mbr}_{2}(1) \vee \poutcome^{\sf mbr}_{3}(1) \right) \wedge\\
   & \wedge
   \left( \neg\poutcome^{\sf mbr}_{1}(1) \vee \neg\poutcome^{\sf mbr}_{2}(1)\right)  \wedge  \left( \neg\poutcome^{\sf mbr}_{1}(1) \vee \neg\poutcome^{\sf mbr}_{3}(1) \right) \wedge \left( \neg\poutcome^{\sf mbr}_{2}(0) \vee \neg\poutcome^{\sf mbr}_{3}(0) \right)
   \end{split}
\end{equation}

\noindent Next, the \buildwil\ procedure is called as usual for each unknown fact $p(2) \in uP(2) = \{r_{2(2)}, b_{1}(2), b_{2}(2)\}$. The information worth detailing here is that now we have 
$\leafs{\mbox{\sl Imp}_{\rightarrow b_{2}(2)}} = \{b_{2}(-1),r_{1}(-1),\poutcome^{\sf mbr}_{1}(1)\}$, and
$\support(b_{2}(2)) = \{b_{2}(-1),\poutcome^{\sf mbr}_{1}(1)\}$. However, we still have $\Phi \rightarrow \bigvee_{l\in\support(p(2))}{l}$ for no $p(2) \in uP(2)$, and thus the set of known facts $P(2)$ remains equal to $P(1) = \{r_{1}\}$.

Returning from the call to the \buildwil\ procedure, \buildprpg\ proceeds with checking the fixpoint termination condition. This time, the first three equalities of the condition do hold, yet the condition is not satisfied due to \getp$(2, \goal) >$ \getp$(t, \goal)$. To see the latter, notice that we have 
\[
{\mathsf{get\mbox{-}P}}(2,\goal) = {\sf WMC}\left(\Phi \wedge  \varphi_{b_{2}}\right),
\]
where $\Phi$ is given by Equation~\ref{e:updatedPhi}, 
\begin{equation}
\label{e:stam3}
\begin{split}
  \varphi_{b_{2}} &= \left( \langle r_{1,b_{2}} \rangle \vee \langle b_{2,b_{2}} \rangle \vee \poutcome^{\sf mbr}_{1}(1) \right) \wedge 
  \left(r_{1}(-1) \leftrightarrow \langle r_{1,b_{2}}\rangle \right) \wedge \left(b_{2}(-1) \leftrightarrow \langle b_{2,b_{2}}\rangle \right),
 \end{split}
\end{equation}
and 
\begin{equation}
\label{e:superstam3}
\begin{split}
  \weight\left( \langle b_{2,b_{2}}\rangle \right) & =  \subweight_{b_{2}(1)}\left( b_{2}(-1) \right) = 1 \hspace{1cm}.\\
  \weight\left( \langle r_{1,b_{2}}\rangle \right) & =  \subweight_{b_{2}(1)}\left( r_{1}(-1) \right) = 0.91\\
  \weight( \poutcome^{\sf mbr}_{1}(1)) & = \subweight_{b_{2}(1)}( \poutcome^{\sf mbr}_{1}(1)) = 0.7
 \end{split}
\end{equation}
It is not hard to verify that 
\begin{eqnarray*}
   {\mathsf{get\mbox{-}P}}(2,\goal) & = & {\mathsf{get\mbox{-}P}}(1,\goal) + \initbelief(r_{2},b_{1})\cdot \weight(\poutcome^{\sf mbr}_{1}(1))\\
   &=& 0.899 + 0.02 \cdot 0.7\\
   &=& 0.913
\end{eqnarray*}
Note that now we do have ${\mathsf{get\mbox{-}P}}(2,\goal) \geq \goalprob$, and therefore \buildprpg\ aborts the {\bf while} loop by passing the goal satisficing test, and sets $T=2$. This finalizes the construction of PRPG, and thus, our example.

\subsection{Extracting a Probabilistic Relaxed Plan}
\label{ss:heuristic2}

If the construction of the PRPG succeeds in reaching the goals with
the estimated probability of success
${\mathsf{get\mbox{-}P}}(T,\goal)$ exceeding $\goalprob$, then we
extract a relaxed plan consisting of $\actions' \subseteq
\actions(0),\dots,\actions(T-1)$, and use the size of $\actions'$ as
the heuristic value of the evaluated belief state $\belief_{\plan}$.

Before we get into the technical details, consider that there are some
key differences between relaxed (no delete lists) probabilistic
planning on the one hand, and both relaxed classical and relaxed
qualitative conformant planning on the other hand. In relaxed
probabilistic planning, it might make sense to execute the same action
numerous times in consecutive time steps. In fact, this might be
essential -- just think of throwing a dice in a game until a ``6''
appears. In contrast, in the relaxed classical and qualitatively
uncertain settings this is not needed -- once an effect has been
executed, it remains true forever. Another complication in
probabilistic planning is that the required goal-achievement
probability is specified over a conjunction (or, possibly, some more
complicated logical combination) of different facts. While increasing
the probability of achieving each individual sub-goal $g \in \goal$ in
relaxed planning will always increase the overall probability of
achieving $\goal$, choosing the right distribution of effort among the
sub-goals to pass the required threshold $\goalprob$ for the whole
goal $\goal$ is a non-trivial problem.

A fundamental problem is the aforementioned lack of guarantees of the
weight propagation. On the one hand, the construction of PRPG and
Lemma~\ref{lemma:disjunction} imply that $\plan\onecondrel$
concatenated with an arbitrary linearization $\relplanbig$ of
$\actions(0),\dots,\actions(T-1)$ is executable in $\initbelief$. On
the other hand, due to the independence assumption made in the
\buildwil\ procedure, ${\mathsf{get\mbox{-}P}}(T,\goal) \geq
\goalprob$ does {\em not} imply that the probability of achieving
$\goal$ by $\plan\onecondrel$ concatenated with $\relplanbig$ exceeds
$\goalprob$. A ``real'' relaxed plan, in that sense, might not even
exist in the constructed PRPG.

Our answer to the above difficulties is to extract relaxed plans that
are correct {\em relative to the weight propagation.} Namely, we use
an implication graph ``reduction'' algorithm that computes a minimal
subset of that graph which still -- according to the weight
propagation -- sufficiently supports the goal. The relaxed plan then
corresponds to that subset. Obviously, this ``solves'' the difficulty
with the lack of ``real'' relaxed plans; we just do the relaxed plan
extraction according to the independence assumption (besides ignoring
deletes and removing all but one condition of each effect). The
mechanism also naturally takes care of the need to apply the same
action several times: this corresponds to several implication graph
edges which are all needed in order to obtain sufficient weight. The
choice of how effort is distributed among sub-goals is circumvented in
the sense that all sub-goals are considered in conjunction, that is, the
reduction is performed once and for all. Of course, there remains a
choice in which parts of the implication graph should be removed. We
have found that it is a useful heuristic to make this choice based on
which actions have already been applied on the path to the belief. We
will detail this below.

Making another assumption on top of the previous relaxations can of
course be bad for heuristic quality. The ``relaxed plans'' we extract
are not guaranteed to actually achieve the desired goal
probability. Since the relaxed plans are used only for search
guidance, per se this theoretical weakness is only of marginal
importance. However, an over-estimation of goal probability might
result in a bad heuristic because the relaxed plan does not include
the right actions, or does not apply them often enough. In
Section~\ref{ss:results}, we will discuss an example domain where
\pff\ fails to scale for precisely this reason.

\begin{sloppypar}
Figure~\ref{code:extract-PRPlan} shows the main routine \extractprp\
for extracting a relaxed plan from a given PRPG (note that $T$ is the
index of the highest PRPG layer, c.f.\
Figure~\ref{code:build-PRPG}). The sub-routines of \extractprp\ are
shown in
Figures~\ref{code:extract-subroutines1}-\ref{code:extract-subroutines2}. At
a high level, the \extractprp\ procedure consists of two parts:
\begin{enumerate}[1.]
 \item 
 {\em Reduction} of the implication graph, aiming at identifying a set
 of time-stamped action effects that can be ignored without decreasing
 our estimate of goal-achievement probability \getp$(T,\goal)$ below
 the desired threshold $\goalprob$, and
 \item 
 {\em Extraction} of a valid relaxed plan $\relplan$ such that
 (schematically) constructing PRPG with $\relplan$ instead of the full
 set of $\actions(0),\dots,\actions(T)$ would still result in
 \getp$(T,\goal) \geq \goalprob$.
\end{enumerate}
\end{sloppypar}

\begin{figure}[thb]
{\small
\begin{tabbing}
{\bf procedure} \= \underline{\sf extract-PRPlan}$(PRPG(\plan, \actions, \cinitial,
\goal, \goalprob, \onecondrel))$,\\
\> selects actions from $A(0), \dots, A(T-1)$\\ 
$\mbox{\sl Imp}'$ := ${\mathsf{reduce\mbox{-}implication\mbox{-}graph}}()$\\
\iterativesg$(\mbox{\sl Imp}')$\\
${\mathsf{sub\mbox{-}goal}}(\goal \cap P(T))$\\
{\bf for} \= decreasing time steps $t$ := $T, \dots, 1$ {\bf do}\\
\> {\bf for} \= all $g \in \goal(t)$ {\bf do}\\
\> \> {\bf if} \= $\exists a \in A(t-1), e \in \effs(a), \con(e) \in
P(t-1), \forall \poutcome\in\poutcomeset(e):g \in \add(\poutcome)$ {\bf then}\\
\> \> \> {\bf add-to-relaxed-plan} \= one such $a$ at time $t$\\
\> \> \> ${\mathsf{sub\mbox{-}goal}}(\pre(a) \cup \con(e))$\\
\> \> {\bf else }\\
\> \> \> ${\mbox{\sl Imp}^{\;g(t)}}$ := \constructsupport$(\support(g(t)))$\\
\> \> \> \iterativesg$({\mbox{\sl Imp}^{\;g(t)}})$\\
\> \> {\bf endif}\\
\> {\bf endfor}\\
{\bf endfor}
\end{tabbing}}
\caption{\label{code:extract-PRPlan}Extracting a probabilistic relaxed plan.}
\end{figure}

The first part is accomplished by the \reduceimp\ procedure, depicted
in Figure~\ref{code:extract-subroutines1}. As of the first step in the
algorithm, the procedure considers only the parts of the implication
graph that are relevant to achieving the unknown sub-goals. Next,
\reduceimp\ performs a greedy iterative elimination of actions from
the ``future'' layers $0,\dots,T-1$ of PRPG until the probability
estimate \getp$(T,\goal)$ over the reduced set of actions goes below
$\goalprob$. While, in principle, any action from $A(0),\dots,A(T-1)$
can be considered for elimination, in \reduceimp\ we examine only {\em
repetitions of the actions that already appear in
$\plan$}. Specifically, \reduceimp\ iterates over the actions $a$ in
$\plan\onecondrel$, and if $a$ repeats somewhere in the ``future''
layers of PRPG, then one such repetition $a(t')$ is considered for
removal. If removing this repetition of $a$ is found safe with respect
to achieving $\goalprob$,\footnote{Note here that the formula for {\sf
WMC} is constructed exactly as for the \getp\ function, c.f.\
Figure~\ref{code:get-P}.} then it is effectively removed by
eliminating all the edges in $\mbox{\sl Imp}$ that are induced by
$a(t')$. Then the procedure considers the next repetition of $a$. If
removing another copy of $a$ is not safe anymore, then the procedure
breaks the inner loop and considers the next action.

%

%

\begin{figure}[thb]
{\small
\begin{tabbing}
{\bf procedure} \= \underline{\sf reduce-implication-graph}$()$\\
\> operates on the PRPG;\\
\> returns a sub-graph of $\mbox{\sl Imp}$.\\
$\mbox{\sl Imp}' := \cup_{g \in G\setminus P(T)} \mbox{\sl Imp}_{\rightarrow g(T)}$\\
{\bf for} \= all actions $a \in \plan\onecondrel$ {\bf do}\\
\> {\bf for} \= all edges $(\poutcome(t'),p(t'+1)) \in \mbox{\sl Imp}''$, induced by $a(t') \in A(t')$, for some $t' \geq 0$ {\bf do}\\
\> \> $\mbox{\sl Imp}''$ \= $:= \mbox{\sl Imp}'$\\
\> \>  remove from $\mbox{\sl Imp}''$ all the edges induced by $a \in A(t')$\\ 
\> \> {\bf for} \= all $g \in G \setminus P(t)$ {\bf do}\\
\> \> \> for each $l \in \leafs{{\mbox{\sl Imp}''}_{\rightarrow g(T)}}$, introduce a chance proposition $\langle l_{g} \rangle$ with weight $\subweight_{g(T)}\left(l\right)$\\
\> \> \> $\varphi_{g} := (\bigvee_{l \in \mbox{\scriptsize $\leafs{{\mbox{\sl
    Imp}''}_{\rightarrow g(T)}}$} } l) 
\wedge\bigwedge_{l \in \mbox{\scriptsize $\leafs{{\mbox{\sl Imp}''}_{\rightarrow g(T)}}$} \cap uP(-m)}{\left(\neg l \vee \langle l_{g} \rangle \right)}$\\
%
%
\> \> {\bf endfor}\\
\> \> {\bf if} \= ${\sf WMC}(\Phi \wedge \bigwedge_{g \in G \setminus P(T)} \varphi_{g}) \geq \goalprob$ {\bf then} $\mbox{\sl Imp}' := \mbox{\sl Imp}''$ {\bf else break endif}\\
\> {\bf endfor}\\
{\bf endfor}\\
{\bf return} $\mbox{\sl Imp}'$
\end{tabbing}}
\caption{\label{code:extract-subroutines1}The procedure reducing the implication graph.}
\end{figure}

To illustrate the intuition behind our focus on the repetitions of the
actions from $\plan$, let us consider the following example of a
simple logistics-style planning problem with probabilistic actions.
Suppose we have two locations $A$ and $B$, a truck that is known to be
initially in $A$, and a heavy and uneasy to grab package that is known
to be initially on the truck. The goal is to have the package unloaded
in $B$ with a reasonably high probability, and there are two actions
we can use -- moving the truck from $A$ to $B$ ($a^{\sf m}$), and
unloading the package ($a^{\sf u}$). Moving the truck does not
necessarily move the truck to $B$, but it does that with an extremely
high probability. On the other hand, unloading the bothersome  package
succeeds with an extremely low probability, leaving the package on the
truck otherwise.  Given this data, consider the belief state
$\belief_{\plan}$ corresponding to ``after trying to move the truck
once'', that is, to the action sequence $\langle a^{\sf m}
\rangle$. To achieve the desired probability of success, the PRPG will
have to be expanded to a very large time horizon $T$, allowing the
action $a^{\sf u}$ to be applied sufficiently many times. However, the
fact ``truck in $B$'' is not known in the belief state
$\belief_{\plan}$, and thus the implication graph will also contain
the same amount of applications of $a^{\sf m}$. Trimming away most of
these applications of $a^{\sf m}$ will still keep the probability
sufficiently high.

The reader might ask at this point what we hope to achieve by
``trimming away most of the applications of $a^{\sf m}$''. The point
is, intuitively, that the implication graph reduction mechanism is a
means to {\em understand what has been accomplished already, on the
path to $\belief_{\plan}$.} Without such an understanding, the relaxed
planning can be quite indiscriminative between search states. Consider
the above example, and assume we have not one but two troubled
packages, $P1$ and $P2$, on the truck, with unload actions $a^{\sf
u1}$ and $a^{\sf u2}$. The PRPG for $\belief_{\plan}$ contains copies
of $a^{\sf u1}$ and $a^{\sf u2}$ at layers up to the large horizon
$T$. Now, say our search starts to unload $P1$. In the resulting
belief, the PRPG still has $T$ steps because the situation has not
changed for $P2$. Each step of the PRPG still contains copies of both
$a^{\sf u1}$ and $a^{\sf u2}$ -- and hence the heuristic value remains
the same as before! In other words, without an implication graph
reduction technique, relevant things that are accomplished may remain
hidden behind other things that have not yet been accomplished. In the
above example, this is not really critical because, as soon as we have
tried an unload for {\em each} of $P1$ and $P2$, the time horizon $T$
decreases by one step, and the heuristic value is reduced. It is,
however, often the case that some sub-task must be accomplished before
some other sub-task can be attacked. In such situations, without
implication graph reduction, the search staggers across a huge plateau
until the first task is completed. We observed this in a variety of
benchmarks, and hence designed the implication graph reduction to make
the relaxed planning aware of what has already been done.

Of course, since our weight propagation may over-estimate true probabilities, and 
hence over-estimate what was achieved in the past, the implication
graph reduction may conclude prematurely that a sub-task has been
``completed''. This leads us to the main open question in this
research; we will get back to this at the end of
Section~\ref{ss:results}, where we discuss this in the context of an
example where \pff's performance is bad.

%

\begin{figure}[htb]
{\small
\begin{tabbing}
{\bf procedure} \= \underline{\sf extract-subplan}$(\mbox{\sl Imp}')$\\
\>  actions that are helpful for achieving uncertain goals $\goal \cap uP(T)$ and \\
\> subgoals all the essential conditions of these actions\\
{\bf for} \= 
each edge $(\poutcome(t), p(t+1)) \in {\mbox{\sl Imp}'}$ such that $t \geq 0$ {\bf do}\\
\> {\bf if} \= action $a$ and its effect $e \in \effs(a)$ be responsible for $\poutcome$ at time $t$ {\bf time}\\
\> \> {\bf add-to-relaxed-plan} $a$ at time $t$\\
\> \> ${\mathsf{sub\mbox{-}goal}}((\pre(a) \cup \con(e)) \cap P(t))$\\
\> {\bf endif}
{\bf endfor}\\
\ \\
{\bf procedure} \= \underline{\sf sub-goal}$(P)$\\
\> inserts the propositions in $P$ as sub-goals\\
\> at the layers of their first appearance in the PRPG\\
{\bf for} \= all $p \in P$ {\bf do}\\
\> $t_{0}$ := $\argmin_{t}{\{p \in P(t)\}}$\\
\> {\bf if} $t_0 \geq 1$ {\bf then} $\goal(t_0) := \goal(t_0) \cup \{ p \}$ {\bf endif}\\
{\bf endfor}\\
\\
{\bf procedure} \= \underline{\sf construct-support-graph}$(\support(g(t)))$\\
\> takes a subset $\support(g(t))$ of $\leafs{\mbox{\sl Imp}_{\rightarrow g(t)}}$  weighted according to $g(t)$;\\
\> returns a sub-graph $\mbox{\sl Imp}'$ of $\mbox{\sl Imp}$.\\
$\mbox{\sl Imp}'$ := $\emptyset$\\
$open$ := $\support(g(t))$\\
{\bf while} \= $open \neq \emptyset$ {\bf do}\\
\> $open$ := $open \setminus \{p(t')\}$\\
\> {\bf choose} \= $a \in A(t'), e \in \effs(a), \con(e)=\{p\}$ such that\\
\> \> $\forall \poutcome \in \poutcomeset(e): (p(t'), \poutcome(t')) \in \mbox{\sl Imp}_{g(t)} \wedge \weight_{g(t)}(\poutcome(t')) = \weight(\poutcome(t'))$\\
\> {\bf for} \= each $\poutcome \in \poutcomeset(e)$ {\bf do}\\
\> \> {\bf choose} $q \in \add(\poutcome)$ such that $\weight_{g(t)}(q(t'+1)) = 1$\\
\> \> $\mbox{\sl Imp}'$ := $\mbox{\sl Imp}' \cup \{(p(t'),\poutcome(t')),(\poutcome(t'),q(t'+1))\}$\\
\> \> $open$ := $open \cup \{q(t'+1)\}$\\
\> {\bf endfor}
{\bf endwhile}\\
{\bf return} $\mbox{\sl Imp}'$
\end{tabbing}}
\caption{\label{code:extract-subroutines2}Sub-routines for ${\mathsf{extract\mbox{-}PRPlan}}$.}
\end{figure}

Let us get back to explaining the \extractprp\ procedure. After the
implication graph reduction, the procedure proceeds with the relaxed
plan extraction. The process makes use of proposition sets
$G(1),\dots,G(T)$, which are used to store time-stamped sub-goals
arising at layers $1 \leq t \leq T$ during the relaxed plan
extraction. The sub-routine \iterativesg\
(Figure~\ref{code:extract-subroutines2})
\begin{enumerate}[1.]
\item adds to the constructed relaxed plan all the time-stamped
actions responsible for the edges of the reduced implication graph
${\mbox{\sl Imp}'}$, and
\item subgoals everything outside the implication graph that condition
the applicability of the effects responsible for the edges of
${\mbox{\sl Imp}'}$.
\end{enumerate}
Here and in the later phases of the process, the sub-goals are added
into the sets $G(1),\dots,G(T)$ by the \subgoal\ procedure that simply
inserts each given proposition as a sub-goal at the first layer of its
appearance in the PRPG. Having accomplished this extract-and-subgoal
pass of \iterativesg\ over ${\mbox{\sl Imp}'}$, we also subgoal all
the goal conjuncts known at time $T$.

In the next phase of the process, the sub-goals are considered layer
by layer in decreasing order of time steps $T \geq t \geq 1$. For each
sub-goal $g$ at time $t$, certain supporting actions are selected into
the relaxed plan. If there is an action $a$ and some effect $e \in
\effs(a)$ that are known to be applicable at time $t-1$, and guarantee
to achieve $g$ with certainty, then $a$ is added to the constructed
relaxed plan at $t-1$. Otherwise, we
\begin{enumerate}[1.]
\item use the \constructsupport\ procedure to extract a sub-graph
${\mbox{\sl Imp}^{\;g(t)}}$ consisting of a set of implications that
together ensure achieving $g$ at time $t$, and
\item use the already discussed procedure \iterativesg\ to
\begin{enumerate}
\item add to the constructed relaxed plan all the time-stamped actions
responsible for the edges of ${\mbox{\sl Imp}^{\;g(t)}}$, and
\item subgoal everything outside this implication graph ${\mbox{\sl
Imp}^{\;g(t)}}$ that condition the applicability of the effects
responsible for the edges of ${\mbox{\sl Imp}^{\;g(t)}}$.
\end{enumerate}
\end{enumerate}
Processing this way all the sub-goals down to $G(1)$ finalizes the
extraction of the relaxed plan estimate. Section~\ref{ss:ex-extract} provides a detailed illustration of this
process on the PRPG constructed in Section~\ref{ss:ex-prpg}. In any event, it is easy to verify that the relaxed plan we extract is sound relative to our weight propagation, in the following sense.

\begin{proposition}
Let $(\actions,\initBN,\goal,\goalprob)$ be a probabilistic planning
task, $\seqactions$ be a sequence of actions applicable in
$\initprob$, and $\onecondrel$ be a relaxation function for $A$ such
that \buildprpg$(\seqactions, \actions, \cinitial,\goal,
\goalprob,\onecondrel)$ returns TRUE. Let $A(0)^{s},\dots,A(T-1)^{s}$
be the actions selected from $A(0),\dots,A(T-1)$ by \extractprp. When
constructing a relaxed planning graph using only
$A(0)^{s},\dots,A(T-1)^{s}$, then \getp$(T, \goal) \geq \goalprob$.
\end{proposition}

\begin{proof}
By construction: ${\mathsf{reduce\mbox{-}implication\mbox{-}graph}}$
leaves enough edges in the graph so that the weight propagation
underlying \getp\ still concludes that the goal probability is high
enough.
\end{proof}

\commentout{

\begin{theorem}
\label{t:sound}
Let $(\actions,\initBN,\goal,\goalprob)$ be a probabilistic planning
task, $\seqactions$ be a sequence of actions applicable in
$\initprob$, and $\onecondrel$ be a relaxation function for $A$ such
that \buildprpg$(\seqactions, \actions, \cinitial,\goal,
\goalprob,\onecondrel)$ returns TRUE. Let $A(0)^{s},\dots,A(T-1)^{s}$
be the actions selected from by $A(0),\dots,A(T-1)$ \extractprp.  If,
for each pair $g,g' \in \goal$, and each pair $l,l' \in uP(-m), l\neq
l'$, we have $\mbox{\sl Imp}_{l(-m)\rightarrow g(T)}\cap \mbox{\sl
Imp}_{l'\rightarrow g'(T)} = \emptyset$,
then 
$\seqactions\onecondrel$ concatenated with an arbitrary linearization
of $A(0)^{s},$$\dots,A(T-1)^{s}$ is a plan for
$(\actions\onecondrel,\initBN,\goal,\goalprob)$.
\end{theorem}

}

\subsection{Example: Extracting a Relaxed Plan from PRPG}
\label{ss:ex-extract}

We illustrate the process of the relaxed plan extraction on the PRPG as in Figure~\ref{fig:impexample}, constructed for the belief state and problem specification as in example in Section~\ref{ss:ex-prpg}. In this example we have $T=2$, $\goal \cap uP(2) = \{ b_{2} \}$, and thus the implication graph $\mbox{\sl Imp}$ gets immediately reduced to its sub-graph $\mbox{\sl Imp}'$ depicted in Figure~\ref{fig:impexample4}a.  As the plan $\plan$ to the belief state in question consists of only a single action $a^{\sf mbr}$, the only action instances that are considered for elimination by the outer {\bf for} loop of \reduceimp\  are  $a^{\sf mbr}(0)$ and $a^{\sf mbr}(1)$. If $a^{\sf mbr}(0)$ is chosen to be examined, then the implication sub-graph $\mbox{\sl Imp}'' = \mbox{\sl Imp}'$ is further reduced by removing all the edges due to $a^{\sf mbr}(0)$, and the resulting $\mbox{\sl Imp}''$ appears\footnote{The dashed edges in Figure~\ref{fig:impexample4}b can be removed from $\mbox{\sl Imp}''$ either now or at a latter stage if $\mbox{\sl Imp}''$ is chosen to replace $\mbox{\sl Imp}'$.} in Figure~\ref{fig:impexample4}b. The $\Phi$ and $\varphi_{b_{2}}$ components of the evaluated formula $\Phi\wedge\varphi_{b_{2}}$ are given by Equation~\ref{e:updatedPhi} and Equation~\ref{e:stam3}, respectively, and the weights associated with the chance propositions in Equation~\ref{e:stam3} over the reduced implication graph $\mbox{\sl Imp}''$ are
\begin{equation}
\label{e:superstam4}
\begin{split}
  \weight\left( \langle b_{2,b_{2}}\rangle \right) & =  \subweight_{b_{2}(1)}\left( b_{2}(-1) \right) = 1 \\
  \weight\left( \langle r_{1,b_{2}}\rangle \right) & =  \subweight_{b_{2}(1)}\left( r_{1}(-1) \right) = 0.7 \hspace{1cm}.\\
  \weight( \poutcome^{\sf mbr}_{1}(1)) & = \subweight_{b_{2}(1)}( \poutcome^{\sf mbr}_{1}(1)) = 0.7
 \end{split}
\end{equation}
The weight model counting of $\Phi\wedge\varphi_{b_{2}}$ evaluates to $0.724 < \theta$, and thus $\mbox{\sl Imp}''$ does not replace $\mbox{\sl Imp}'$. The only alternative action removal is this of $a^{\sf mbr}(1)$, and it can be seen from the example in Section~\ref{ss:ex-prpg} that this attempt for action elimination will also result in probability estimate lower than $\theta$. Hence, the only effect of \reduceimp\ 
on the PRPG processed by the \extractprp\ procedure is the reduction of the implication graph to only the edges relevant to achieving $\{b_{2}\}$ at time $T=2$. The reduced implication sub-graph $\mbox{\sl Imp}'$ returned by the \reduceimp\ procedure is 
depicted in Figure~\ref{fig:impexample4}a.

\begin{figure}[t]
\begin{center}
\begin{tabular}{c}
\begin{minipage}{\textwidth}
\def\objectsizestyle{\scriptstyle}
\xymatrix@R=20pt@C=30pt{							 &  *+[F-:<3pt>]{\poutcome^{\sf mbr}_{1}(\mbox{-}1)} \ar `r/18pt[d] `/5pt[ddr] [ddr] & & *+[F-:<3pt>]{\poutcome^{\sf mbr}_{1}(0)} \ar `r/18pt[d] `/5pt[ddr] [ddr] & & *+[F-:<3pt>]{\poutcome^{\sf mbr}_{1}(1)}  \ar `r/18pt[d] `/5pt[ddr] [ddr]\\
*+[F]{r_{1}(\mbox{-}1)} \ar[ur]  \ar[r] & *+[F-:<3pt>]{\poutcome^{r_{1}}(\mbox{-}1)} \ar[r] & *+[F]{r_{1}(0)}  \ar[ur]  & & & &\\
*+[F]{b_{2}(\mbox{-}1)} \ar[r] & *+[F-:<3pt>]{\poutcome^{b_{2}}(\mbox{-}1)} \ar[r] & *+[F]{b_{2}(0)} \ar[r] & *+[F-:<3pt>]{\poutcome^{b_{2}}(0)} \ar[r] & *+[F]{b_{2}(1)} \ar[r] & *+[F-:<3pt>]{\poutcome^{b_{2}}(1)} \ar[r] & *+[F]{b_{2}(2)}
}
\end{minipage}\\
\ \\
(a)\\
\ \\
\begin{minipage}{\textwidth}
\def\objectsizestyle{\scriptstyle}
\xymatrix@R=20pt@C=30pt{							 &  *+[F-:<3pt>]{\poutcome^{\sf mbr}_{1}(\mbox{-}1)} \ar `r/18pt[d] `/5pt[ddr] [ddr] & &  & & *+[F-:<3pt>]{\poutcome^{\sf mbr}_{1}(1)}  \ar `r/18pt[d] `/5pt[ddr] [ddr]\\
*+[F]{r_{1}(\mbox{-}1)} \ar[ur]  \ar@{-->}[r] & *+[F-:<3pt>]{\poutcome^{r_{1}}(\mbox{-}1)} \ar@{-->}[r] & *+[F]{r_{1}(0)}   & & & &\\
*+[F]{b_{2}(\mbox{-}1)} \ar[r] & *+[F-:<3pt>]{\poutcome^{b_{2}}(\mbox{-}1)} \ar[r] & *+[F]{b_{2}(0)} \ar[r] & *+[F-:<3pt>]{\poutcome^{b_{2}}(0)} \ar[r] & *+[F]{b_{2}(1)} \ar[r] & *+[F-:<3pt>]{\poutcome^{b_{2}}(1)} \ar[r] & *+[F]{b_{2}(2)}
}
\end{minipage}\\
\ \\
(b)\\
\ \\
\begin{minipage}{\textwidth}
\def\objectsizestyle{\scriptstyle}
\xymatrix@R=20pt@C=30pt{
							 & & & *+[F-:<3pt>]{\poutcome^{\sf ml}(0)}  \ar `r/7pt[d] `[dr] [dr] \\
*+[F]{r_{1}(\mbox{-}1)} \ar[r] & *+[F-:<3pt>]{\poutcome^{r_{1}}(\mbox{-}1)} \ar[r] & *+[F]{r_{1}(0)}  \ar[r] & *+[F-:<3pt>]{\poutcome^{r_{1}}(0)}  \ar[r] & *+[F-,]{r_{1}(1)} & \; & \; & \; \\
*+[F]{r_{2}(\mbox{-}1)} \ar[r] & *+[F-:<3pt>]{\poutcome^{r_{2}}(\mbox{-}1)} \ar[r] & *+[F]{r_{2}(0)} \ar `r/15pt[u] `[uur] [uur] 
}
\end{minipage}\\
\ \\
(c)
\end{tabular}
\end{center}
\caption{\label{fig:impexample4} Illustrations for various steps of the relaxed plan extraction from the PRPG constructed in Section~\ref{ss:ex-prpg}, and, in particular, from the implication graph of the latter, depicted in Figure~\ref{fig:impexample}.}
\end{figure}
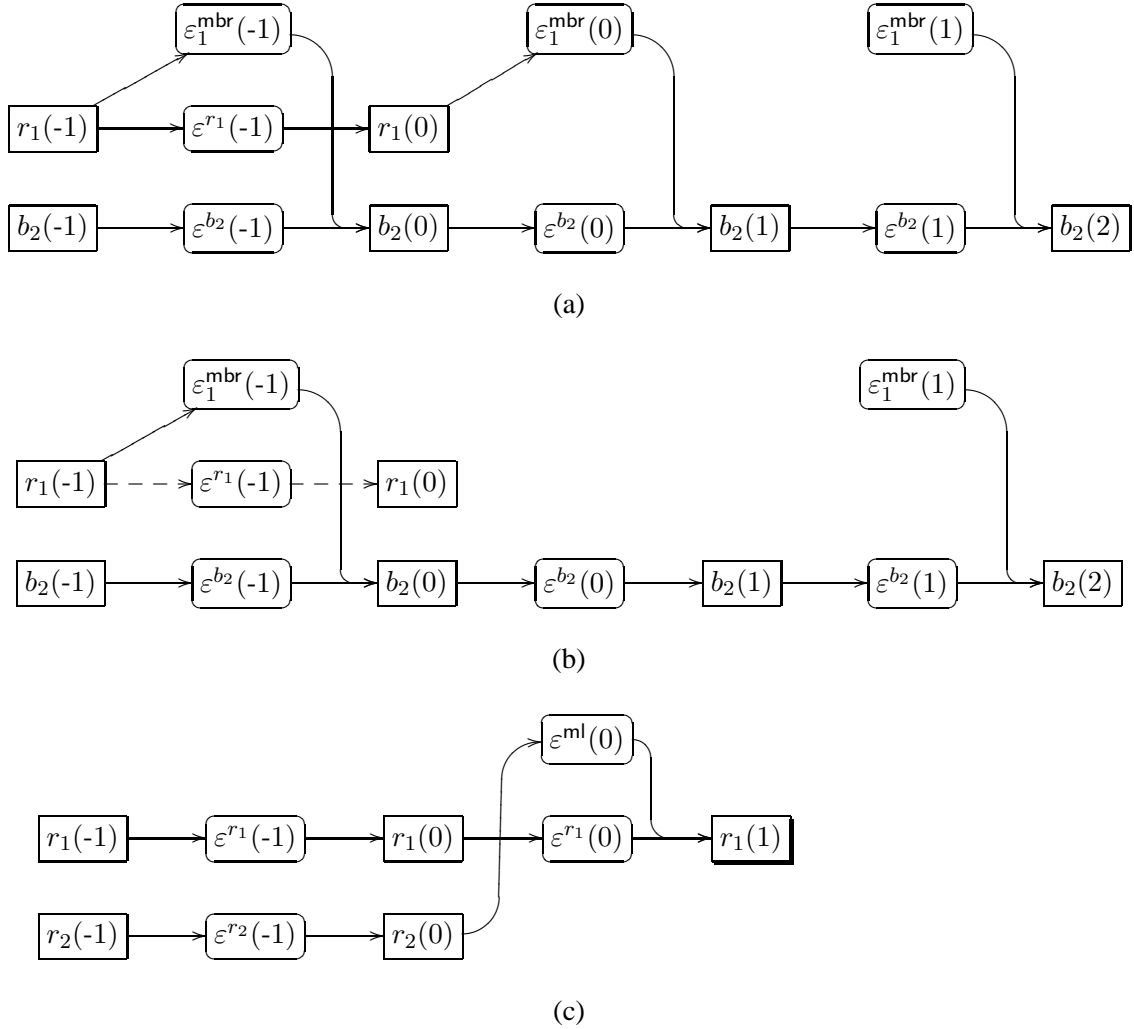

Next, the \iterativesg\ procedure iterates over the edges of $\mbox{\sl Imp}'$ and adds to the initially empty relaxed plan applications of $a^{\sf mbr}$ at times $0$ and $1$. The action $a^{\sf mbr}$ has no preconditions, and the condition $r_{1}$ of the effect $\poutcome^{\sf mbr}_{1} \in \effs(a^{\sf mbr})$ is known at time $1$. Hence, \iterativesg\ invokes the \subgoal\ procedure on $\{r(1)\}$, and the latter is added into the proposition set $G(1)$. The subsequent call \subgoal$(\goal\cap P(T))$ = \subgoal$(\{r_{1}\})$ leads to no further extensions of $G(2),G(1)$ as we already have $r_{1} \in G(1)$. Hence, the outer {\bf for}  loop of \extractprp\ starts with $G(2) = \emptyset$, and $G(1) = \{r_{1}\}$.

Since  $G(2)$ is empty, the first sub-goal considered by \extractprp is $r_{1}$ from $G(1)$. For $r_{1}$ at time $1$, no action effect at time $0$ passes the test of the {\bf if} statement---the condition $r_{2}$ of $\poutcome^{\sf ml}$ is not known at time $0$, and the same is true\footnote{In fact, it is easy to see from the construction of the \subgoal\ procedure that if $p$ belongs to $G(t)$, then the condition of the noop's effect $\poutcome^{p}$ cannot be known at time $t-1$.} for $\poutcome^{r_{1}}$. Hence, the subgoal $r_{1}(1)$ is processed by extracting a sub-plan to support achieving it with certainty. First, the \constructsupport\ procedure is called with $\support(r_{1}(1)) = \{r_{1}(-1),r_{2}(-1)\}$ (see Section~\ref{ss:ex-prpg}). The extracted sub-graph ${\mbox{\sl Imp}^{\;r_{1}(1)}}$ of the original implication graph ${\mbox{\sl Imp}}$ is depicted in Figure~\ref{fig:impexample4}c, and invoking the procedure \iterativesg\ on ${\mbox{\sl Imp}^{\;r_{1}(1)}}$ results in adding (i) application of $a^{\sf ml}$ at time $0$, and (ii) no new subgoals. Hence, the proposition sets $G(1),G(2)$ get emptied, and thus we end up with extracting a relaxed plan $\langle a^{\sf mbr}(0), a^{\sf ml}(0), a^{\sf mbr}(1) \rangle$.

\section{Empirical Evaluation}
\label{ss:results}

We have implemented \pff\ in C, starting from the \cff\ code. With
$\goalprob=1.0$, \pff\ behaves exactly like \cff\ (except that \cff\
cannot handle non-deterministic effects).  Otherwise, \pff\ behaves as
described in the previous sections, and uses Cachet
\cite{sang:etal:aaai-05} for the weighted model counting.
To better home in on strengths and weaknesses of our approach, the
empirical evaluation of \pff\ has been done in two steps. In
Section~\ref{ss:results1} we evaluate \pff\ on problems having
non-trivial uncertain initial states, but only deterministic actions.
In Section~\ref{ss:results2} we examine \pff\ on problems with
probabilistic action effects, and with both sources of uncertainty. We
compare \pff's performance to that of the probabilistic planner
\pond~\cite{pond:icaps-06}. The reasons for choosing \pond\ as the
reference point are twofold. First, similarly to \pff, \pond\
constitutes a forward-search planner guided by a non-admissible
heuristic function based on (relaxed) planning graph
computations. Second, to our knowledge, \pond\ clearly is the most
efficient probabilistic planner reported in the
literature.\footnote{In our experiments we have used a recent version
2.1 of \pond\ that significantly enhances \pond2.0~\cite{pond:icaps-06}. The authors would like to thank Dan Bryce and
Rao Kambhampati for providing us with a binary distribution of
\pond2.1.}

The experiments were run on a PC running at 3GHz with 2GB main memory
and 2MB cache running Linux. Unless stated otherwise, each
domain/problem pair was tried at four levels of desired probability of
success $\theta \in \{0.25, 0.5, 0.75, 1.0\}$. Each run of a planner
was time-limited by 1800 seconds of user time. \pff\ was run in the
default configuration inherited from \ff, performing one trial of
enforced hill-climbing and switching to best-first search in case of
failure. In domains without probabilistic effects, we found that
\pff's simpler relaxed plan extraction developed for that case~\cite{domshlak:hoffmann:icaps-06}, performs better than
the one described in here. We hence switch to the simpler version in
these domains.\footnote{Without probabilistic effects, relaxed plan
extraction proceeds very much like in \cff, with an additional
straightforward backchaining selecting support for the unknown
goals. The more complicated techniques developed in here to deal with
relaxed plan extraction under probabilistic effects appear to have a
more unstable behavior than the simpler techniques. If there {\em are}
probabilistic effects, then the simple backchaining is not meaningful
because it has no information on how many times an action must be
applied in order to sufficiently support the goal.} 

Unlike \pff, the heuristic computation in \pond\ has an element of
randomization; namely, the probability of goal achievement is
estimated via sending a set of random particles through the relaxed
planning graph (the number of particles is an input parameter). For
each problem instance, we averaged the runtime performance of \pond\
over 10 independent runs. In special cases where \pond\ timed out on
some runs for a certain problem instance, yet not on all of the 10
runs, the average we report for \pond\ uses the lower-bounding time
threshold of 1800s to replace the missing time points. In some cases,
\pond's best-case performance differs a lot from its average
performance; in these cases, the best-case performance is also
reported. We note that, following the suggestion of Dan Bryce, \pond\
was run in its default parameter setting, and, in particular, this
includes the number of random particles (64) selected for computing
\pond's heuristic estimate~\cite{pond:icaps-06}.

\subsection{Initial State Uncertainty and Deterministic Actions}
\label{ss:results1}

We now examine the performance of \pff\ and \pond\ in a collection of
domains with probabilistic initial states, but with deterministic
 action effects. We will consider the domains one by one, discussing for each
a set of runtime plots. For some of the problem instances,
Table~\ref{T:results} shows more details, providing features of the
instance size as well as detailed results for \pff, including the
number of explored search states and the plan length.

\begin{table*}[tb]
\begin{center}
{\scriptsize
\begin{tabular}{|l|l||c|c||c||c|}\hline
    &       &        $\goalprob=0.25$  & $\goalprob=0.5$   & $\goalprob=0.75$ & $\goalprob=1.0$ \\\hline
Instance           & \#actions/\#facts/\#states& $t$/$|S|$/$l$      & $t$/$|S|$/$l$      &      $t$/$|S|$/$l$  & $t$/$|S|$/$l$  \\ \hline\hline
Safe-uni-70 &70/71/140      & 1.39/19 /18       & 4.02/36/35        & 8.06/54/53  & 4.62/71 /70 \\ \hline     
Safe-cub-70 &70/70/138     &   0.28/6/5         & 0.76/13/12        & 1.54/22/21   & 4.32/70/69 \\ \hline \hline
Cube-uni-15 &6/90/3375       &    3.25/145/26     & 3.94/150/34       & 5.00/169/38   & 25.71/296/42 \\ \hline     
Cube-cub-15 &6/90/3375      &    0.56/41/8       & 1.16/70/13         &1.95/109/18    &26.35/365/42 \\ \hline    \hline    
Bomb-50-50  &2550/200/$> 2^{100}$     &    0.01/1/0        &0.10/17/16         & 0.25/37/36   & 0.14/51/50   \\ \hline     
Bomb-50-10  &510/120/$> 2^{60}$      &    0.00/1/0        & 0.89/248/22       & 4.04/778/62  & 1.74/911/90  \\ \hline   
Bomb-50-5   &255/110/$> 2^{55}$      &    0.00/1/0        & 1.70/468/27       & 4.80/998/67  &2.17/1131/95 \\ \hline  
Bomb-50-1   &51/102/$> 2^{51}$      &    0.00/1/0        &  2.12/662/31      & 6.19/1192/71 &2.58/1325/99 \\ \hline    \hline
Log-2    &3440/1040/$> 20^{10}$   &   0.90/117/54      & 1.07/152/62       & 1.69/205/69  & 1.84/295/78\\ \hline
Log-3    &3690/1260 /$> 30^{10}$  &    2.85/159/64     & 8.80/328/98      & 4.60/336/99  & 4.14/364/105\\ \hline
Log-4    &3960/1480/$> 40^{10}$   &    2.46/138/75     & 8.77/391/81      & 6.20/377/95  & 8.26/554/107 \\ \hline \hline  
Grid-2    &2040/825 /$> 36^{10}$   &  0.07/39/21        & 1.35/221/48       &6.11/1207/69 & 6.14/1207/69     \\ \hline     
Grid-3    &2040/841 /$> 36^{10}$   &  16.01/1629/76     & 15.8/1119/89      &82.24/3974/123  & 66.26/3974/123  \\ \hline     
Grid-4    &2040/857 /$> 36^{10}$   &  28.15/2167/96     & 51.58/2541/111     & 50.80/2541/115      &193.47/6341/155     \\ \hline  \hline
Rovers-7    &393/97 /$> 6^{3}*3^8$      &  0.01/ 37/18        &0.01/ 37/18         &0.01/ 37/18          &0.01/ 37/18      \\ \hline
RoversP-7    &393/133 /$> 6^{3}*3^8$    &  2.15/942/65       &2.23/983/75        &2.37/1008/83        &2.29/1008/83         \\ \hline
RoversPP-7   &393/133 /$> 6^{3}*3^8$    &  8.21/948/65       &12.48/989/75       &12.53/994/77        &16.20/1014/83        \\ \hline
RoversPPP-7  &395/140 /$> 6^{3}*3^8$    &  25.77/950/67      &41.18/996/79       & 0.01/UNSAT          & 0.01/UNSAT             
\\ \hline
\end{tabular}
}
\end{center}
\caption{\label{T:results}Empirical results for problems with
  probabilistic initial states. Times $t$ in seconds, search space
  size $|S|$ (number of calls to the heuristic function), plan length
  $l$.}
\end{table*}

Our first three domains are probabilistic versions of traditional
conformant benchmarks: ``Safe'', ``Cube'', and ``Bomb''. In Safe, out
of $n$ combinations one opens the safe. We are given a probability
distribution over which combination is the right one. The only type
of action in Safe is trying a combination, and the objective is to
open the safe with probability $\geq \goalprob$. We experimented with
two probability distributions over the $n$ combinations, a uniform one
(``Safe-uni'') and a distribution that declines according to a cubic
function (``Safe-cub'').  Table~\ref{T:results} shows that \pff\ can
solve this very efficiently even with $n=70$. Figure~\ref{graph:safe}
compares between \pff\ and \pond, plotting their time performance on
an identical linear scale, where $x$-axes show the number of
combinations.

\begin{figure}[tb]
\begin{center}
\begin{tabular}{cc}
  \includegraphics[width=9.0cm]{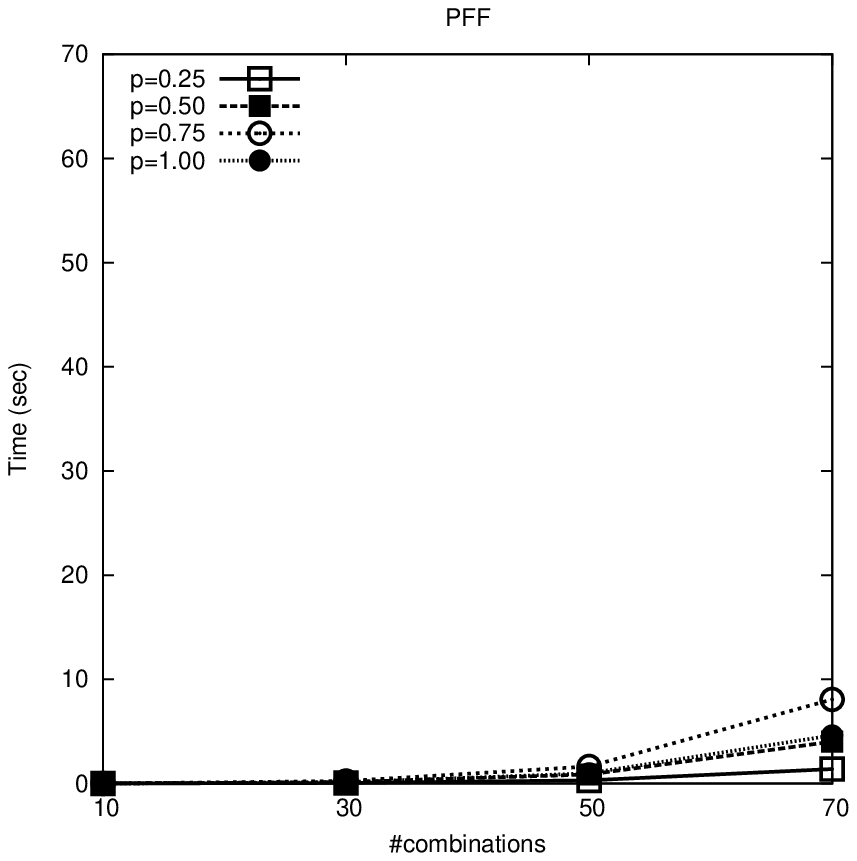} & \hspace{-1.5cm}
\includegraphics[width=9.0cm]{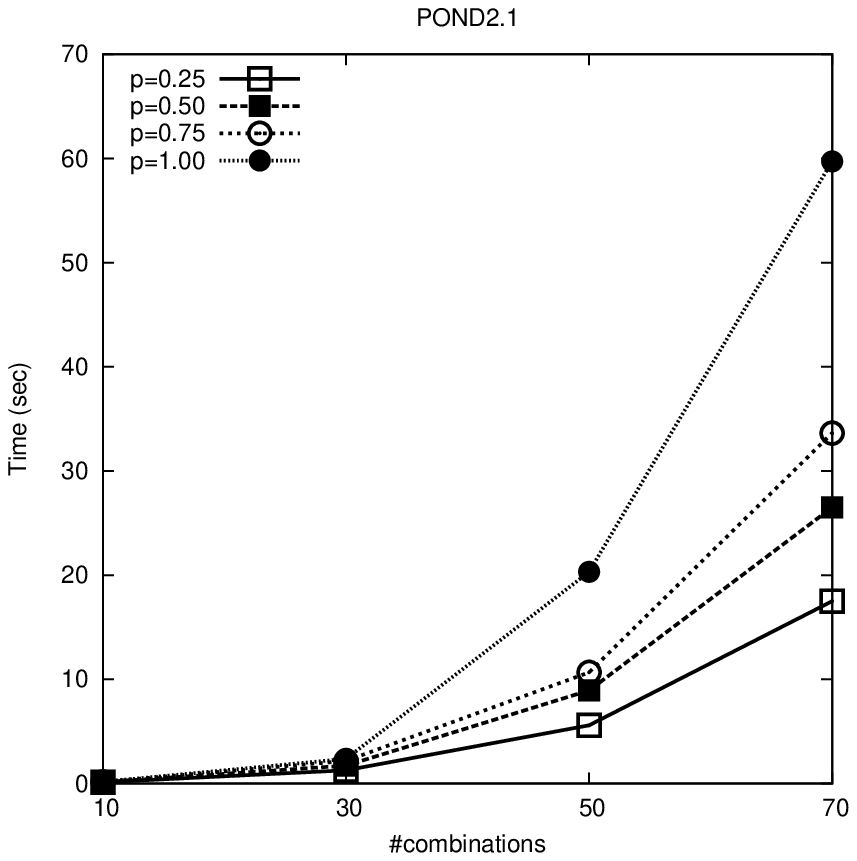}\\
\multicolumn{2}{c}{(a) Uniform prior distribution over the combinations.}\\
  \includegraphics[width=9.0cm]{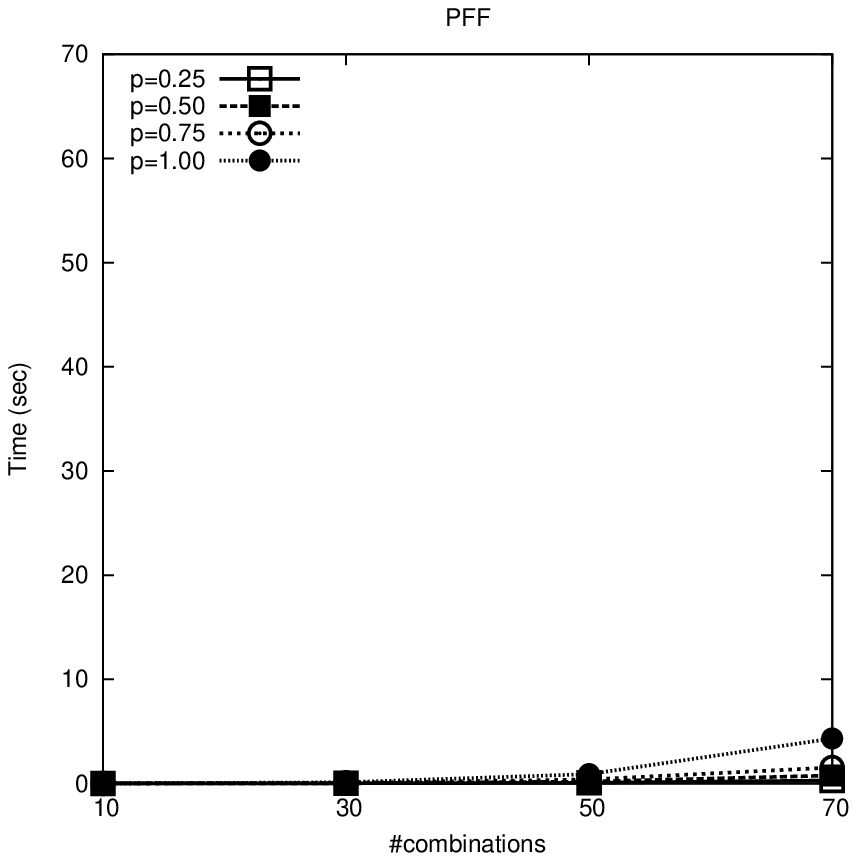} &\hspace{-1.5cm}
\includegraphics[width=9.0cm]{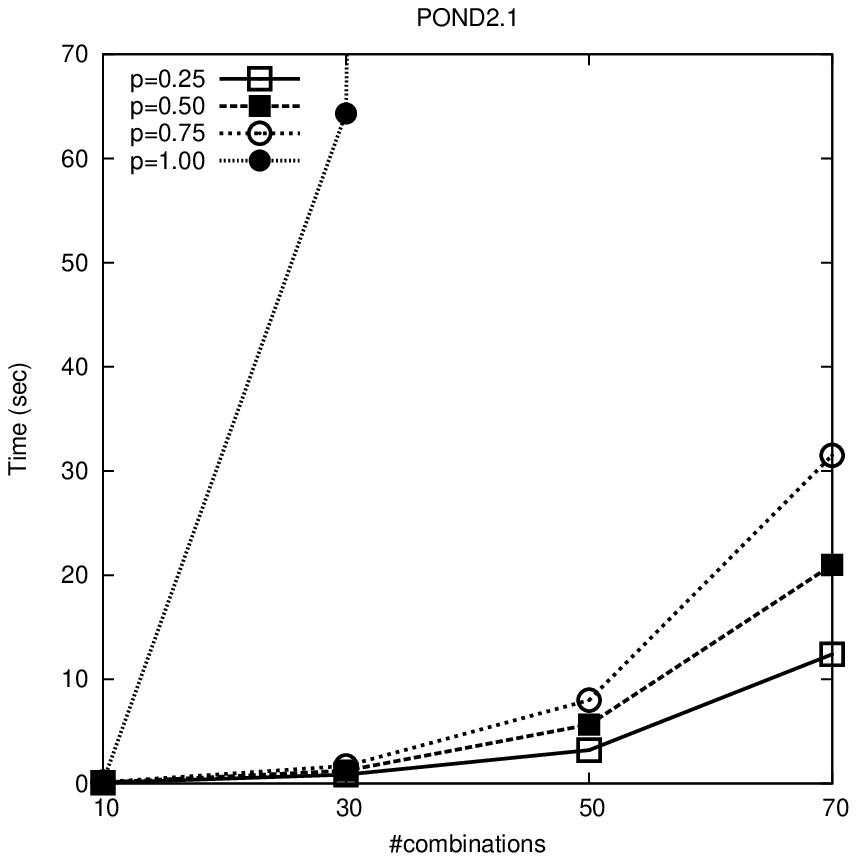}\\
\multicolumn{2}{c}{(b) Cubic decay prior distribution over the combinations.}
\end{tabular}
\end{center}
\vspace{-0.2cm}
\caption{\label{graph:safe} The Safe domain, \pff\ (left) vs.\ \pond\ (right).}
\end{figure}

From the graphs it is easy to see that \pff\ outperforms \pond\ by at
least an order of magnitude on both Safe-uni and Safe-cub. But a more
interesting observation here is not necessarily the difference in time
performance, but the relative performance of each planner on Safe-uni
and Safe-cub. Note that Safe-cub is somewhat ``easier'' than Safe-uni
in the sense that, in Safe-cub, fewer combinations must be tried to
guarantee a given probability $\goalprob$ of opening the safe.  This
because the dominant part of the probability mass lies on the
combinations at the head of the cubic distribution (the last
combination has probability $0$ to be the right combination, and thus
it needs not be tried even when $\goalprob=1.0$). The question is now
whether the heuristic functions of \pff\ and \pond\ exploit this
difference between Safe-uni and Safe-cub. Table~\ref{T:results} and
Figure~\ref{graph:safe} provide an affirmative answer for this
question for the heuristic function of \pff. The picture with \pond\
was less clear as the times spent by \pond\ on (otherwise identical)
instances of Safe-uni and Safe-cub were roughly the same.\footnote{On
Safe-cub with $n=70$ and $\theta \in \{0.75, 1.0\}$, \pond\ undergoes
an exponential blow-up that is not shown in the graphs since these
data points would obscure the other data points; anyway, we believe
that this blow-up is due only to some unfortunate troubles with
numerics.}

Another interesting observation is that, for both \pff\ and \pond,
moving from $\goalprob=1.0$ to $\goalprob<1.0$, that is, from planning
with qualitative uncertainty to truly probabilistic planning,
typically did not result in a performance decline. We even get {\em
improved} performance (except for $\goalprob=0.75$ in Safe-uni). The
reason seems to be that the plans become shorter. This trend can be
observed also in most other domains. The trend is particularly
remarkable for \pff, since moving from $\goalprob=1.0$ to
$\goalprob<1.0$ means to move from a case where no model counting is
needed to a case where it is needed. (In other words, \pff\
automatically ``specializes'' itself for the qualitative uncertainty,
by not using the model counting. To our knowledge, the same is not
true of \pond, which uses the same techniques in both cases.)

\begin{figure}[tb]
\begin{center}
\begin{tabular}{cc}
\includegraphics[width=9.0cm]{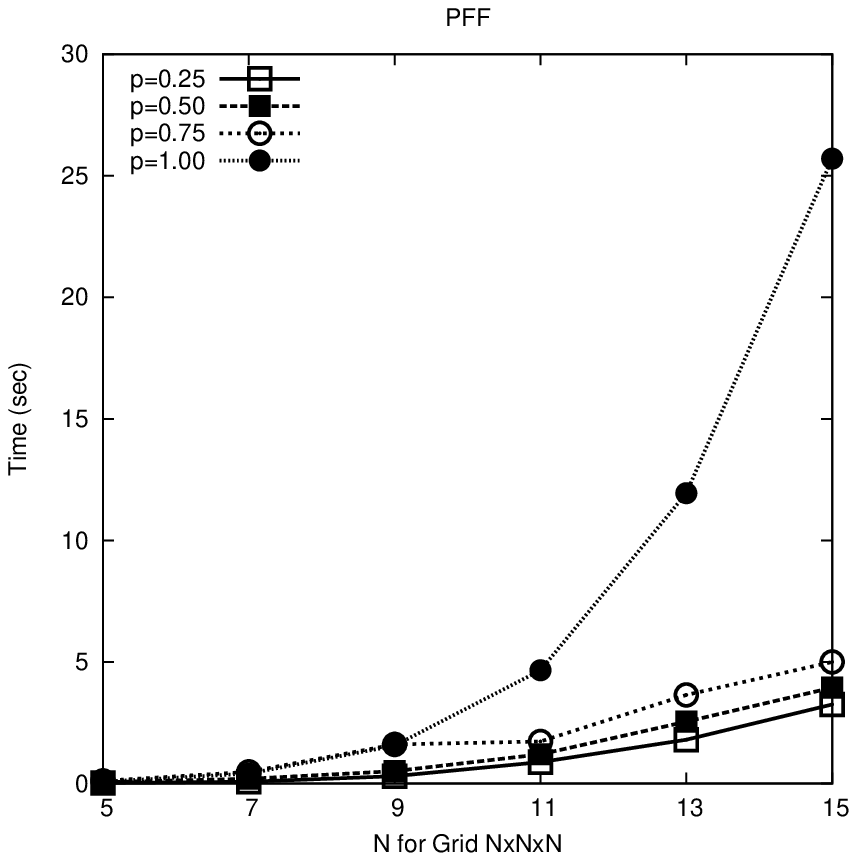} & \hspace{-1.5cm}
\includegraphics[width=9.0cm]{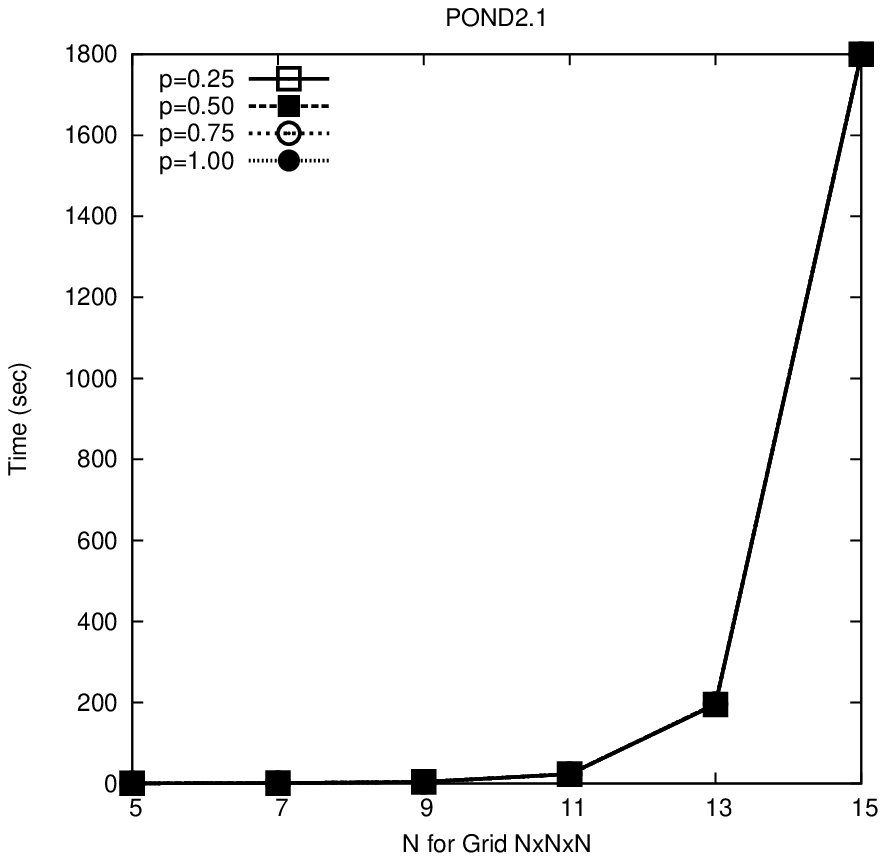}\\
\multicolumn{2}{c}{(a) Uniform prior distribution over the initial position.}\\
\includegraphics[width=9.0cm]{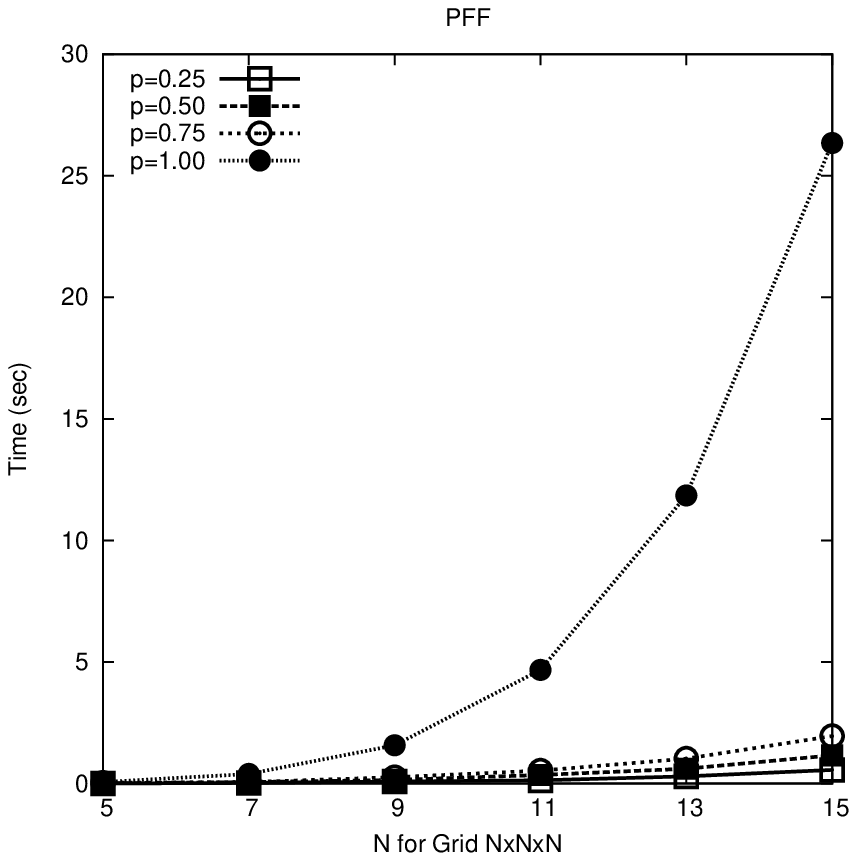} & \hspace{-1.5cm}
\includegraphics[width=9.0cm]{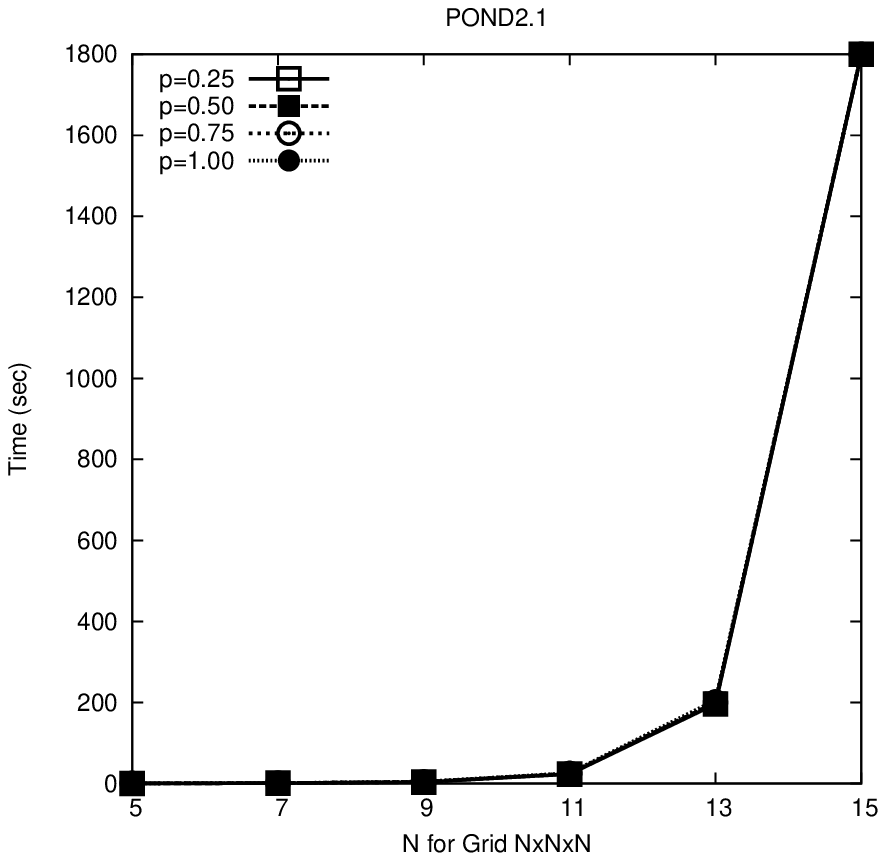}\\
\multicolumn{2}{c}{(b) Cubic decay prior distribution over the initial position.}
\end{tabular}
\end{center}
\vspace{-0.2cm}
\caption{\label{graph:cube} The Cube domain, \pff\ (left) vs.\ \pond\ (right).}
\end{figure}

In Cube, the task is to move into a corner of a $3$-dimensional grid,
and the actions correspond to moving from the current cube cell to one
of the (up to 6) adjacent cube cells. Again, we created problem
instances with uniform and cubic distributions (over the initial
position in each dimension), and again, \pff\ scales well, easily
solving instances on a $15\times 15\times 15$ cube.  Within our time
limit, \pond\ was capable of solving Cube problems with cube width
$\leq 13$. Figure~\ref{graph:cube} compares between \pff\ and \pond\
in more detail, plotting their time performance on {\em different}
linear scales (with $x$-axes capturing the width of the grid in each
dimension), and showing at least an order of magnitude advantage for
\pff. Note that,
\begin{itemize}
\item \pff\ generally becomes faster with decreasing $\goalprob$ (with
decreasing hardness of achieving the objective), while $\goalprob$
does not seem to have a substantial effect on the performance of
\pond,
\item \pff\ exploits the relative easiness of Cube-cub (e.g., see
Table~\ref{T:results}), while the time performance of \pond\ on
Cube-cub and Cube-uni is qualitatively identical.
\end{itemize}
We also tried a version of Cube where the task is to move into the
grid {\em center}. \pff\ is bad at doing so, reaching its performance
limit at $n=7$. This weakness in the Cube-center domain is inherited
from \cff. As detailed by \citeA{cff}, the reason for the weakness
lies in the inaccuracy of the heuristic function in this domain. There
are two sources of this inaccuracy. First, to solve
Cube-center in reality, one must start with moving into a corner in order to establish her position; in the relaxation, without delete lists, this is not necessary. Second, the relaxed planning graph computation over-approximates
not only what can be achieved in future steps, but also what has
already been achieved on the path to the considered belief state. For
even moderately long paths of actions, the relaxed planning graph
comes to the (wrong) conclusion that the goal has already been
achieved, so the relaxed plan becomes empty and there is no heuristic
information.

\begin{figure}[tb]
\begin{center}
\begin{tabular}{cc}
\includegraphics[width=9.0cm]{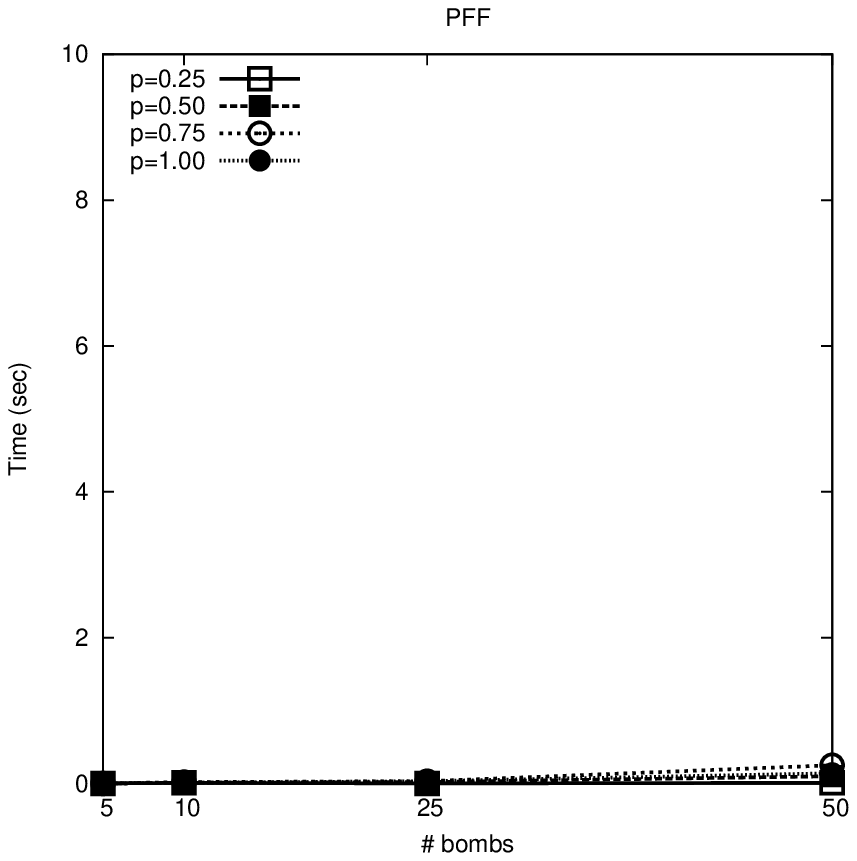} & \hspace{-1.5cm}
\includegraphics[width=9.0cm]{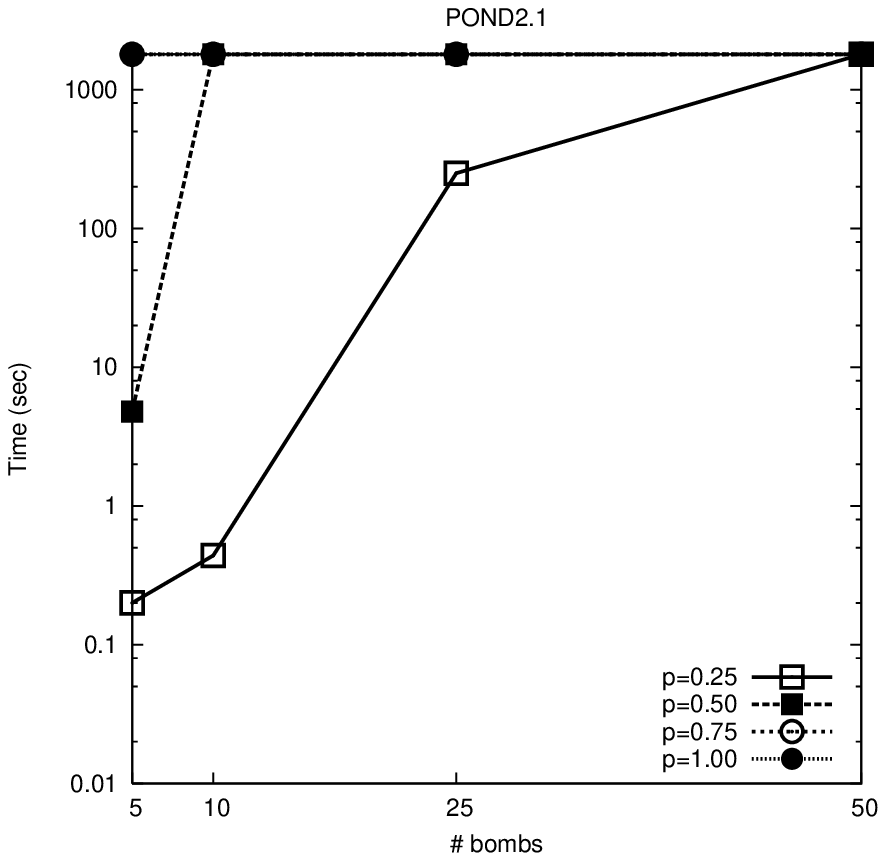}\\
\multicolumn{2}{c}{(a) 50 toilets}\\
\includegraphics[width=9.0cm]{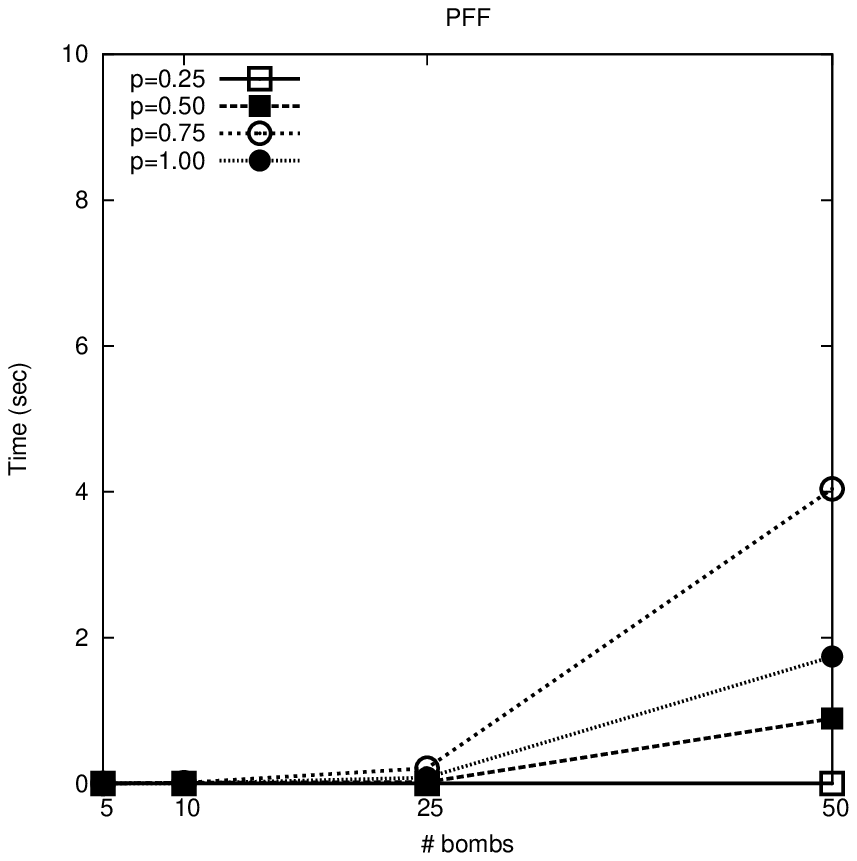} & \hspace{-1.5cm}
\includegraphics[width=9.0cm]{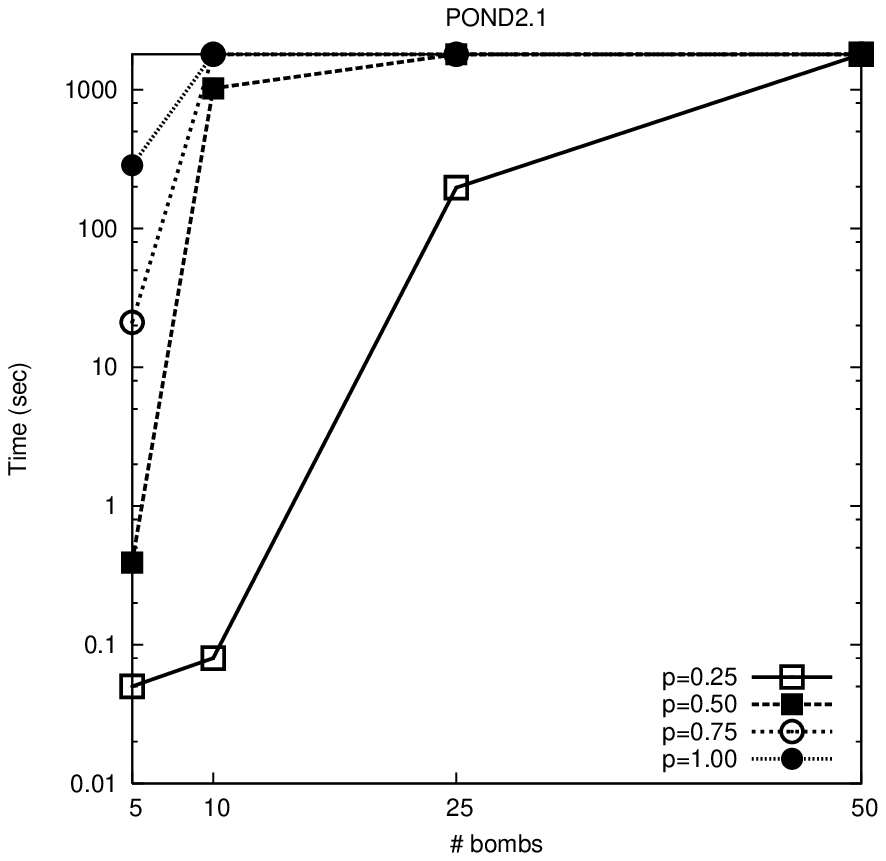}\\
\multicolumn{2}{c}{(b) 10 toilets}\\
\end{tabular}
\end{center}
\vspace{-0.2cm}
\caption{\label{graph:bomb}The Bomb domain, \pff\ (left) vs.\ \pond\ (right).}
\end{figure}

\begin{figure}[tb]
\begin{center}
\begin{tabular}{cc}
\includegraphics[width=9.0cm]{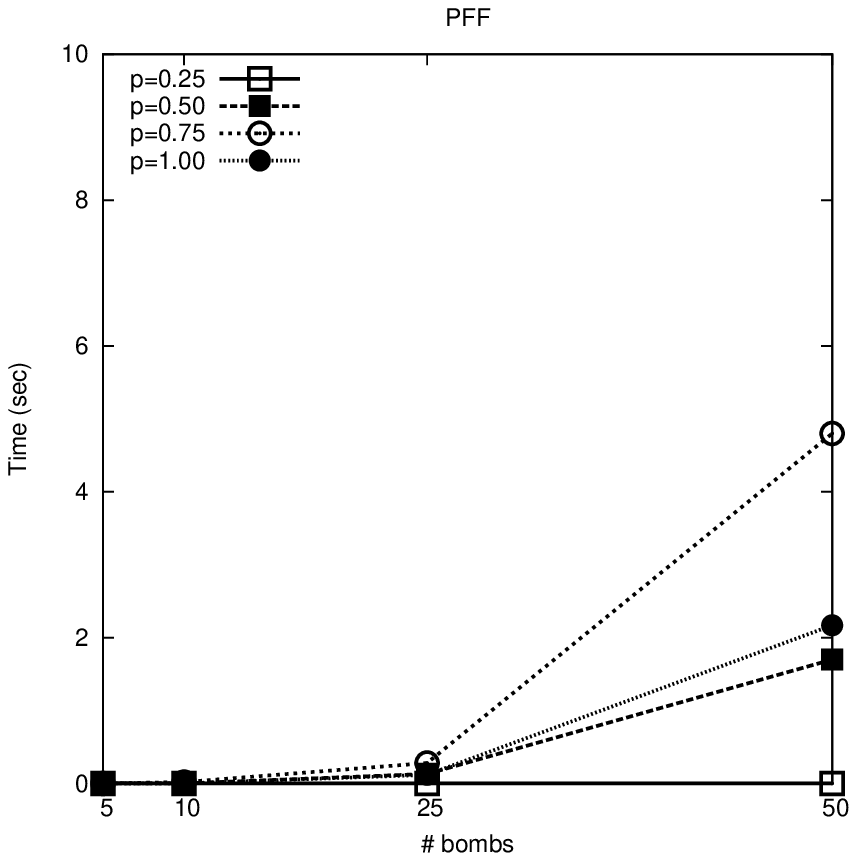} & \hspace{-1.5cm}
\includegraphics[width=9.0cm]{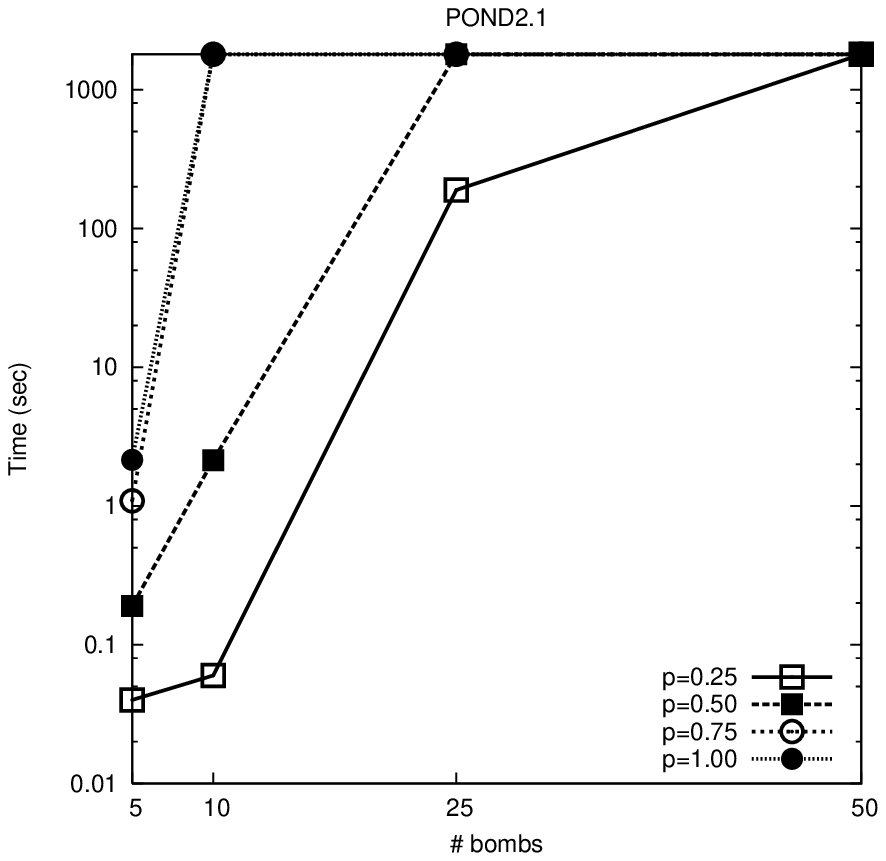}\\
\multicolumn{2}{c}{(c) 5 toilets}\\
\includegraphics[width=9.0cm]{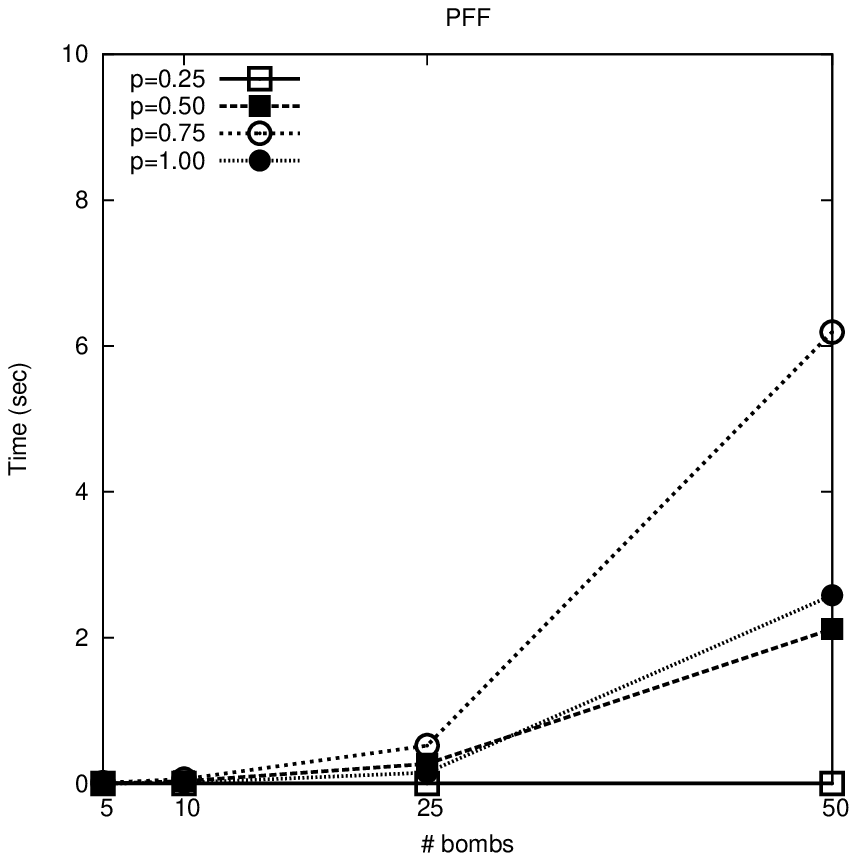} & \hspace{-1.5cm}
\includegraphics[width=9.0cm]{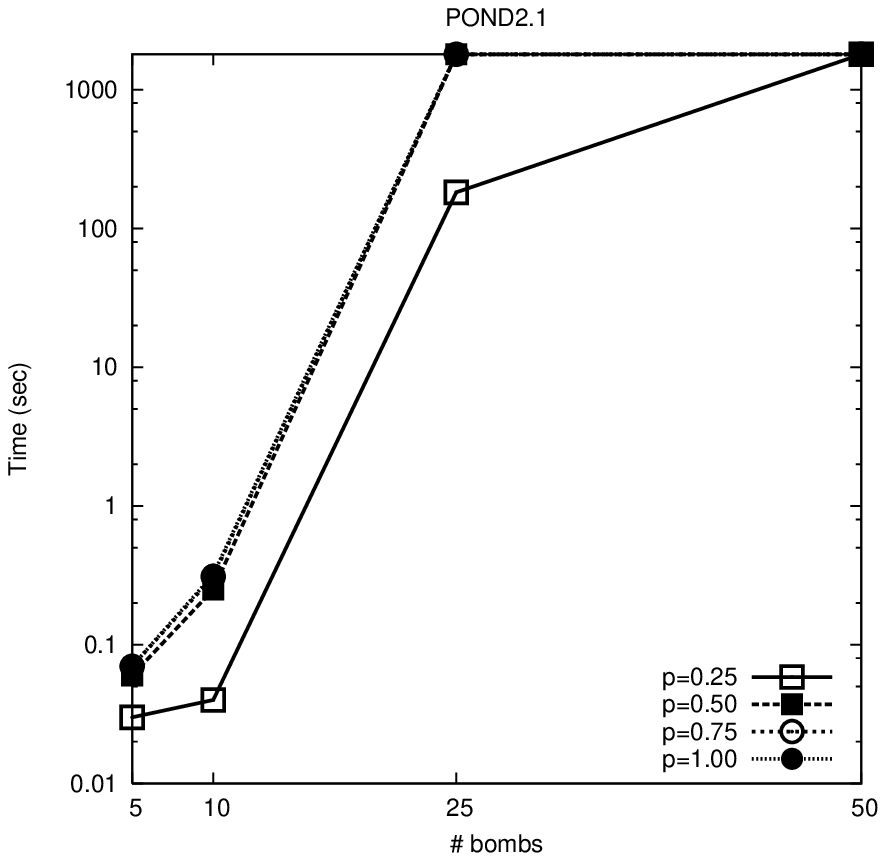}\\
\multicolumn{2}{c}{(d) 1 toilet}
\end{tabular}
\end{center}
\vspace{-0.2cm}
\caption{\label{graph:bomb2}The Bomb domain, \pff\ (left) vs.\ \pond\ (right).}
\end{figure}

Next we consider the famous Bomb-in-the-Toilet domain (or Bomb, for
short). Our version of Bomb contains $n$ bombs and $m$ toilets, where
each bomb may be armed or not armed {\em independently} with
probability $1/n$, resulting in huge numbers of initially possible
world states.  Dunking a bomb into an unclogged toilet disarms the
bomb, but clogs the toilet. A toilet can be unclogged by flushing it.
Table~\ref{T:results} shows that \pff\ scales nicely to $n=50$, and
becomes faster as $m$ increases. The latter is logical and desirable
as having more toilets means having more ``disarming devices'',
resulting in shorter plans needed. Figures~\ref{graph:bomb}
and~\ref{graph:bomb2} compare between \pff\ and \pond, plotting the
time performance of \pff\ on a linear scale, and that of \pond\ on a
logarithmic scale. The four pairs of graphs correspond to four choices
of number of toilets $m \in \{50,10,5,1\}$. The $x$-axes in all these
graphs correspond to the number of potentially armed bombs, where we
checked problems with $n \in \{5, 10, 25,
50\}$. Figure~\ref{graph:bomb} shows that this time \pff\ is at least
four orders of magnitude faster than \pond; At the extremes, while the
hardest combination of $n=50$, $m=1$, and $\goalprob=0.75$ took \pff\
less than 7 seconds, \pond\ timed-out on most of the problem
instances. In addition,
\begin{itemize}
\item In Bomb as well, \pff\ exhibit the nice pattern of improved
performance as we move from non-probabilistic ($\goalprob=1.0$) to
probabilistic planning (specifically, $\goalprob\leq 0.5$; for
$\goalprob\leq 0.25$, the initial state is good enough already).
\item While the performance of \pff\ improves with the number of
toilets, \pond\ seems to exhibit the inverse dependence, that is,
being more sensitive to the number of states in the problem (see
Table~\ref{T:results}) rather to the optimal solution depth.
\end{itemize}
Finally, we remark that, though length-optimality is not explicitly
required in probabilistic conformant planning, for all of Safe, Cube,
and Bomb, \pff's plans are optimal (the shortest possible).

Our next three domains are adaptations of benchmarks from
deterministic planning: ``Logistics'', ``Grid'', and ``Rovers''. We
assume that the reader is familiar with these domains. Each
Logistics-$x$ instance contains 10 cities, 10 airplanes, and 10
packages, where each city has $x$ locations. The packages are with
chance $0.88$ at the airport of their origin city, and uniformly at
any of the other locations in that city. The effects of all loading
and unloading actions are conditional on the (right) position of the
package. Note that higher values of $x$ increase not only the space of
world states, but also the initial uncertainty.
Grid is the complex grid world run in the AIPS'98 planning
competition~\cite{ipc1}, featuring locked positions that must be opened with
matching keys.  Each Grid-$x$ here is a modification of instance nr. 2
(of 5) run at AIPS'98, with a $6 \times 6$ grid, 8 locked positions,
and 10 keys of which 3 must be transported to a goal position. Each
lock has $x$ possible, uniformly distributed shapes, and each of the 3
goal keys has $x$ possible, uniformly distributed initial
positions. The effects of pickup-key, putdown-key, and open-lock
actions are conditional.



Finally, our last set of problems comes from three cascading
modifications of instance nr. 7 (of 20) of the Rovers domain used at
the AIPS'02 planning competition. This problem instance has 6
waypoints, 3 rovers, 2 objectives, and 6 rock/soil samples.  From
Rovers to RoversPPP we modify the instance/domain as follows.
\begin{itemize}
\item Rovers is the original AIPS'02 problem instance nr. 7, and we
use it hear mainly for comparison.
\item In RoversP, each sample is with chance $0.8$ at its original
waypoint, and with chance $0.1$ at each of the others two waypoints. Each
objective may be visible from 3 waypoints with uniform distribution
(this is a probabilistic adaptation of the domain suggested by~\citeR{CAltAlt}).
\item RoversPP enhances RoversP by {\em conditional} probabilities in
the initial state, stating that whether or not an objective is visible
from a waypoint depends on whether or not a rock sample (intuition: a
large piece of rock) is located at the waypoint.  The probability of
visibility is much higher if the latter is not the case. Specifically,
the visibility of each objective depends on the locations of two rock
samples, and if a rock sample is present, then the visibility
probability drops to $0.1$.
\item RoversPPP extends RoversPP by introducing the need to collect
data about water existence. Each of the soil samples has a certain
 probability ($< 1$) to be ``wet''. For communicated sample data, an
additional operator tests whether the sample was wet. If so, a fact
``know-that-water'' contained in the goal is set to true.  The
probability of being wet depends on the location of the sample.
\end{itemize}
We show no runtime plots for Logistics, Grid, and Rovers, since \pond\
runs out of either time or memory on all considered instances of these
domains. Table~\ref{T:results} shows that the scaling behavior of
\pff\ in these three domains is similar to that observed in the
previous domains. The goals in the RoversPPP problem cannot be
achieved with probabilities $\goalprob \in \{0.75, 1.0\}$. This is
proved by \pff's {\em heuristic function}, providing the correct
answer in split seconds.


\subsection{Probabilistic Actions}
\label{ss:results2}

Our first two domains with probabilistic actions are the famous
``Sand-Castle''~\cite{Maxplan} and ``Slippery-Gripper''~\cite{Buridan}
domains. The domains are simple, but they posed the first challenges
for probabilistic planners; our performance in these domains serves an
indicator of the progress relative to previous ideas for probabilistic
planning.

\begin{figure}[tb]
\begin{center}
\begin{tabular}{cc}
\includegraphics[width=9.0cm]{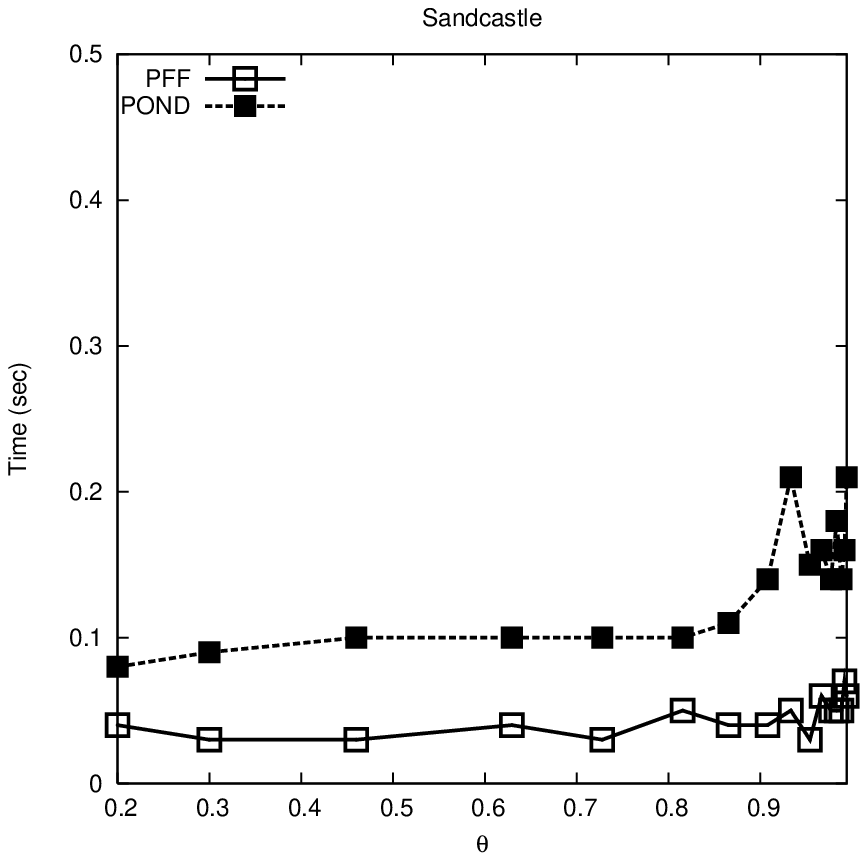} & \hspace{-1.5cm}
\includegraphics[width=9.0cm]{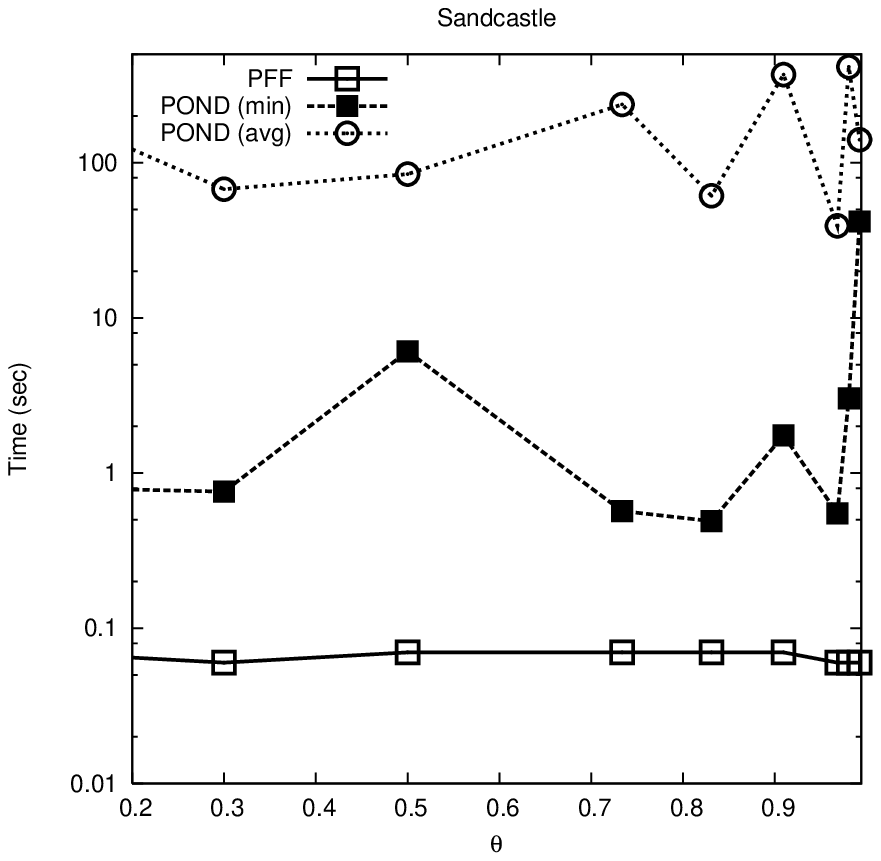}\\
(a) & (b)
\end{tabular}
\end{center}
\vspace{-0.2cm}
\caption{\label{graph:sandcastle} \pff\ and \pond\ on problems from (a) Sand-Castle, and (b) Slippery-Gripper.}
\end{figure}

In Sand-Castle, the states are specified by two boolean variables
$moat$ and $castle$, and state transitions are given by two actions
$dig\mbox{-}moat$ and $erect\mbox{-}castle$. The goal is to erect the
castle. Building a moat with $dig\mbox{-}moat$ might fail with
probability $0.5$. Erecting a castle with $erect\mbox{-}castle$
succeeds with probability $0.67$ if the moat has already been built,
and with probability $0.25$, otherwise. If failed,
$erect\mbox{-}castle$ also destroys the moat with probability
$0.5$. Figure~\ref{graph:sandcastle}(a) shows that both \pff\ and
\pond\ solve this problem in less than a second for arbitrary high
values of $\goalprob$, with the performance of both planners being
almost independent of the required probability of success.

Slippery-Gripper is already a bit more complicated domain. The states
in Slippery-Gripper are specified by four boolean variables
$grip\mbox{-}dry$, $grip\mbox{-}dirty$, $block\mbox{-}painted$, and
$block\mbox{-}held$, and there are four actions $dry$, $clean$,
$paint$, and $pickup$. In the initial state, the block is neither
painted nor held, the gripper is clean, and the gripper is dry with
probability $0.7$.  The goal is to have a clean gripper holding a
painted block.  Action $dry$ dries the gripper with probability
$0.8$. Action $clean$ cleans the gripper with probability
$0.85$. Action $paint$ paints the block with probability $1$, but
makes the gripper dirty with probability $1$ if the block was held,
and with probability $0.1$ if it was not. Action $pickup$ picks up the
block with probability $0.95$ if the gripper is dry, and with
probability $0.5$ if the gripper is wet.

Figure~\ref{graph:sandcastle}(b) depicts (on a log-scale) the relative
performance of \pff\ and \pond\ on Slippery-Gripper as a function of
growing $\theta$. The performance of \pff\ is nicely flat around
$0.06$ seconds. This time, the comparison with \pond\ was somewhat
problematic, because, for any fixed $\theta$, \pond\ on
Slippery-Gripper exhibited a huge variance in runtime. In
Figure~\ref{graph:sandcastle}(b) we plot the best runtimes for \pond,
as well as its average runtimes.  The best run-times for \pond\ for
different values of $\theta$ vary around a couple of seconds, but the
average runtimes are significantly worse. (For some high values of
$\theta$ \pond\ timed-out on some sample runs, and thus the plot
provides a lower bound on the average runtimes.)

\begin{figure}[tb]
\begin{center}
\begin{tabular}{cc}
\includegraphics[width=9.0cm]{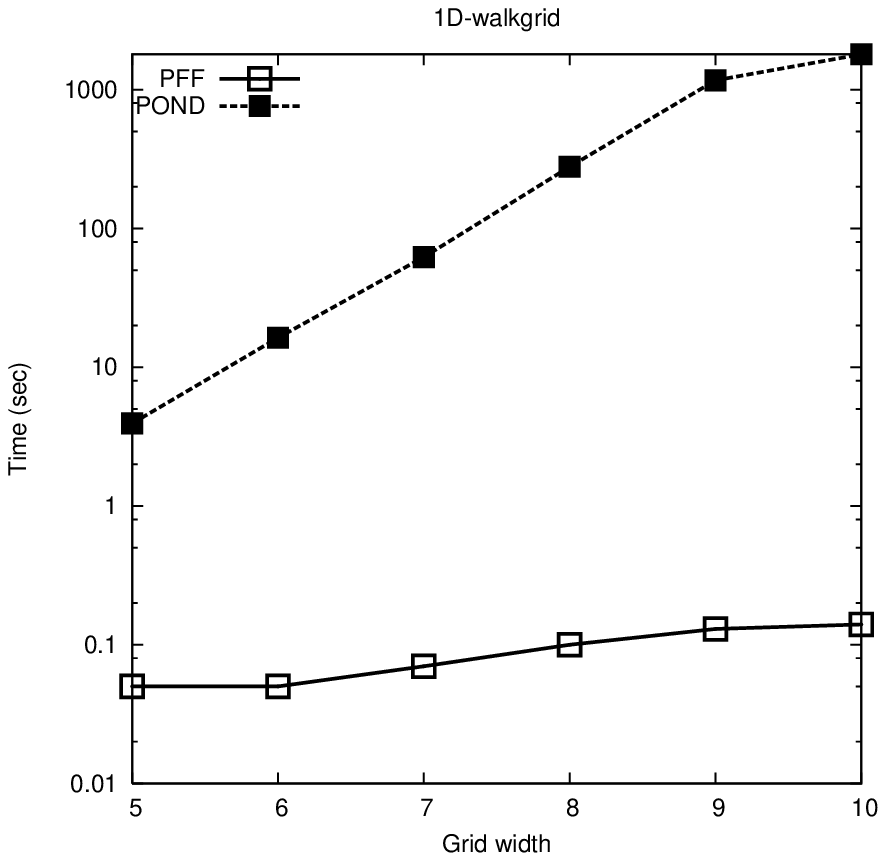} & \hspace{-1.5cm}
\includegraphics[width=9.0cm]{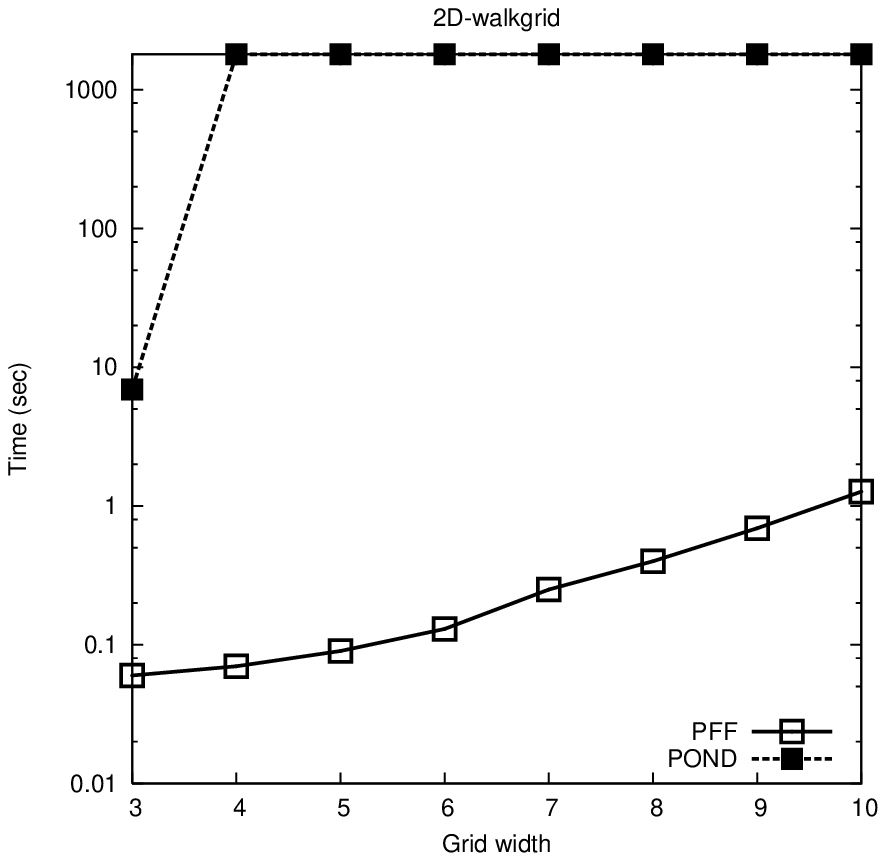}\\
(a) & (b)
\end{tabular}
\end{center}
\vspace{-0.2cm}
\caption{\label{graph:walkgrid} \pff\ and \pond\ on problems from (a)
1D-WalkGrid with $\theta=0.9$, and (b) 2D-WalkGrid with
$\theta=0.01$.}
\end{figure}

In the next two domains, ``1D-WalkGrid'' and ``2D-WalkGrid'', the
robot has to pre-plan a sequence of conditional movements taking it
from a corner of the grid to the farthest (from the initial position)
corner~\cite{CPP2}. In 1D-WalkGrid the grid is one-dimensional, while
in 2D-WalkGrid the grid is two-dimensional.
Figure~\ref{graph:walkgrid}(a) depicts (on a log-scale) a snapshot of
the relative performance of \pff\ and \pond\ on one-dimensional grids
of width $n$ and $\theta = 0.9$. The robot is initially at $(1,1)$,
should get to $(1,n)$, and it can try moving in each of the two
possible directions. Each of the two movement actions moves the robot
in the right direction with probability $0.8$, and keeps it in place
with probability $0.2$. It is easy to see from
Figure~\ref{graph:walkgrid}(a) that the difference between the two
planners in this domain is substantial---while runtime of \pff\ grows
only linearly with $x$, the same dependence for \pond\ is seemingly
exponential.

The 2D-WalkGrid domain is already much more challenging for
probabilistic planning. In all 2D-WalkGrid problems with $n \times n$
grids the robot is initially at $(1,1)$, should get to $(n,n)$, and it
can try moving in each of the four possible directions. Each of the
four movement actions advances the robot in the right direction with
probability $0.8$, in the opposite direction with probability $0$, and
in either of the other two directions with probability $0.1$.
Figure~\ref{graph:walkgrid}(a) depicts (on a log-scale) a snapshot of
the relative performance of \pff\ and \pond\ on 2D-WalkGrid with very
low required probability of success $\theta = 0.01$, and this as a
function of the grid's width $n$. The plot shows that \pff\ still
scales well with increasing $n$ (though not linearly anymore), while
\pond\ time-outs for all grid widths $n > 3$. For higher values of
$\theta$, however, \pff\ does reach the time-out limit on rather small
grids, notably $n=6$ and $n=5$ for $\theta=0.25$ and $\theta=0.5$,
respectively. The reason for this is that \pff's heuristic function is
not good enough at estimating how many times, at an early point in the
plan, a probabilistic action must be applied in order to sufficiently
support a high goal threshold at the end of the plan. We explain this
phenomenon in more detail at the end of this section, where we find
that it also appears in a variant of the well-known Logistics domain.

Our last set of problems comes from the standard Logistics
domain. Each problem instance $x\mbox{-}y\mbox{-}z$ contains $x$
locations per city, $y$ cities, and $z$ packages. We will see that
\pff\ scales much worse, in Logistics, in the presence of
probabilistic effects than if there is ``only'' initial state
uncertainty (we will explain the reason for this at the end of this
section). Hence we use much smaller instances than the ones used above
in Section~\ref{ss:results1}. Namely, to allow a direct comparison to
previous results in this domain, we closely follow the specification of~\citeA{CPP2}. We use
instances with configurations $x\mbox{-}y\mbox{-}z =$
$2\mbox{-}2\mbox{-}2$, $4\mbox{-}2\mbox{-}2$, and
$2\mbox{-}2\mbox{-}4$, and distinguish between two levels of
uncertainty.
\begin{itemize}
\item $L\mbox{-}x\mbox{-}y\mbox{-}z$ correspond to problems with
uncertainty only in the outcome of the $load$ and $unload$
actions. Specifically, the probabilities of success for $load$ are
$0.875$ for trucks and $0.9$ for airplanes, and for $unload$, $0.75$
and $0.8$, respectively.
\item $LL\mbox{-}x\mbox{-}y\mbox{-}z$ extends
$L\mbox{-}x\mbox{-}y\mbox{-}z$ with independent uniform priors for
each initial location of a package within its start city.
\end{itemize}

\begin{figure}[tb]
\begin{center}
\begin{tabular}{ccc}
\multicolumn{3}{c}{
\includegraphics[width=0.42\textwidth]{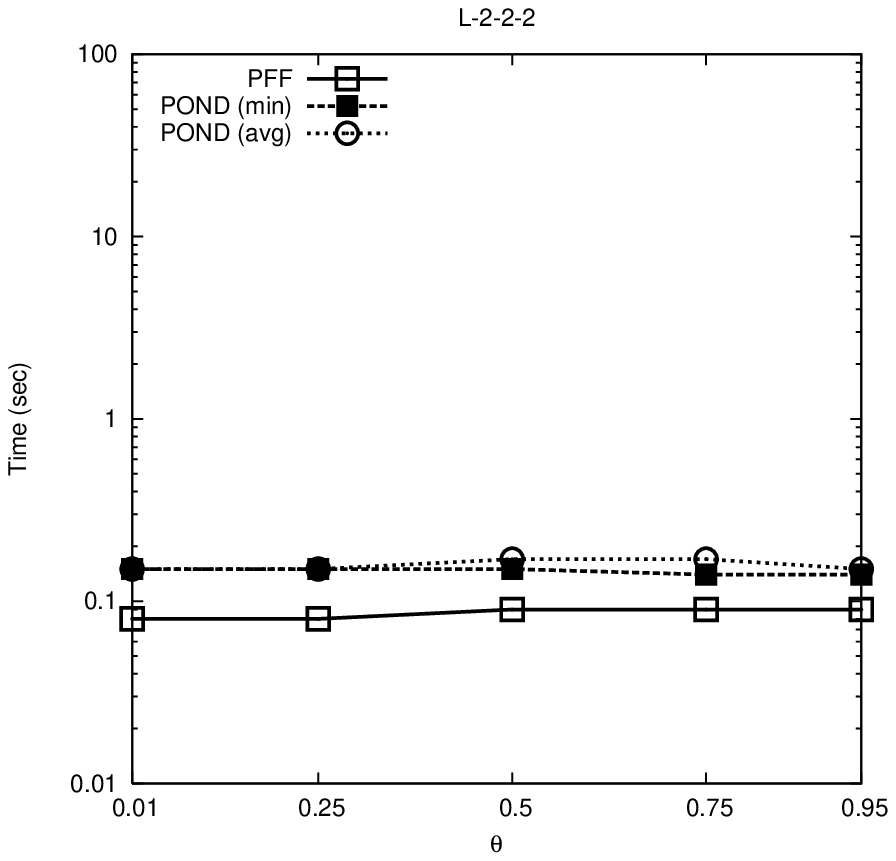}
\hspace{-1.7cm} \includegraphics[width=0.42\textwidth]{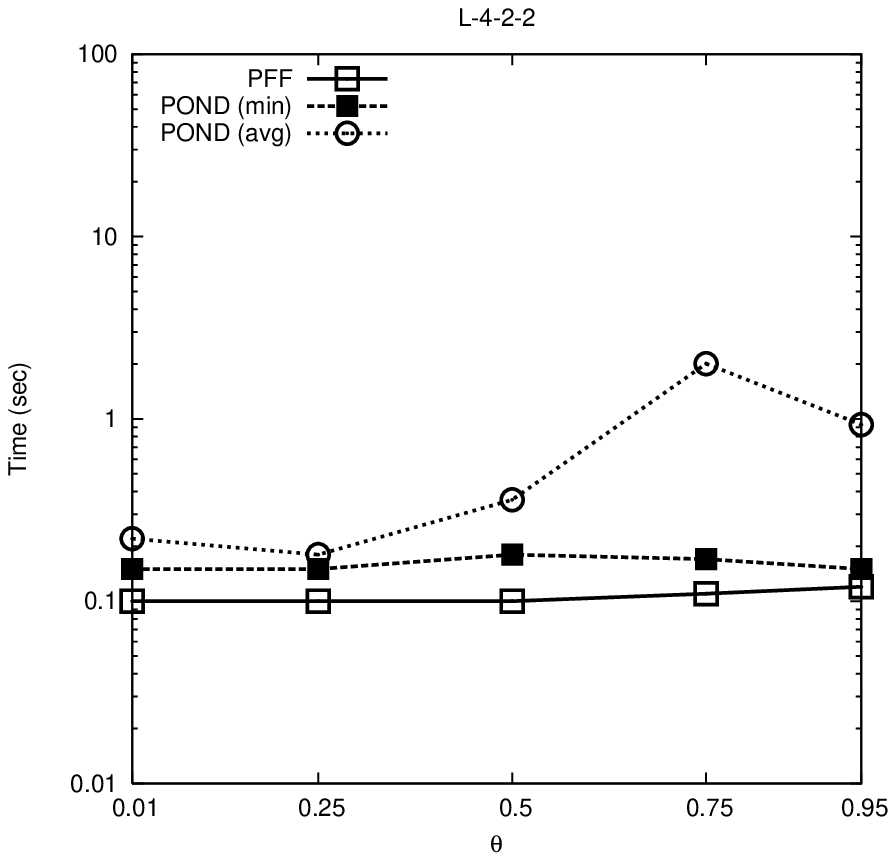}\hspace{-1.5cm} \includegraphics[width=0.42\textwidth]{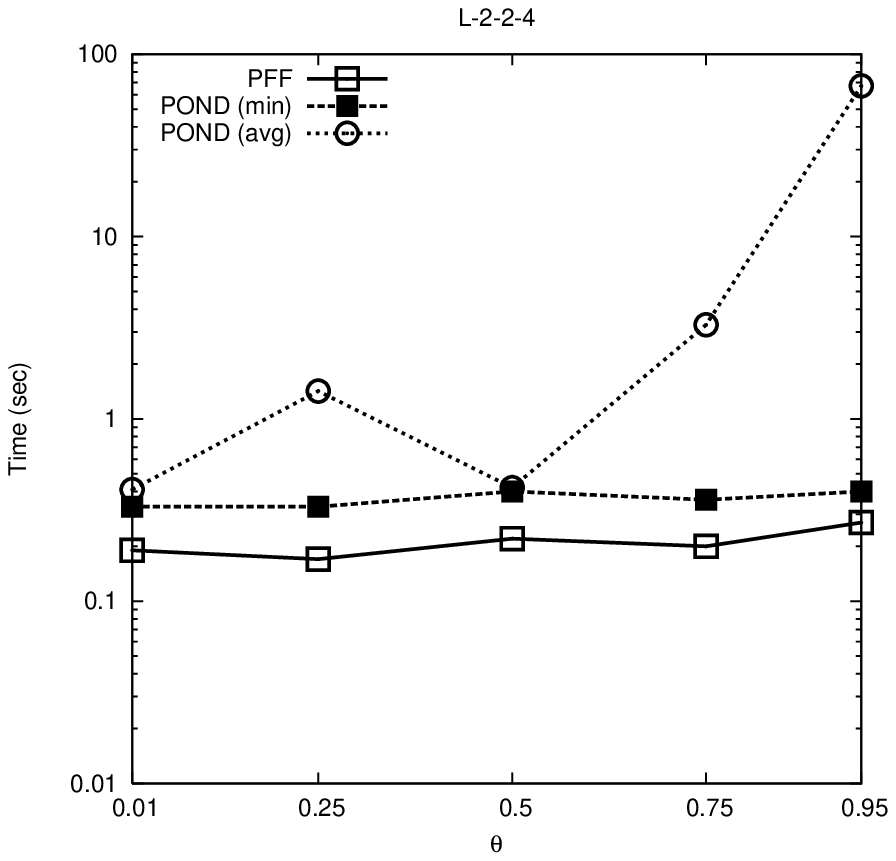}}\\
\begin{minipage}{0.3\textwidth}
\centering{(a)}
\end{minipage} & 
\begin{minipage}{0.3\textwidth}
\centering{(b)}
\end{minipage} &
\begin{minipage}{0.3\textwidth}
\centering{(c)}
\end{minipage} 
\end{tabular}
\end{center}
\vspace{-0.2cm}
\caption{\label{graph:log-ond} \pff\ and \pond\ on problems from
Logistics (a) $L\mbox{-}2\mbox{-}2\mbox{-}2$, (b)
$L\mbox{-}4\mbox{-}2\mbox{-}2$, and (c)
$L\mbox{-}2\mbox{-}2\mbox{-}4$.}
\end{figure}

Figure~\ref{graph:log-ond} depicts (on a log scale) runtimes of \pff\
and \pond\ on $L\mbox{-}2\mbox{-}2\mbox{-}2$,
$L\mbox{-}4\mbox{-}2\mbox{-}2$, and $L\mbox{-}2\mbox{-}2\mbox{-}4$, as
a function of growing $\goalprob$. On these problems, both planners
appear to scale well, with the runtime of \pff\ and the optimal
runtime of \pond\ being roughly the same, and the average runtime of
\pond\ somewhat degrading from $2\mbox{-}2\mbox{-}2$ to
$4\mbox{-}2\mbox{-}2$ to $2\mbox{-}2\mbox{-}4$. This shows that both
planners are much more efficient in this domain than the previously
known SAT and CSP based techniques. However, moving to
$LL\mbox{-}x\mbox{-}y\mbox{-}z$ changes the picture for both
planners. The results are as follows:
\begin{enumerate}
\item On $LL\mbox{-}2\mbox{-}2\mbox{-}2$, the runtimes of \pff\ were
identical to those on $L\mbox{-}2\mbox{-}2\mbox{-}2$, and the optimal
runtimes of \pond\ only slightly degraded to $2-8$ seconds. However,
for all examined values of $\goalprob$, some runs of \pond\ resulted
in timeouts.
\item On $LL\mbox{-}4\mbox{-}2\mbox{-}2$, the runtimes of \pff\ were
identical to those on $L\mbox{-}4\mbox{-}2\mbox{-}2$ for $\goalprob
\in \{0.01,0.25,0.5,0.75\}$, yet \pff\ time-outed on
$\goalprob=0.95$. The optimal runtimes of \pond\ degraded from those
for $L\mbox{-}4\mbox{-}2\mbox{-}2$ only to $9-18$ seconds, and again,
for all values of $\goalprob$, some runs of \pond\ resulted in
timeouts.
\item On $LL\mbox{-}2\mbox{-}2\mbox{-}4$, \pff\ experienced hard
times, finishing in $0.19$ seconds for $\goalprob=0.01$, and
time-outing for all other examined values of $\goalprob$. The optimal
runtimes of \pond\ degraded from those for
$L\mbox{-}2\mbox{-}2\mbox{-}4$ to $120-700$ seconds, and here as well,
for all values of $\goalprob$, some runs of \pond\ resulted in
timeouts.
\end{enumerate}
We also tried a variant of $LL\mbox{-}x\mbox{-}y\mbox{-}z$ with
non-uniform priors over the initial locations of the packages, but
this resulted in a qualitatively similar picture of absolute and
relative performance.

The $LL\mbox{-}x\mbox{-}y\mbox{-}z$ domain remains challenging, and
deserves close attention in the future developments for probabilistic
planning. In this context, it is interesting to have a close look at
what the reasons for the failure of \pff\ is. It turns out that \pff\
is not good enough at estimating how many times, at an early point in
the plan, a probabilistic action must be applied in order to
sufficiently support a high goal threshold at the end of the plan. To
make this concrete, consider a Logistics example with uncertain
effects of all load and unload actions. Consider a package P that must
go from a city A to a city B. Let's say that P is initially not at A's
airport. If the goal threshold is high, this means that, to be able to
succeed, the package has to be brought to A's airport with a high
probability {\em before} loading it onto an airplane. This is exactly
the point where \pff's heuristic function fails. The relaxed plan
contains too few actions unloading P at A's airport. The effect is
that the search proceeds too quickly to loading P onto a plane and
bringing it to B. Once the search gets to the point where B should be
unloaded to its goal location, the goal threshold cannot be achieved
no matter how many times one unloads P. At this point, \pff's enforced
hill-climbing enters a loop and eventually fails because the relaxed
plan (which over-estimates the past achievements) becomes
empty.\footnote{This does not happen in the above
$L\mbox{-}2\mbox{-}2\mbox{-}2$, $L\mbox{-}4\mbox{-}2\mbox{-}2$, and
$L\mbox{-}2\mbox{-}2\mbox{-}4$ instances simply because they are too
small and a high goal probability can be achieved without thinking too
much about the above problem; if one increases the size of these
instances, the problem appears. The problem appears earlier in the
presence of initial state uncertainty -- even in small instances such
as $LL\mbox{-}2\mbox{-}2\mbox{-}2$, $LL\mbox{-}4\mbox{-}2\mbox{-}2$,
and $LL\mbox{-}2\mbox{-}2\mbox{-}4$ -- because with uncertainty about
the start position of the packages one needs to try unloading them at
the start airports more often.}

The challenge here is to devise methods that are better at recognizing
how many times P has to be unloaded at A's airport in order to
sufficiently support the goal threshold. The error made by \pff\ lies
in that our propagation of weights on the implication graph
over-estimates the goal probability. Note here that this is much more
critical for actions that must be applied early on in the plan, than
for actions that are applied later. If an action $a$ appears early on
in a plan, then the relaxed plan, when $a$ is executed, will be
long. Recall that the weight propagation proceeds backwards, from the
goal towards the current state. At each single backwards step, the
propagation makes an approximation that might lose precision of the
results. Over several backwards steps, these imprecisions
accumulate. Hence the quality of the approximation decreases quickly
over the number of backwards steps. The longer the distance between
goal and current state is, the more information is lost. We have
observed this phenomenon in detailed experiments with different weight
propagation schemes, that is, with different underlying
assumptions. Of the propagation schemes we tried, the independence
assumption, as presented in this paper, was by far the most accurate
one. All other schemes failed to deliver good results even for much
shorter distances between the goal and the current state.

It is interesting to consider how this issue affects \pond, which uses
a very different method for estimating the probability of goal
achievement: instead of performing a backwards propagation and
aggregation of weight values, \pond\ sends a set of random particles
through the relaxed planning graph in a forward fashion, and stops the
graph building if enough particles end up in the goal. From our
empirical results, it seems that this method suffers from similar
difficulties as \pff, but not to such a large extent. \pond's optimal
runtimes for $LL\mbox{-}x\mbox{-}y\mbox{-}z$ are much higher than
those for $L\mbox{-}x\mbox{-}y\mbox{-}z$. This indicates that it is
always challenging for \pond\ to ``recognize'' the need for applying
some action $a$ many times early on in the plan. More interestingly,
\pond\ never times-out in $L\mbox{-}x\mbox{-}y\mbox{-}z$, but it does
often time-out in $LL\mbox{-}x\mbox{-}y\mbox{-}z$. This indicates
that, to some extent, it is a matter of chance whether or not \pond's
random particles recognize the need for applying an action $a$ many
times early on in the plan. An intuitive explanation is that the
``good cases'' are those where sufficiently many of the particles
failed to reach the goal due to taking the ``wrong effect'' of
$a$. Based on this intuition, one would expect that it helps to
increase the number of random particles in \pond's heuristic
function. We did so, running \pond\ on $LL\mbox{-}x\mbox{-}y\mbox{-}z$
with an increased number of particles, $200$ and $600$ instead of the
default value of $64$. To our surprise, the qualitative behavior of
\pond\ did not change, time-outing in a similar number of cases. It is
unclear to us what the reason for this phenomenon is. Certainly, it
can be observed that the situation encoded in
$LL\mbox{-}x\mbox{-}y\mbox{-}z$ is not solved to satisfaction by
either of \pff's weight propagation or \pond's random particle
methods, in their current configurations.

At the time of writing, it is unclear to the authors how better
methods could be devised. It seems unlikely that a weight propagation
-- at least one that does not resort to expensive reasoning -- exists
which manages long distances better than the independence assumption.
An alternative way out might be to simply define a weaker notion of
plans that allows to repeat certain kinds of actions -- throwing a
dice or unloading a package -- arbitrarily many times. However, since
our assumption is that we do not have any observability during plan
execution, when executing such a plan there would still arise the
question how often an action should be tried. Since Logistics is a
fairly well-solved domain in simpler formalisms -- by virtue of \pff,
even in the probabilistic setting as long as the effects are
deterministic -- we consider addressing this problem as a quite
pressing open question.

\section{Conclusion}
\label{ss:conclusion}

We developed a probabilistic extension of \cff's search space
representation, using a synergetic combination of \cff's SAT-based
techniques with recent techniques for weighted model counting. We
further provided an extension of conformant relaxed planning with
approximate probabilistic reasoning. The resulting planner scales well
on a range of benchmark domains. In particular it outperforms its only
close relative, \pond, by at least an order of magnitude in almost all
of the cases we tried.

While this point may be somewhat obvious, we would like to emphasize
that our achievements do {\em not} solve the (this particular) problem
once and for all. \pff\ inherits strengths {\em and} weaknesses from
FF and \cff, like domains where \ff's or \cff's heuristic functions
yield bad estimates (e.g. the mentioned Cube-center variant). What's
more, the probabilistic setting introduces several new potential
impediments for \ff's performance. For one thing, weighted model
counting is inherently harder than SAT testing. Though this did not
happen in our set of benchmarks, there are bound to be cases where the
cost for exact model counting becomes prohibitive even in small
examples. A promising way to address this issue lies in recent methods
for {\em approximate} model
counting~\cite{gomes:etal:aaai06,gomes:etal:ijcai07}. Such methods are
much more efficient than exact model counters. They provide
high-confidence lower bounds on the number of models. The lower bounds
can be used in \pff\ in place of the exact counts. It has been shown
that good lower bounds with very high confidecne can be achieved
quickly. The challenge here is to extend the methods -- which are
currently designed for non-weighted CNFs -- to handle {\em weighted}
model counting.

More importantly perhaps, in the presence of probabilistic
effects there is a fundamental weakness in \pff's -- and \pond's --
heuristic information. This becomes a pitfall for performance even in
a straightforward adaptation of the Logistics domain, which is
otherwise very easy for this kind of planners. As outlined, the key
problem is that, to obtain a high enough confidence of goal
achievement, one may have to apply particular actions several times
early on in the plan. Neither \pff's nor \pond's heuristics are good
enough at identifying {\em how many} times. In our view, finding
techniques that address this issue is currently the most important
open topic in this area.

Apart from addressing the latter challenge, we intend to work towards
applicability in real-word settings. Particularly, we will look at the
space application settings that our Rovers domain hints at, at
medication-type treatment planning domains, and at the power supply
restoration domain~\cite{psr}.

%

%

%

%

%

%

%

%

%

%


\section*{Acknowledgments}
The authors would like to thank Dan Bryce and
Rao Kambhampati for providing a binary distribution of
\pond2.1. Carmel Domshlak was partially supported by the Israel Science Foundations grant 2008100, as well as by the C. Wellner Research Fund. Some major parts of this research have been accomplished at the time that J\"org Hoffmann was employed at the Intelligent Information Systems Institute, Cornell University. 

\appendix
\setcounter{theorem}{1}

\section{Proofs}
\label{ss:proofs}


\begin{proposition}
\label{prop:sizecompA}
Let $(\actions,\initBN,\goal,\goalprob)$ be a probabilistic planning
problem described over $k$ state variables, and $\seqactions$ be an
$m$-step sequence of actions from $\actions$. Then, we have
$|\BN_{\belief_{\seqactions}}| = O(|\initBN| + m\alpha(k+1))$
  where $\alpha$ is the largest
description size of an action in $\actions$.
\end{proposition}

\begin{proof}
The proof is rather straightforward, and it exploits the local
structure of $\BN_{\belief_{\seqactions}}$'s CPTs. The first
nodes/CPTs layer $\BNvars\tstamp{0}$ of $\BN_{\belief_{\seqactions}}$
constitutes an exact copy of $\initBN$. Then, for each $1 \leq t \leq
m$, the $t$-th layer of $\BN_{\belief_{\seqactions}}$ contains $k+1$
node $\{Y\tstamp{t}\}\cup\BNvars\tstamp{t}$.

First, let us consider the ``action node'' $Y\tstamp{t}$. While
specifying the CPT $T_{Y(t)}$ in a straightforward manner as if
prescribed by Eq.~\ref{e:cpt1p} might result in an exponential blow
up, the same Eq.~\ref{e:cpt1p} suggests that the original description
of $a\tind{t}$ is by itself a compact specification of
$T_{Y(t)}$. Therefore, $T_{Y(t)}$ can be described in space
$O(\alpha)$, and this description can be efficiently used for
answering queries $T_{Y(t)}(Y\tstamp{i} = \poutcome \mid \pi)$ as in
Eq.~\ref{e:cpt1p}.  Next, consider the CPT $T_{X(t)}$ of a
state-variable node $X(t) \in \BNvars\tstamp{t}$. This time, it is
rather evident from Eq.~\ref{e:cpt1d} that $T_{X(t)}$ can be described
in space $O(\alpha)$ so that queries $T_{X(t)}(X\tstamp{t} = x \mid
X\tstamp{t-1} = x')$ could be efficiently answered. Thus, summing up
for all layers $1 \leq t \leq m$, the description size of
$|\BN_{\belief_{\seqactions}}| = O(|\initBN| + m\alpha(k+1))$
\end{proof}

\setcounter{theorem}{3}
\begin{lemma}
\label{lemma:subweight-completeA}
Given a node $v(t')\in\mbox{\sl Imp}_{\rightarrow p(t)}$, we have
$\subweight_{p(t)}\left(v(t')\right) = \weight\left(v(t')\right)$ if
and only if, given $v$ at time $t'$, the sequence of effects
$\grapheffects{\mbox{\sl Imp}_{v(t')\rightarrow p(t)}}$ achieves $p$
at $t$ with probability $1$.
\end{lemma}

\begin{proof}
The proof of Lemma~\ref{lemma:subweight-complete} is by a backward
induction on the time layers of $\mbox{\sl Imp}_{v(t')\rightarrow
p(t)}$. For time $t$, the only node of $\mbox{\sl Imp}_{\rightarrow
p(t)}$ time-stamped with $t$ is $p(t)$ itself.  For this node we do
have $\subweight_{p(t)}\left(p(t)\right) = \weight\left(p(t)\right) =
1$, but, given $p$ at time $t$, an empty plan corresponding to (empty)
$\grapheffects{\mbox{\sl Imp}_{p(t)\rightarrow p(t)}}$ trivially
``re-establishes'' $p$ at $t$ with certainty.  Assuming now that the
claim holds for all nodes of $\mbox{\sl Imp}_{\rightarrow p(t)}$ time
stamped with $t'+1,\dots,t$, we now show that it holds for the nodes
time stamped with $t'$.

It is easy to see that, for any node $v(t') \in \mbox{\sl
Imp}_{\rightarrow p(t)}$, we get $\subweight_{p(t)}\left(v(t')\right)
= \weight\left(v(t')\right)$ only if $\flbweight$ goes down to zero.
First, consider the chance nodes $\poutcome(t')\in\mbox{\sl
Imp}_{v\rightarrow p(t)}$. For such a node, $lb$ is set to zero if and
only if we have $\subweight_{p(t)}\left(r(t'+1)\right) = 1$ for some
$r \in \add(\poutcome)$. However, by our inductive assumption, in this
and only in this case the effects $\grapheffects{\mbox{\sl
Imp}_{\poutcome(t')\rightarrow p(t+1)}}$ achieve $p$ at $t$ with
probability $1$, given the occurrence of $\poutcome$ at time $t'$.

Now, consider the fact nodes $q(t')\in\mbox{\sl Imp}_{v\rightarrow
p(t)}$. For such a node, $\flbweight$ can get nullified only by some
effect $e \in \effs(a), a \in A(t'),\con(e)=q$.
The latter happens if
only if, for {\em all} possible outcomes of $e$, (i) the node
$\poutcome(t')$ belongs to $\mbox{\sl Imp}_{\rightarrow p(t)}$, and
(ii) and the estimate $\subweight_{p(t)}\left(\poutcome(t')\right) =
\weight(\poutcome(t'))$. In other words, by our inductive assumption,
given {\em any} outcome $\poutcome \in \poutcomeset(e)$ at time $t'$,
the effects $\grapheffects{\mbox{\sl Imp}_{\poutcome(t')\rightarrow
p(t)}}$ achieve $p$ at $t$ with probability $1$. Thus, given $q$ at
time $t'$, the effects $\grapheffects{\mbox{\sl Imp}_{q(t')\rightarrow
p(t)}}$ achieve $p$ at $t$ with probability $1$ {\em independently} of
the actual outcome of $e$.  Alternatively, if for $q(t')$ we have $lb
> 0$, then for each effect $e$ conditioned on $q(t)$, there exists an
outcome $\poutcome$ of $e$ such that, according to what we just proved
for the chance nodes time-stamped with $t'$, the effects
$\grapheffects{\mbox{\sl Imp}_{\poutcome(t')\rightarrow p(t+1)}}$ do
not achieve $p$ at $t$ with probability $1$. Hence, the whole set of
effects $\grapheffects{\mbox{\sl Imp}_{q(t')\rightarrow p(t+1)}}$ does
not achieve $p$ at $t$ with probability $1$.
\end{proof}

\begin{lemma}
\label{lemma:disjunctionA}
Let $(\actions,\initBN,\goal,\goalprob)$ be a probabilistic planning
task, $\seqactions$ be a sequence of actions applicable in
$\initprob$, and $\onecondrel$ be a relaxation function for $A$. For
each time step $t \geq -m$, and each proposition $p \in \fluents$, if
$P(t)$ is constructed by \buildprpg$(\seqactions, \actions,
\cinitial,\goal,\goalprob,\onecondrel)$, then $p$ at time $t$ can be
achieved by a relaxed plan starting with $\seqactions\onecondrel$
\begin{enumerate}[(1)]
\item with probability $> 0$  (that is, $p$ is not negatively known at time $t$) if and only if $p \in uP(t) \cup P(t)$, and
\item with probability $1$ (that is, $p$ is known at time $t$) if and only if $p \in P(t)$.
\end{enumerate}
\end{lemma}

\begin{proof}
The proof of the ``if'' direction is by a straightforward induction on
$t$. For $t=-m$ the claim is immediate by the direct initialization of
$uP(-m)$ and $P(-m)$. Assume that, for $-m \leq t' < t$, if $p\in
uP(t')\cup P(t')$, then $p$ is not negatively known at time $t'$, and
if $p\in P(t')$, then $p$ is known at time $t'$.

First, consider some $p(t) \in uP(t)\cup P(t)$, and suppose that $p$
is egatively know at time $t$. By the inductive assumption, and the
property of the PRPG construction that $uP(t-1)\cup P(t-1) \subseteq
uP(t)\cup P(t)$, we have $p \not\in uP(t-1)\cup P(t-1)$. Therefore,
$p$ has to be added into $uP(t)$ (and then, possibly, moved from there
to $P(t)$) in the first {\bf for} loop of the \buildts\ procedure.
However, if so, then there exists an action $a \in A(t-1)$, $e \in
\effs(a)$, and $\poutcome\in\poutcomeset(e)$ such that (i) $\con(e)
\in uP(t-1)\cup P(t-1)$, and (ii) $p \in \add(\poutcome)$. Again, by
the assumption of the induction we have that $\pre(a)$ is known at
time $t-1$, and $\con(e)$ is not negatively known at time
$t-1$. Hence, the non-zero probability of $\poutcome$ occurring at
time $t$ implies that $p$ can be achieved at time $t$ with probability
greater than $0$, contradicting that $p$ is negatively know at time
$t$.

Now, let us consider some $p(t) \in P(t)$. Notice that, for $t > -m$,
we have $p(t) \in P(t)$ if and only if
\begin{equation}
\label{e:disjaux}
\Phi \longrightarrow {\bigvee_{l \in \mbox{\scriptsize \sl support}(p(t))}{\hspace{-0.5cm} l}}\hspace{0.3cm}.
\end{equation}
Thus, for each world state $w$ consistent with $\initbelief$, we have
either $q \in w$ for some fact proposition $q(-m) \in \mbox{\sl
support}(p(t))$, or, for some effect $e$ of an action $a(t') \in
A(t')$, $t' < t$, we have $\con(e)\in P(t')$ and $\{\poutcome(t')\mid
\poutcome\in\poutcomeset(e)\} \subseteq \mbox{\sl support}(p(t))$.
In this first case, Lemma~\ref{lemma:subweight-complete} immediately
implies that the concatenation of $\seqactions\onecondrel$ with an
arbitrary linearization of the (relaxed) actions $A(0),\dots,A(t-1)$
achieves $p$ at $t$ with probability $1$, and thus $p$ is known at
time $t$. In the second case, our inductive assumption implies that
$\con(e)$ is known at time $t$, and together with
Lemma~\ref{lemma:subweight-complete} this again implies that the
concatenation of $\seqactions\onecondrel$ with an arbitrary
linearization of the (relaxed) actions $A(0),\dots,A(t-1)$ achieves
$p$ at $t$ with probability $1$.

The proof of the ``only if'' direction is by induction on $t$ as well.
For $t=-m$ this claim is again immediate by the direct initialization
of $P(-m)$. Assume that, for $-m \leq t' < t$, if $p$ is not
negatively known at time $t'$, then $p\in uP(t')\cup P(t')$, and if
$p$ is known at time $t'$, then $p\in P(t')$. First, suppose that $p$
is not negatively known at time $t$, and yet we have $p \not\in
uP(t)\cup P(t)$. From our inductive assumption plus that $A(t-1)$
containing all the NOOP actions for propositions in $uP(t-1)\cup
P(t-1)$, we know that $p$ {\em is} negatively known at time $t-1$.  If
so, then $p$ can become not negatively known at time $t$ only due to
some $\poutcome\in\poutcomeset(e)$, $e \in \effs(a)$, such that
$\pre(a)$ is known at time $t-1$, and $\con(e)$ is not negatively
known at time $t-1$. By our inductive assumption, the latter
conditions imply $\con(e) \in uP(t-1)\cup P(t-1)$, and $\pre(a) \in
P(t-1)$. But if so, then $p$ has to be added to $uP(t)\cup P(t)$ by
the first {\bf for} loop of the \buildts\ procedure, contradicting our
assumption that $p \not\in uP(t)\cup P(t)$.

Now, let us consider some $p$ known at time $t$.  By our inductive
assumption, $P(t-1)$ contains all the facts known at time $t-1$, and
thus $A(t-1)$ is the maximal subset of actions $\actions\onecondrel$
applicable at time $t-1$. Let us begin with an exhaustive
classification of the effects $e$ of the actions $A(t-1)$ with respect
to our $p$ at time $t$.
\begin{enumerate}[(I)]
\item \label{set2} 
      $\forall \poutcome\in\poutcomeset(e): p\in\add(\poutcome)$, and 
      $\con(e) \in P(t-1)$
\item \label{set4} 
      $\forall \poutcome\in\poutcomeset(e): p\in\add(\poutcome)$, and 
      $\con(e) \in uP(t-1)$
\item \label{set6}
      $\exists \poutcome\in\poutcomeset(e): p\not\in\add(\poutcome)$ or 
      $\con(e) \not\in P(t-1)\cup uP(t-1)$
\end{enumerate}
If the set (\ref{set2}) is not empty, then, by the construction of
\buildwil$(p(t),\mbox{\sl Imp})$, we have
\[
\{\poutcome(t-1) \mid \poutcome \in  \poutcomeset(e)\} \subseteq \mbox{\sl support}(p(t)),
\]
for each $e \in \mbox{(\ref{set2})}$. Likewise, by the construction of
\buildts\ (notably, by the update of $\Phi$), for each $e \in
\mbox{(\ref{set2})}$, we have
\[
\Phi \longrightarrow \bigvee_{\{\poutcome(t-1) \mid \poutcome \in  \poutcomeset(e)\}}\poutcome(t-1).
\]
Putting these two facts together, we have that Eq.~\ref{e:disjaux}
holds for $p$ at time $t$, and thus we have $p \in P(t)$.

Now, suppose that the set (\ref{set2}) is empty.  It is not hard to
verify that no subset of {\em only} effects (\ref{set6}) makes $p$
known at time $t$. Thus, the event ``at least one of the effects
(\ref{set4}) occurs'' must hold with probability $1$. First, by the
construction of \buildwil$(p(t),\mbox{\sl Imp})$, we have
\[
\mbox{\sl support}\left(p(t)\right) \supseteq \bigcup_{e\in\mbox{\scriptsize (\ref{set4})}}
{\mbox{\sl support}\left(\con(e)(t-1)\right)}
\]
Then, 
and~\ref{lemma:subweight-complete} from
Lemma~\ref{lemma:subweight-complete} we have that the event ``at least
one of the effects (\ref{set4}) occurs'' holds with probability $1$ if
and only if
\[
\Phi \longrightarrow \bigvee_{\substack{e\in\mbox{\scriptsize (\ref{set4})}\\ l\in \mbox{\scriptsize \sl support}\left(\con(e)(t-1)\right)}}\hspace{-0.6cm}l
\]
Putting these two facts together, we have that Eq.~\ref{e:disjaux}
holds for $p$ at time $t$, and thus we have $p \in P(t)$.
\end{proof}


\begin{theorem}
\label{t:completeA}
Let $(\actions,\initBN,\goal,\goalprob)$ be a probabilistic planning
task, $\seqactions$ be a sequence of actions applicable in
$\initprob$, and $\onecondrel$ be a relaxation function for $A$. If
\buildprpg$(\seqactions, \actions, \cinitial,\goal,
\goalprob,\onecondrel)$ returns FALSE, 
then there is no relaxed plan for
$(\actions,\initbelief,\goal,\goalprob)$ that starts with $\seqactions\onecondrel$.
\end{theorem}

\begin{proof}
Let $t > 0$ be the last layer of the PRPG upon the termination of
\buildprpg. For every $-m \leq t' \leq t$, by the construction of PRPG and Lemma~\ref{lemma:disjunction}, the sets $P(t')$ and
$uP(t')$ contain all (and only all) propositions that are known
(respectively unknown) after executing all the actions in the action
layers up to and including $A(t'-1)$.
%

First, let us show that if \buildprpg\ returns FALSE, then the
corresponding termination criterion would hold in all future
iterations. If $P(t+1)=P(t)$, then we have
$A(t+1)=A(t)$. Subsequently, since $P(t+1)\cup uP(t+1) = P(t)\cup
uP(t)$ and $A(t+1)=A(t)$, we have $P(t+2)\cup uP(t+2) = P(t+1)\cup
uP(t+1)$. Given that, we now show that $P(t+2) = P(t+1)$ and $uP(t+2)
= uP(t+1)$.

Assume to the contrary that there exists $p(t+2) \in P(t+2)$ such that
$p(t+1) \not\in P(t+1)$, that is
$p(t+1) \in uP(t+1)$. By the construction of the sets $P(t+1)$ and $P(t+2)$ in the \buildts\ procedure, we have
\begin{equation}
\label{e:comp1}
\begin{split}
\Phi &\longrightarrow {\bigvee_{l \in \mbox{\scriptsize \sl support}(p(t+2))}{\hspace{-0.5cm} l}}\hspace{0.3cm},\\
\Phi &\not\longrightarrow {\bigvee_{l \in \mbox{\scriptsize \sl support}(p(t+1))}{\hspace{-0.5cm} l}}\hspace{0.3cm}
\end{split}
\end{equation}
Consider an exhaustive classification of the effects $e$ of the
actions $A(t+1)$ with respect to our $p$ at time $t+2$.
\begin{enumerate}[(I)]
\item \label{type1} 
      $\forall \poutcome\in\poutcomeset(e): p\in\add(\poutcome)$, and 
      $\con(e) \in P(t+1)$
\item \label{type2} 
      $\forall \poutcome\in\poutcomeset(e): p\in\add(\poutcome)$, and 
      $\con(e) \in uP(t+1)$
\item \label{type3}
      $\exists \poutcome\in\poutcomeset(e): 
      p\not\in\add(\poutcome)$ or 
      $\con(e) \not\in P(t+1)\cup uP(t+1)$
\end{enumerate}
Suppose that the set (\ref{type1}) is not empty, and let $e \in
\mbox{(\ref{type1})}$. From $P(t) = P(t+1)$ we have that $\con(e) \in
P(t)$, and thus $\{\poutcome(t) \mid \poutcome \in \poutcomeset(e)\}
\subseteq \mbox{\sl support}(p(t+1))$. By the update of $\Phi$ in
\buildts\ we then have $\Phi \longrightarrow \bigvee_{\{\poutcome(t)
\mid \poutcome \in \poutcomeset(e)\}}\poutcome(t)$, and thus $\Phi
\longrightarrow {\bigvee_{l \in \mbox{\scriptsize \sl
support}(p(t+1))}{l}}$, contradicting Eq.~\ref{e:comp1}.

Alternatively, assume that the set (\ref{type1}) is empty.  Using the
arguments similar to these in the proof of
Lemma~\ref{lemma:disjunction}, $p(t+2)\in P(t+2)$ and $p(t+1) \not\in
P(t+1)$ in this case imply that \begin{equation} \label{e:comp2}
\begin{split}
\Phi & \longrightarrow \bigvee_{\substack{e\in\mbox{\scriptsize (\ref{type2})}\\ l\in \mbox{\scriptsize \sl support}\left(\con(e)(t+1)\right)}}\hspace{-0.6cm}l\\
\Phi & \not\longrightarrow \bigvee_{\substack{e\in\mbox{\scriptsize (\ref{type2})}\\ l\in \mbox{\scriptsize \sl support}\left(\con(e)(t)\right)}}\hspace{-0.6cm}l
\end{split}
\end{equation}
However, $A(t+1) = A(t)$, $uP(t+1) = uP(t)$, and $P(t+1) = P(t)$
together imply that all the action effects that can possibly take
place at time $t+1$ are also feasible to take place at time
$t$. Therefore, since for each $e \in \mbox{(\ref{type2})}$ we have
$\con(e) \in uP(t+1)$ by the definition of (\ref{type2}),
Eq.~\ref{e:comp2} implies that
\begin{equation}
\label{e:comp22}
\bigcup_{e\in\mbox{\scriptsize (\ref{type2})}}{\support\left(\con(e)(t+1)\right)}\cap uP(-m) \neq 
\bigcup_{e\in\mbox{\scriptsize (\ref{type2})}}{\support\left(\con(e)(t)\right)} \cap uP(-m),
\end{equation}
contradicting our termination condition. Hence, we arrived into
contradiction with our assumption that $p(t+1) \not\in P(t+1)$.

Having shown that $P(t+2)=P(t+1)$ and $uP(t+2)=uP(t+1)$, we now show
that the termination criteria implies that, for each $q(t+2) \in
uP(t+2)$, we have
\[
uP(-m)\cap\support(p(t+2)) = uP(-m)\cap\support(p(t+1)).
\]
Let $E_{p(t+2)}$ be the set of all effects of actions $A(t+1)$ such
that $\con(e) \in uP(t+1)$, and, for each outcome $\poutcome \in
\poutcomeset(e)$, we have $p\in\add(\poutcome)$. Given that, we have
\begin{equation}
\label{e:fsupport}
\begin{split}
uP(-m)\cap\support(p(t+2)) & =
    uP(-m)\cap\bigcup_{e\in E_{p(t+2)}}{\support(\con(e)(t+1))}\\
    &= uP(-m)\cap\bigcup_{e\in E_{p(t+2)}}{\support(\con(e)(t))}\\
    &= uP(-m)\cap\support(p(t+1))
\end{split},
\end{equation}
where the first and third equalities are by the definition of
$\support$ sets via Lemma~\ref{lemma:subweight-complete}, and the
second equation is by our termination condition.

The last things that remains to be shown is that our termination
criteria implies \getp$(t+2,G) = $\getp$(t+1,G)$. Considering the
simple cases first, if $\goal \not\subseteq P(t+1)\cup uP(t+1)$, from
$P(t+2)\cup uP(t+2) = P(t+1)\cup uP(t+1)$ we have \getp$(t+2,G) =
$\getp$(t+1,G) = 0$. Otherwise, if $\goal \subseteq P(t+1)$, from
$P(t+2) = P(t+1)$ we have \getp$(t+2,G) = $\getp$(t+1,G) = 1$.

This leaves us with the case of $\goal \subseteq P(t+1)\cup uP(t+1)$ and $\goal\cap uP(t+1) \neq \emptyset$. 
From $P(t+2) = P(t+1)$, $uP(t+2) = uP(t+1)$, and the termination condition, we have 
\[
\goal\cap uP(t) = \goal\cap uP(t+1) = \goal\cap uP(t+2).
\] 
From \getp$(t+1,G) = $\getp$(t,G)$ we know that action effects that
become feasible only in $A(t)$ do not increase our estimate of
probability of achieving any $g \in \goal\cap uP(t+1)$ from time $t$
to time $t+1$. However, from $P(t+1) = P(t)$, $uP(t+1) = uP(t)$, and
$A(t+1) = A(t)$, we have that no action effect will become feasible at
time $t+1$ if it is not already feasible at time $t$, and thus
\getp$(t+1,G) = $\getp$(t,G)$ will imply \getp$(t+2,G) =
$\getp$(t+1,G)$.

To this point we have shown that if \buildprpg\ returns FALSE, then
the corresponding termination criterion would hold in all future
iterations. Now, assume to the contrary to the claim of the theorem
that \buildprpg\ returns FALSE at some iteration $t$, yet there exists
a relaxed plan for $(\actions,\initbelief,\goal,\goalprob)$ that
starts with $\seqactions\onecondrel$.
First, if $\goalprob = 1$, then Lemma~\ref{lemma:disjunction} implies
that there exists time $T$ such that $\goal \subseteq P(T)$. If so,
then the persistence of our ``negative'' termination condition implies
$\goal \subseteq P(t)$. However, in this case we would have
\getp$(t,G) = 1$ (see the second {\bf if} of the \getp\ procedure),
and thus \buildprpg\ would return TRUE before ever getting to check
the ``negative'' termination condition in iteration $t$.
Alternatively, if $\goalprob = 0$, then \buildprpg\ would have
terminated with returning TRUE before the ``negative'' termination
condition is checked even once.

This leaves us with the case of $0 < \goalprob < 1$ and \getp$(t,G) <
\theta$. (\getp$(t,G) \geq \theta$ will again contradict reaching the
negative termination condition at iteration $t$.) We can also assume
that $\goal \subseteq P(t) \cup uP(t)$ because $P(t)\cup uP(t)$
contains all the facts that are not negatively known at time $t$, and
thus persistence of the negative termination condition together with
$\goal \not\subseteq P(t) \cup uP(t)$ would imply that there is no
relaxed plan for any $\goalprob > 0$. Let us consider the sub-goals
$\goal \cap uP(t) \neq \emptyset$.
\begin{enumerate}[(1)]
\item If for all subgoals $g \in \goal \cap uP(t)$, the implications in $\mbox{\sl Imp}_{\rightarrow g(t)}$ are {\em only} due to deterministic outcomes of the effects $\grapheffects{\mbox{\sl Imp}_{\rightarrow g(t)}}$, then the uncertainty about achieving $\goal \cap uP(t)$ at time $t$ is {\em only} due to the uncertainty about the initial state. Since the initial belief state is reasoned about with no relaxation, in this case \getp$(t,G) = {\sf WMC}(\Phi \wedge \bigwedge_{g \in G \setminus P(t)} \varphi_{g})$ provides us with an {\em upper bound} on the probability of achieving our goal $\goal$ by $\seqactions\onecondrel$ concatenated with an arbitrary linearization of an arbitrary subset of $A(0),\dots,A(t-1)$. The termination sub-condition \getp$(t+1,G) = $\getp$(t,G)$ and the persistence of the action sets $A(T)$, $T \geq t$, imply then that \getp$(t,G)$ provides us with an upper bound on the probability of achieving $\goal$ by $\seqactions\onecondrel$ concatenated with an arbitrary linearization of an arbitrary subset of $A(0),\dots,A(T)$, for all $T \geq t$. Together with \getp$(t,G) < \theta$, the latter conclusion contradicts our assumption that a desired relaxed plan exists.
\item If there exists a subgoal $g \in \goal \cap uP(t)$ such that some implications in $\mbox{\sl Imp}_{\rightarrow g(t)}$ are due to truly probabilistic outcomes of the effects $\grapheffects{\mbox{\sl Imp}_{\rightarrow g(t)}}$, then repeating the (relaxed) actions $A(t)$ in $A(t+1)$ will {\em necessarily} result in ${\sf WMC}(\Phi \wedge \bigwedge_{g \in G \setminus P(t+1)} \varphi_{g}) > {\sf WMC}(\Phi \wedge \bigwedge_{g \in G \setminus P(t)} \varphi_{g})$, contradicting our termination sub-condition condition \getp$(t+1,G) = $\getp$(t,G)$.
\end{enumerate}
Hence, we arrived into contradiction that our assumption that
\buildprpg\ returns FALSE at time $t$, yet there exists a relaxed plan for $(\actions,\initbelief,\goal,\goalprob)$ that starts with $\seqactions\onecondrel$.
\end{proof}


\bibliography{pffbib}
\bibliographystyle{theapa}

\end{document}